\documentclass[a4paper,fleqn]{cas-sc}

\usepackage[authoryear,longnamesfirst]{natbib}

\def\R		{\mathbb{R}}

\def\N		{\mathbb{N}}

\def\d		{\:\mathrm{d}}

\newcommand{\etal}{et al.\@\xspace}
\def\tsc#1{\csdef{#1}{\textsc{\lowercase{#1}}\xspace}}
\tsc{WGM}
\tsc{QE}
\tsc{EP}
\tsc{PMS}
\tsc{BEC}
\tsc{DE}
\usepackage{amsmath}
\usepackage{amsthm}
\usepackage{amssymb}
\usepackage{float}

\newtheorem{lem}{Lemma}
\newtheorem{cor}{Corollary}

\newtheorem{defn}{Definition}
\newtheorem{ex}{Example}
\usepackage{subcaption}
\usepackage{array}
\usepackage{ragged2e}
\usepackage{thm-restate}
\usepackage{placeins}

\begin{document}
\let\WriteBookmarks\relax
\def\floatpagepagefraction{1}
\def\textpagefraction{.001}
\shorttitle{Flexible Moment-Invariant Bases from Irreducible Tensors}
\shortauthors{T. Karella et~al.}
%\begin{frontmatter}

\title [mode = title]{Flexible Moment-Invariant Bases from Irreducible Tensors}                      
\tnotemark[1]
\author[1]{Tom\'{a}\v{s} Karella\corref{ca}}

\cormark[1]
\fnmark[1]
\ead{karella@lanl.gov}

\author[1]{Emily Shinkle}
% Shortening
%\ead{eshinkle@lanl.gov}

\author[2]{Alice A. E. Allen}
%\ead{aallen@lanl.gov}

\author[3]{Tom\'{a}\v{s} Suk}
%\ead{suk@utia.cas.cz}

\author[1]{Nicholas Lubbers}

% or include affiliations in footnotes:
\author[1]{Roxana Bujack}
%\ead{bujack@lanl.gov}

\affiliation[1]{organization={Los Alamos National Laboratory},
                addressline={P.O. Box 1663}, 
                city={Los Alamos},
                postcode={87545}, 
                state={NM},
                country={USA}}

\affiliation[2]{organization={Max Planck Institute for Polymer Research},
                addressline={Ackermannweg 10}, 
                city={Mainz},
                postcode={55128}, 
                %state={},
                country={Germany}}
\affiliation[3]{organization={Institute of Information Theory and Automation of the CAS},
                addressline={Pod Vod\'arenskou v\v e\v z\'i 4}, 
                city={Prague 8},
                postcode={55128}, 
                %state={},
                country={Czechia}}

\begin{abstract}
Representations that respect rotational symmetry are beneficial across a wide range of applications, as they enable robust and consistent treatment of differently oriented inputs.
Moment invariants encode this symmetry explicitly and can serve both as interpretable handcrafted descriptors and as symmetry-preserving features in deep-learning architectures. For these applications, a useful basis of invariants should be complete, algebraically independent, robust to degeneracies in the input, and flexible enough to support different representation and learning settings. Existing state-of-the-art methods for constructing such bases remain well defined when moment tensors vanish identically, but can fail under other degeneracies frequently encountered in practice, including those arising from spherical functions. We address this limitation by combining two complementary approaches to moment invariants: spherical-harmonic representations and Cartesian tensor algebra. The resulting construction yields complete and independent bases that remain effective under a broader class of degeneracies and can be readily adapted for both handcrafted methods and deep-learning pipelines.
\end{abstract}

\begin{graphicalabstract}
\includegraphics[width=\linewidth]{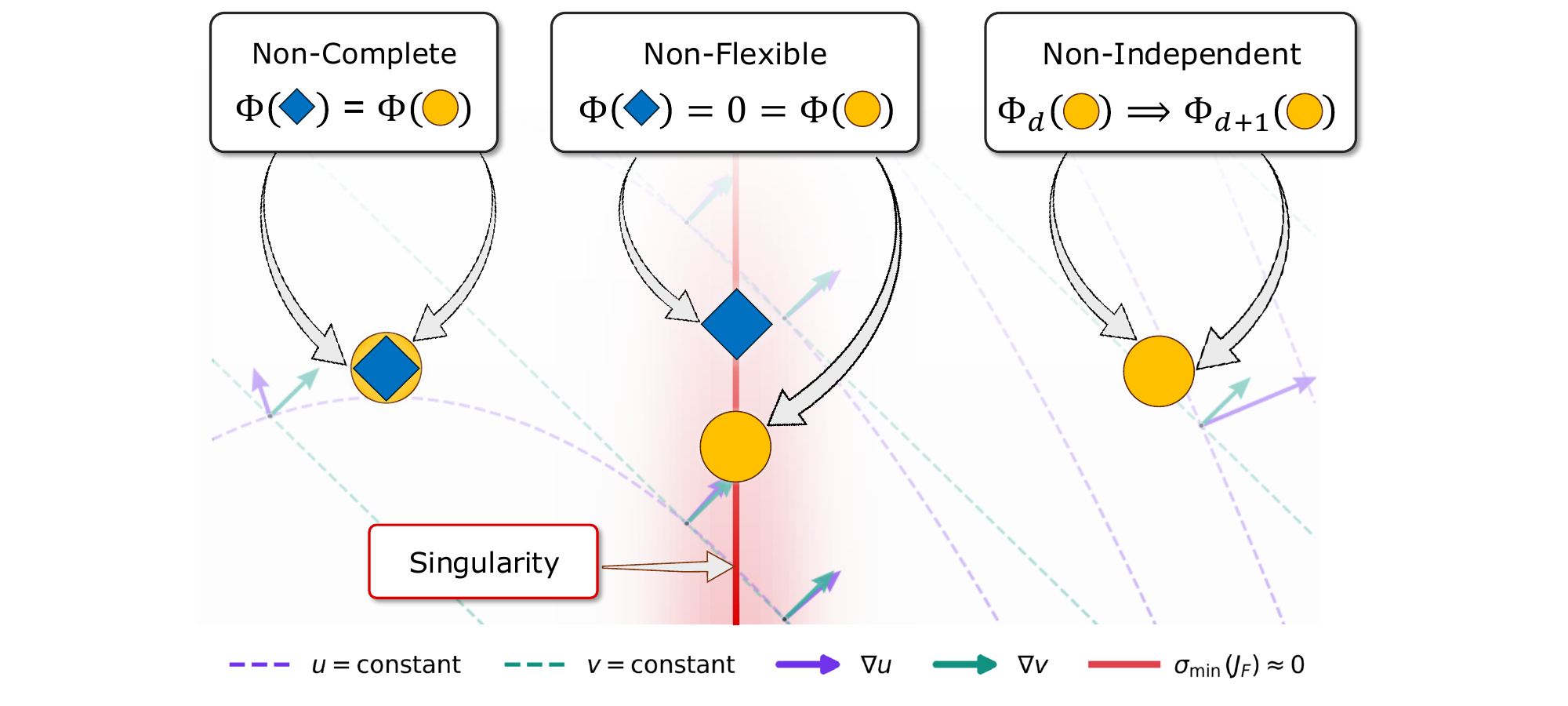}
\end{graphicalabstract}

\begin{highlights}
\item Combines spherical harmonics and Cartesian tensor algebra for rotation invariants.

\item Develops complete, independent bases under degeneracies that surpasses prior methods.

\item Supports handcrafted descriptors and symmetry-preserving deep-learning features.
\end{highlights}

\begin{keywords}
Moment invariants \sep rotation invariant \sep spherical functions, deviators \sep spherical decomposition \sep symmetric-tracefree tensors
\end{keywords}

\maketitle

\section{Introduction}
% More CARs 
Respecting the symmetries inherent in data is beneficial across many scientific and machine-learning applications, particularly when those symmetries reflect the physical system under study. Encoding such prior knowledge reduces the space of functions that a model must learn and can thereby improve sample efficiency, accelerate convergence, and enhance generalization~\cite{brehmer2025, karella2026}. In applications involving spatial data, rotational symmetry is especially important because equivalent patterns or functions may occur in arbitrary orientations.

Moment invariants provide a systematic way to construct such rotation-invariant representations. Moments are projections of a function onto a basis of a function space~\cite{FSZ16}. By combining moments through appropriately chosen sums and products, one obtains moment invariants: characteristic descriptors of the underlying function that remain unchanged under rotation. These invariants can be used both as interpretable handcrafted descriptors and as symmetry-aware features in deep-learning architectures or in classical signal processing. 

For a space of admissible functions or patterns, an optimal set of descriptors should be \textbf{independent}, \textbf{complete}, and \textbf{flexible}~\cite{Flu00,bujack2017flexible}. Independence means that no element of the set can be expressed as a function of the others, while completeness means that every descriptor outside the set can be expressed as a function of the set. Flexibility requires the set to remain complete when restricted to any subset of the admissible function space. We refer to a complete and independent set as a \emph{basis}~\cite{Flu00}. Figure~\ref{fig:definitions} illustrates these three properties.

\begin{figure}
    \centering
    \includegraphics[width=0.95\linewidth]{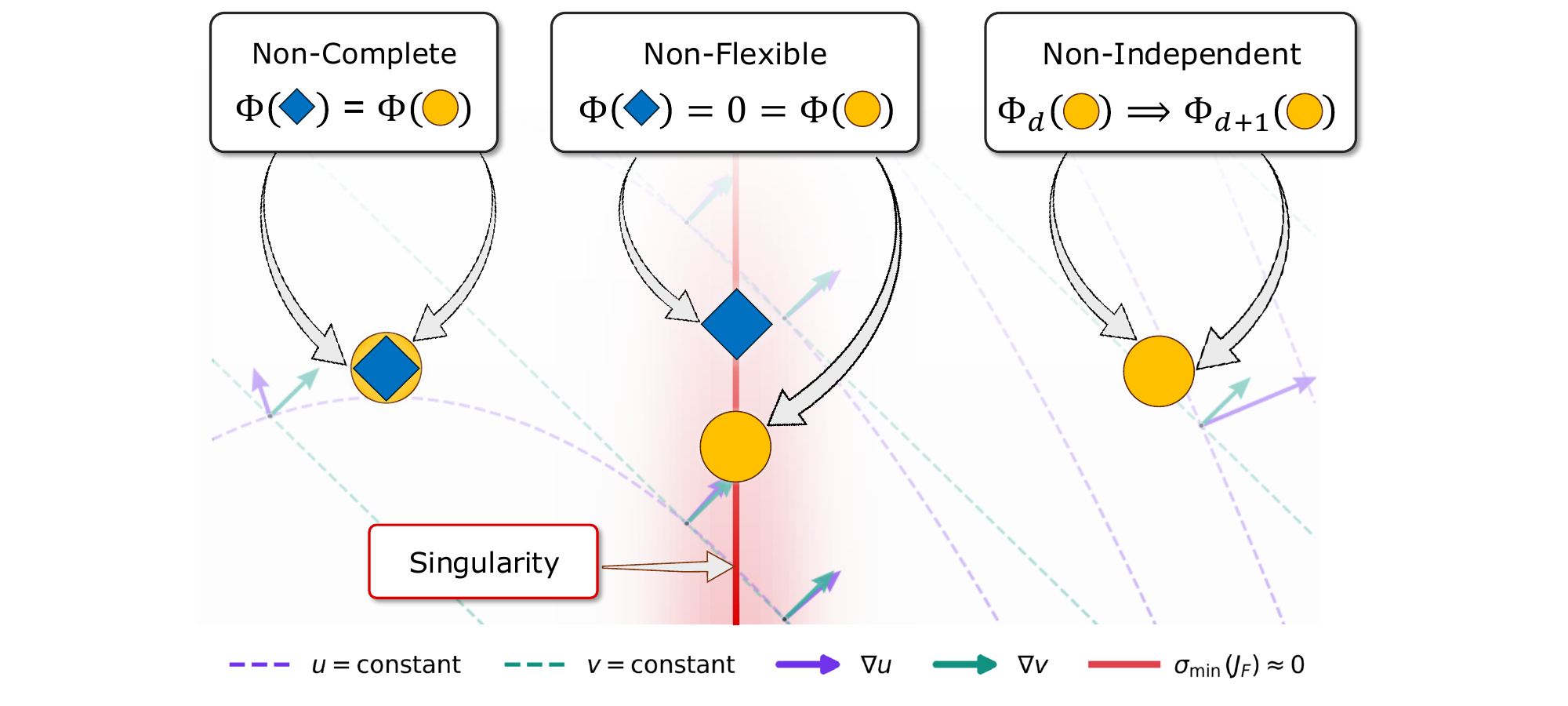}
    \caption{Illustrations of independence, completeness, and flexibility for a set of invariant descriptors. If an invariant basis is incomplete, two distinct patterns may be projected to the same point. If it lacks flexibility, singularities can arise: different patterns—often associated with zero-valued descriptors—may share the same descriptor values while belonging to different regions of the space. A basis is non-independent if a smaller set of invariants produces the same projection of the patterns. }
    \label{fig:definitions}
\end{figure}

Approaches to constructing moment invariants can be broadly categorized according to the underlying function-space representation. Cartesian representations use bases such as monomials and naturally lead to Cartesian tensor algebra, whereas rotational representations use bases such as complex exponentials or spherical harmonics. In two dimensions, rotational representations have been shown to yield flexible—and therefore optimal—bases~\cite{bujack2017flexible}. An analogous result has not yet been established in three dimensions. For the three-dimensional case, Bujack et al.\ constructed an overcomplete set using Cartesian tensor algebra and showed that it preserves completeness when any individual moment tensor vanishes identically~\cite{bujack2022systematic}.

%TODO: This is missing contributions!
Robustness to vanishing moment tensors does not guarantee robustness to all degeneracies of the input function. We show that the existing tensor-algebra construction loses its discrimination capability under a broader class of degeneracies, including degeneracies that commonly arise for spherical functions. Consequently, the available three-dimensional constructions do not yet provide invariant sets that are simultaneously independent, complete, flexible, and robust across the relevant function spaces. A companion paper applies our theory to machine-learned interatomic potentials, outperforming current methods~\cite{allen2025optimal}, while another provides a proof of Corollary~\ref{c:complete} \cite{bujack2026}.

\section{Related Work}
%TODO: Too short add the deep learning stuff. 

Hu introduced moment invariants to image processing~\cite{Hu62}. Flusser later defined a complete, independent set for 2D scalar functions~\cite{Flu00}, and extended it to inverse problems~\cite{Flu02}. Bujack \etal showed that complex moments form an optimal basis in 2D, when the pattern is known, by identifying a non-zero moment and combining it with all of the others~\cite{bujack2017flexible}.

In 3D the problem is more difficult. One approach uses Cartesian moment tensors~\cite{DN77}, from which scalar invariants are formed via zeroth-rank contractions. However, the number of possible contractions is vast and finding independent sets is difficult. Suk and Flusser proposed computing all zeroth-rank contractions up to a given order, eliminating linearly dependent or graph-isomorphic ones~\cite{Suk2011tensor}. This method extends to affine transforms and vector/tensor fields~\cite{suk2004graph, SukFlu:AMIgraph, kostkova2019affine, flusser2023affine}, though the resulting sets may still include higher-order dependencies and can exceed the optimal size by up to a factor of 10.
Langbein and Hagen improved this by using a method of detecting dependencies via derivatives—dependent invariants having linearly dependent derivatives, detectable through rank factorization.~\cite{langbein2009generalization}. Independently,  Hickman proposed a similar derivative-based approach~\cite{hickman2012geometric}. 

A distinct approach uses spherical harmonics, which form an irreducible basis for rotations. Lo and Don used Clebsch-Gordan coefficients to derive 12 rotation invariants up to third order~\cite{LD89}, which was followed by a similar method from Burel and Henocq~\cite{BH95}. Suk \etal verified the linear independence of Lo and Don’s invariants and extended them to fourth order~\cite{suk20153d}. Kazhdan \etal proposed summing spherical harmonics by frequency and computing the norm of each band, leveraging band-limited rotation effects~\cite{KFR03}. This works for all orders but results in incomplete descriptors when functions differ only by band orientation.

Xu and Li defined 3D surface moments using four rotation-invariant geometric primitives: point distances, areas of triangles, inner products with the origin, and volumes of tetrahedra~\cite{xu20063}, combining these via $n$-fold integrals. Independence was not discussed.
Tsai \etal addressed efficient moment computation in distributed systems~\cite{tsai2020approaches}, driven by the needs of the Exascale Computing Project~\cite{ahrens2025ecp}.
Bujack \etal noted that combining moment tensors of different orders can lead to degenerate invariants (e.g., for symmetric functions), and proposed robust bases using products of one or two tensor types~\cite{bujack2022systematic}.  These eliminate vanishing tensors, but the invariants still exhibit degeneracies for spherical-radial separable functions, as we will show.

In this work we combine the completeness and independence of Langbein's tensor contractions~\cite{langbein2009generalization} with the rotational irreducibility of spherical harmonics to construct a highly flexible invariant basis. It is provably robust against vanishing moment tensors and spherical-radial separable functions. Empirically, it appears robust to all known function degeneracies, though this remains unproven.

Most of the above use the generator approach---algebraically deriving invariants. Alternatively, the normalization approach orients functions to a standard pose, simplifying invariants~\cite{PCO85, Cant96, bujack2017tensor}, but is unsuitable for unsupervised learning because of input-dependent orientation. We do not consider this approach here.

% TODO: This should be moved to introductions

\section{Foundations}
Tensors are algebraic objects that behave in a particular way when the coordinate system changes. In a specific basis, they can be represented by numerical arrays, with rank determined by the number of indices (scalars have rank 0, vectors rank 1, matrices rank 2, etc.) For further background, see Bowen and Wang~\cite{bowen2008introduction} and Grinfeld~\cite{grinfeld2013introduction}.
\footnote{Although tensor indices are typically distinguished as covariant or contravariant, this distinction is unnecessary for orthogonal transformations $A\in\mathbb{R}^{3\times3}$ (rotations and reflections). Since we work exclusively with such transformations, we omit the distinction here.}

\begin{defn}\label{d:tensor}

An array $T_{i_1 \dots i_n}$, where each index $i_k \in \{1,2,3\}$ for $k=1,\dots,n$, that transforms according to
\begin{equation}
\begin{aligned}
\label{tensorTransformation}
  {T^\prime}_{i_1\ldots i_n}=\sum_{j_1,\ldots,j_n=1}^3 R_{i_1j_1}\ldots R_{i_nj_n}T_{j_1\ldots j_n}
 \end{aligned} \end{equation}
 under a rotation described by the orthogonal matrix $R_{ij}\in\mathbb{R}^{3\times 3}$ (with $R^TR = I$ and $\det(R) = 1$), is called a \textbf{tensor} of rank $n$. Sometimes we explicitly denote the rank by a left superscript e.g., ${}{^n}{T}{}$. 
\end{defn} 

 \begin{lem}\label{l:product}
 Let $T$ and $\tilde T$ be two tensors of ranks $n$ and $\tilde n$ respectively, then their \emph{product} $T\otimes \tilde T$ (also called outer product or tensor product) is a tensor of rank $n+\tilde n$:
 \begin{equation} \begin{aligned} \label{product}
  (T\otimes \tilde T)_{i_1\ldots i_n \tilde i_1\ldots\tilde i_{\tilde n}}:= {T}_{i_1\ldots i_n} {\tilde T}_{\tilde i_1\ldots\tilde i_{\tilde n}} .
 \end{aligned} \end{equation}
 \end{lem}
 
  \begin{lem}\label{l:contraction}
 Let $T$ be a tensor of rank $n$. Then the \emph{contraction} $T^{(k,l)}$ of two indices $i_k$ and $i_l$, $1\leq k<l\leq n$, is a tensor of rank $n-2$:
  \begin{equation} \begin{aligned} \label{contraction}
  T_{i_1\ldots i_{k-1}  i_{k+1}\ldots i_{l-1}  i_{l+1}\ldots i_n}^{(i_k,j_l)}
:=\sum_{j = 1}^3 T_{i_1\ldots i_{k-1} j i_{k+1}\ldots i_{l-1} j i_{l+1}\ldots i_n} .
 \end{aligned} \end{equation}

 \end{lem}

 \begin{ex}
The product of a matrix $M\in\R^{3\times 3}$ and a vector $v\in\R^3$ is a tensor product $\otimes$ followed by a contraction of the last two indices:
 \begin{equation} \begin{aligned} 
(M\otimes v)^{(2,3)}_i \overset{Lem. \ref{l:contraction}}= \sum_{j = 1}^3 (M\otimes v)_{ijj} \overset{Lem. \ref{l:product}}= \sum_{j = 1}^3 M_{ij} v_j.
\end{aligned} \end{equation}
\end{ex}
For full contractions of an even-rank tensor $2n \rightarrow 0$, where indices are paired, we simplify the notation by labeling each contracted pair with the same number, e.g.,
\begin{equation} \begin{aligned} 
  T(1,2,3,1,3,2) = T^{(i_1,i_4)(i_2,i_6)(i_3,i_5)}
=\sum_{j_1,\ldots,j_3 = 1}^3 T_{j_1,j_2,j_3,j_1,j_3,j_2},
\end{aligned} \end{equation}
and we separate them to simplify the association to the individual tensors in a product:
\begin{equation} \begin{aligned} 
({}{^1}{T}{^2}{}{^2}{T}{^2})(1) (2) (2,3) (3,1)&= ({}{^1}{T}{^2}{}{^2}{T}{^2})^{(i_1,j_6)(i_2,j_3)(i_4,j_5)}
\\&=\sum_{j_1,\ldots,j_3 = 1}^3 {}{^1}{T}{}_{j_1}{}{^1}{T}{}_{j_2}{}{^2}{T}{}_{j_2,j_3}{}{^2}{T}{}_{j_2,j_3}.
\end{aligned} \end{equation}

\subsection{Moment Tensors} 
Dirilten and Newman organized the moments of each order in such a way that they transform like tensors, and used contractions to produce moment invariants with respect to orthogonal transformations~\cite{DN77}.

\begin{defn}\label{d:mom_tensor}
For a scalar function $f:\R^3\to\R$ with compact support over the unit ball $B_1(0)$, the \textbf{moment tensor} ${}{^\ell}M$ of order $\ell\in\N$ takes the shape
\begin{equation}\label{momentTensor} \begin{aligned}
{}{^\ell}M &:=\int_{B_1(0)} x^{\otimes \ell} f(x)\,d^3 x,\\
{}{^\ell}M_{i_1\ldots i_\ell}&=\int_{B_1(0)}x_{i_1}\ldots x_{i_\ell}f(x)\,d^3 x.
\end{aligned}\end{equation}
\end{defn}

It follows immediately from Def.~\ref{d:mom_tensor} and the commutativity of multiplication that moment tensors of scalar functions are totally symmetric. Therefore, a moment tensor of order $\ell$ has only $(\ell+1)(\ell+2)/2$ degrees of freedom~\cite{bujack2022systematic}. The following corollary is the foundation of the generator approach in 3D. 

\begin{cor}\label{c:0}
The rank-zero contractions of any product of moment tensors are moment invariants with respect to rotation and reflection.
\end{cor}

Detailed proofs are available~\cite{bujack2017tensor} but basically the argument is that because moment tensors are tensors, their products are tensors, and hence their contractions are tensors. Rank-zero tensors are scalars and scalars do not change under rotation.

\subsection{Desirable Properties of a Set of Invariants}\label{s:properties}
Corollary~\ref{c:0} provides an infinite number of invariants, many of which carry redundant information. To formalize what constitutes a good descriptor set, let a group \(G\) act on a space of functions \(X\); \(x,y\in X\) are equivalent if \(y=g\cdot x\) for some \(g\in G\).

\textbf{Invariance:}
A function \(\phi:X\to\mathbb{R}\) is called an invariant if 
\begin{equation}
%\phi(g\cdot x)=\phi(x)\quad\text{for all }g\in G,\;x\in X.
\forall g \in G, \; x \in X, \quad \phi(g\cdot x)=\phi(x)
\end{equation}
A collection of such functions defines an invariant set $\Phi(x)=(\phi_1(x),\dots,\phi_m(x))$.
% \begin{equation}
%     \Phi(x)=(\phi_1(x),\dots,\phi_m(x)).
% \end{equation}

To characterize the elements of $X$ effectively, a set of invariants must have three key properties: independence, completeness, and flexibility~\cite{bujack2017flexible}.

%Changing the order
\textbf{Independence}:
A set of invariants is (algebraically) independent for a space $X$ if for any polynomial $F$ that satisfies
\begin{equation}
    % F\big(\phi_1(x),\dots,\phi_m(x)\big)=0\quad\text{for all }x\in X,
    \forall x \in X, \quad F\big(\phi_1(x),\dots,\phi_m(x)\big)=0,
\end{equation}
\( F \) must be identically zero. Independence guarantees that the descriptor set contains no redundancy.

\textbf{Completeness}: 
A set of invariants $\Phi$ is complete in $X$ if for any other invariant $\psi\notin \Phi$, the set $\{\psi\} \cup \Phi$ is dependent in $X$.
% can be expressed as a function of its elements. That is, for any invariant $\psi\notin \Phi$ there exists a function 
% $F \in \mathbb{R}[z_1,\dots,z_m]$ such that
% \begin{equation}
% \psi(x) = F\big(\phi_1(x),\dots,\phi_m(x)\big)\quad \text{for all }x\in X.
% \end{equation}
Completeness ensures that no invariant information is lost.

Invariant bases often require non-trivial assumptions that restrict the input space for completeness to hold. For example, in 2D Flusser’s rotation-invariant basis assumes that certain moments are nonzero~\cite{Flu00}. In the absence of such constraints, some invariants collapse to zero and the completeness property is lost.

\textbf{Flexibility}: 
A set of invariants is flexible if it is complete for all subsets $Y\subseteq X$.
Flexibility ensures that any two distinct functions (up to the group action) can be distinguished by their invariants even if they exhibit certain degeneracies such as vanishing moments or symmetric patterns.
% Flexibility ensures that the invariants are well-defined and discriminative for any admissible function, regardless of degeneracies such as vanishing moments or symmetric patterns.

\subsubsection{Independence}\label{s:independence}
The challenge of eliminating dependent invariants is also known as the problem of algebraic- or functional-invariance testing. 
%The well-known Jacobi criterion which, as the name suggests, goes back to Jacobi~\cite{jacobi1841determinantibus}, 
The well-known Jacobi criterion~\cite{jacobi1841determinantibus}
uses the fact that each invariant $\phi$ is a function of the moments, i.e., the entries of the moment tensors, $\phi({}{^0}M,{}{^1}M_1,\ldots,{^2}M_{11},\ldots)$.
An invariant $\phi$ is dependent on a set of invariants $\{\phi_1,\ldots\phi_m\}$ if there exists a function 
%\begin{equation} \begin{aligned} 
 $\phi = F(\phi_1,\ldots,\phi_m)$.
%\end{aligned} \end{equation}
It follows from the chain rule that each partial derivative with respect to a moment $M$ is a linear combination of the derivatives of the invariants in the set:
\begin{equation} \begin{aligned} 
\frac{ \partial \phi} { \partial M} =\sum_{i=1}^m \frac{ \partial F(\phi_1,\ldots,\phi_m)} { \partial \phi_i}\frac{ \partial \phi_i} { \partial M}.
\end{aligned} \end{equation}
This means that the rank in the Jacobian is full if the invariants are independent. In polynomial algebra, the determinant of the Jacobian is a polynomial, which reduces the independence testing to polynomial identity testing of the determinant against zero~\cite{ehrenborg1993apolarity}, which can be performed exactly in polynomial time~\cite{beecken2013algebraic}. For the immense search space of moment invariants, this is too slow. Langbein and Hagen~\cite{LH09} suggested testing for the linear dependence of the partial derivatives in an approximate but very fast way by numerically evaluating the rank of the Jacobian for an input of random numbers as the moments. We will refer to this as Langbein's algorithm (\ref{appendix:langbein}). 

% For brevity I moved Langbein to appendix

% Just copy/past changing order. 
\subsubsection{Completeness}\label{s:completeness}
Like all tensor-contraction-based algorithms, our work is based on the following corollary~\cite{bujack2026}, which allows us to generate a basis by removing dependent invariants~\cite{bujack2022systematic}.
\begin{cor}\label{c:complete}
The set of all contractions of the powers of a moment tensor to zeroth order is a complete set of rotation/reflection invariants of the tensor's moments.
\end{cor}
\subsubsection{Flexibility}\label{s:flexibility}
Bujack \etal showed that the bases produced with Langbein's algorithm are generally not flexible, i.e., the bases lose their completeness for certain functions~\cite{bujack2022systematic}. They distinguish between \textbf{homogeneous} invariants, which are constructed from only a single moment tensor and its powers, and  \textbf{simultaneous} invariants, which contain more than one kind of moment tensor.
They show that simultaneous invariants are vulnerable to functions with vanishing moment tensors because if a tensor is zero, all of its products, their contractions, and therefore all simultaneous invariants that contain that tensor vanish, too, and the information in its product partners gets lost.
As a consequence, invariants with fewer factors are to be preferred.
Bujack \etal prove that a complete set can be derived by using only homogeneous invariants and simultaneous invariants between moment tensors of no more than two different orders. They construct bases that are flexible with respect to functions with vanishing moment tensors by using the homogeneous invariants plus three simultaneous invariants of one non-vanishing tensor combined with each other order.

We show that these bases are vulnerable to a different degeneracy, which is common in real-world applications, namely spherical functions.

\section{Problem Statement}\label{s:problem}
We motivate our work by showing where the flexible bases proposed thus far fail.

We denote contractions by pairs of upper indices in parentheses. For full contractions of an even-ranked tensor of rank $2n$ to rank zero, where each tensor index is paired with another one, we use the notation of just marking the index pairs by the same number
% , for example,
% \begin{equation} \begin{aligned} 
%   T(1,2,3,1,3,2) = T^{(i_1,i_4)(i_2,i_6)(i_3,i_5)}
% =\sum_{j_1,\ldots ,j_3 = 1}^3 T_{j_1,j_2,j_3,j_1,j_3,j_2},
% \end{aligned} \end{equation}
and we separate them to simplify the association with the individual tensors in a product, for example,
% \begin{equation} \begin{aligned} 
% ({}{^1}{T}{^2}{}{^2}{T}{^2})(1) (2) (2,3) (3,1)=\sum_{j_1,\ldots ,j_3 = 1}^3 {}{^1}{T}{}_{j_1}{}{^1}{T}{}_{j_2}{}{^2}{T}{}_{j_2,j_3}{}{^2}{T}{}_{j_2,j_3}.
% \end{aligned} \end{equation}
\begin{equation} \begin{aligned} 
({}{^1}{T}{^2}\;{^2}{T}{^2})(1) (2) (2,3) (3,1)&= ({}{^1}{T}{^2}{}{^2}{T}{^2})^{(i_1,j_6)(i_2,j_3)(i_4,j_5)}
\\&=\sum_{j_1,\ldots ,j_3 = 1}^3 {}{^1}{T}{}_{j_1}{}{^1}{T}{}_{j_2}{}{^2}{T}{}_{j_2,j_3}{}{^2}{T}{}_{j_2,j_3}.
\end{aligned} \end{equation}

\subsection{Example of the Problem with Flexibility}\label{l:example}

\begin{figure}%[ht]
\captionsetup[subfigure]{labelformat=empty}
\centering
\subcaptionbox{View along the x-axis.}{\includegraphics[width=0.32\linewidth]{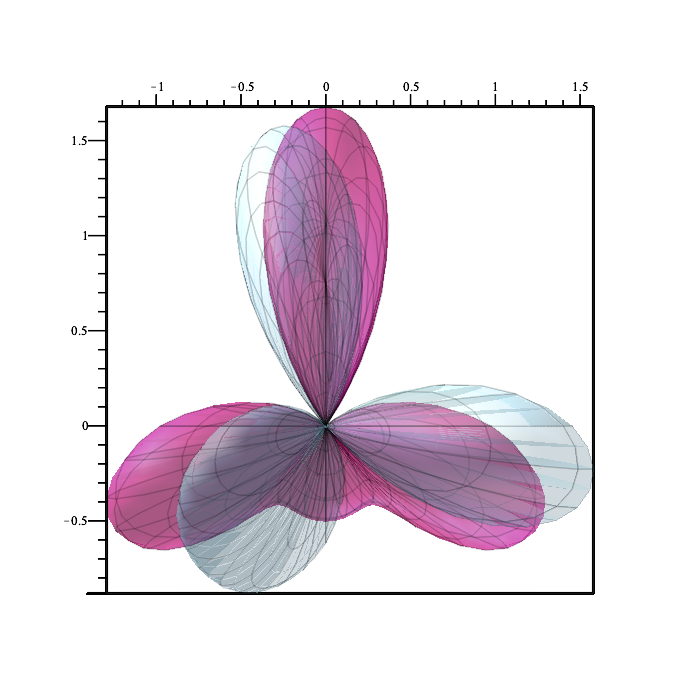}}\hfill
\subcaptionbox{View along the y-axis.}{\includegraphics[width=0.32\linewidth]{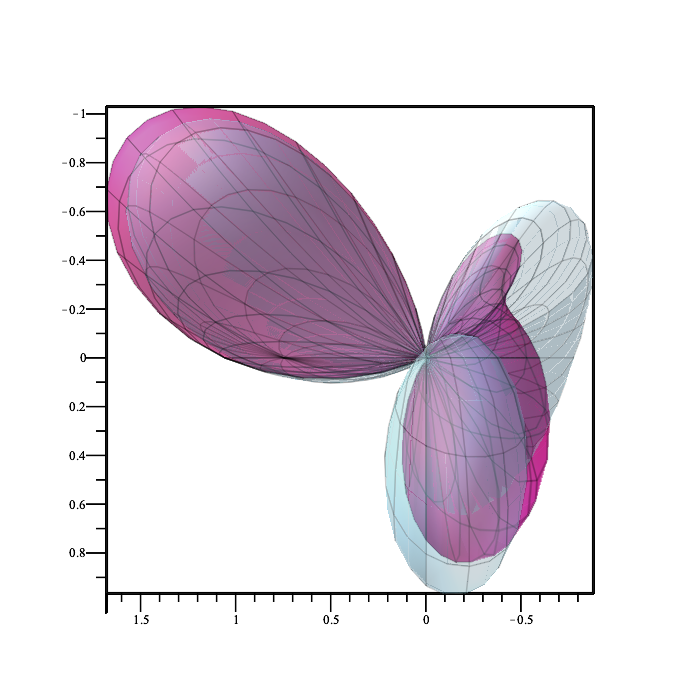}}\hfill
\subcaptionbox{View along the z-axis.}{\includegraphics[width=0.32\linewidth]{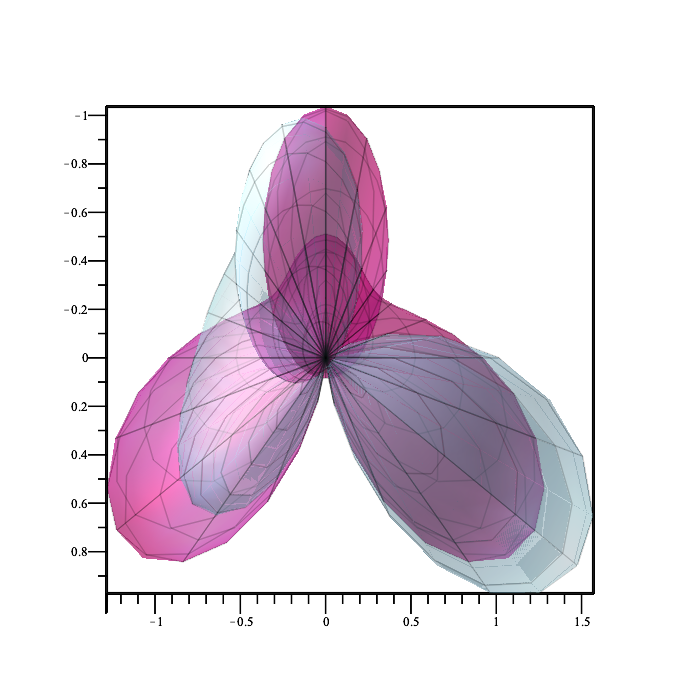}}
\caption{The two functions from~\eqref{3dScalarF} differ but cannot be distinguished by Bujack's or Langbein's bases. $f_1$: pink, $f_2$: blue, with their magnitude depending on the angle visualized by the distance from the origin.\label{f:scalar3d}} 
\end{figure}

The bases of Bujack \etal are not fully flexible. They lose their completeness if the input function is spherical. As an example, Fig.~\ref{f:scalar3d}, consider the two cubic functions
\begin{equation} \label{3dScalarF}\begin{aligned} 
f_1(\phi,\theta) &= 3xy^2 - 3xz^2 - 3\sqrt2 y^2z + \sqrt2z^3,\\
f_2(\phi,\theta) &= 3xy^2 - 3xz^2 + y^3 -3y^2z -3yz^2  + z^3.
\end{aligned} \end{equation}
Their moments up to second order are identically zero and so are their corresponding invariants. With $c=\frac{8\pi}{315}$, the third order moments equal
\begin{equation}\begin{aligned}
f_1:&\;{}{^3}{M}{}_{1,2,2}=c,\quad {}{^3}{M}{}_{1,3,3}=-c,\quad {}{^3}{M}{}_{2,2,3}=-\sqrt2\, c,\quad {}{^3}{M}{}_{3,3,3}=\sqrt2\, c,
\\
f_2:&\;{}{^3}{M}{}_{1,2,2}={}{^3}{M}{}_{2,2,2}={}{^3}{M}{}_{3,3,3}=c,\quad {}{^3}{M}{}_{1,3,3}={^3}{M}{}_{2,2,3}={}{^3}{M}{}_{2,3,3}=-c,
\end{aligned}\end{equation}
and the third order homogeneous invariants for both functions are identical.
% \begin{equation}\label{bujack}\begin{aligned}
% {}{^3}{M}{^2}(1,1,2),(2,3,3)&=0,\\
% {}{^3}{M}{^2}(1,2,3),(1,2,3)&=14c^2,\\
% {}{^3}{M}{^4}(1,1,2),(2,3,4),(3,5,5),(4,6,6)&=0,\\
% {}{^3}{M}{^4}(1,1,2),(2,3,4),(3,5,6),(4,5,6)&=0,\\
% {}{^3}{M}{^4}(1,2,3),(1,2,4),(3,5,6),(4,5,6)&=92c^4,\\
% {}{^3}{M}{^4}(1,1,2),(2,3,4),(3,4,5),(5,6,6)&=0,\\
% {}{^3}{M}{^6}(1,1,2),(2,3,4),(3,4,5),(5,6,7),(6,7,8),(8,9,9)&=0.
% \end{aligned}\end{equation}
Fig.~\ref{f:bujack} shows Suk's graph visualization of Bujack's homogeneous third order invariants~\cite{suk2004graph}. Each node with its inscribed number corresponds to a tensor with that order. An edge with its inscribed number $n$ connecting two nodes corresponds to the contraction of $n$ of their indices. We have excluded self-loops. If the number of incident edges is smaller than the order of the node, all remaining indices are self contractions (traces).

In the order shown in Fig.~\ref{f:bujack}, the invariants take the values $14c^2,0,0,0,92c^4,0,0$.
All simultaneous invariants up to third order are zero, too, because of the vanishing moments of order two and one.
As a result, Bujack's basis up to third order cannot discriminate between these two functions even though they are third-order polynomials that differ by other than just a rotation or reflection.

\begin{figure}%[ht]
\centering
{\includegraphics[width=0.14\linewidth]{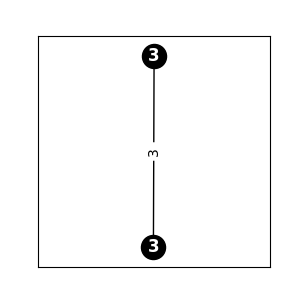}}\hfill
{\includegraphics[width=0.14\linewidth]{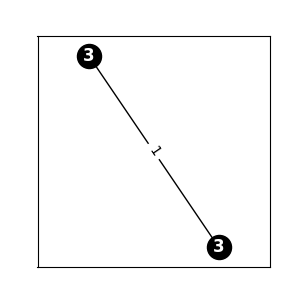}}\hfill
{\includegraphics[width=0.14\linewidth]{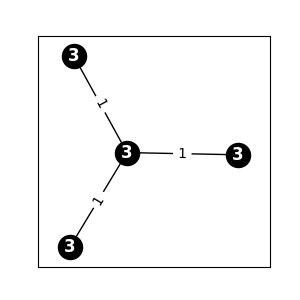}}\hfill
{\includegraphics[width=0.14\linewidth]{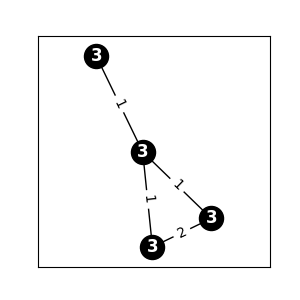}}\hfill
{\includegraphics[width=0.14\linewidth]{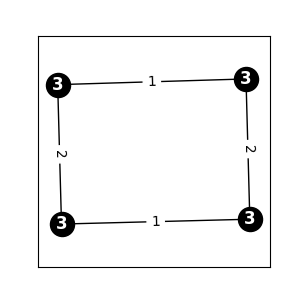}}\hfill
{\includegraphics[width=0.14\linewidth]{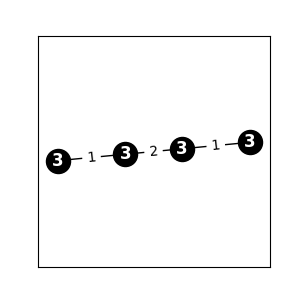}}\hfill
{\includegraphics[width=0.14\linewidth]{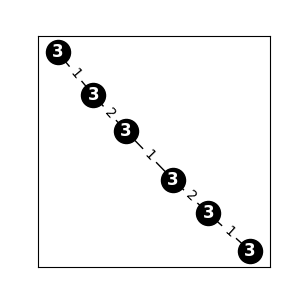}}\hfill \hfill
\caption{The homogeneous third-order invariants in Bujack's basis.\label{f:bujack}} 
\end{figure}

\begin{figure}%[ht]
\centering
{\includegraphics[width=0.14\linewidth]{pics/langbein0.png}}\hfill
{\includegraphics[width=0.14\linewidth]{pics/langbein1.png}}\hfill
{\includegraphics[width=0.14\linewidth]{pics/langbein5.png}}\hfill
{\includegraphics[width=0.14\linewidth]{pics/langbein4.png}}\hfill
{\includegraphics[width=0.14\linewidth]{pics/langbein2.png}}\hfill
{\includegraphics[width=0.14\linewidth]{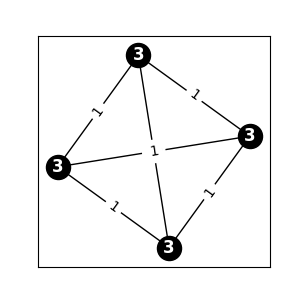}}\hfill\hfill
\caption{The homogeneous third-order invariants in Langbein's basis.\label{f:langbein}} 
\end{figure}
Langbein's basis behaves similarly, Fig.~\ref{f:langbein}. It contains only six homogeneous third-order invariants, which also coincide for the two functions. In the order given in the figure, they take the values $14c^2,0,0,0,92c^4,6c^4$.
% \begin{equation}\label{langbein}\begin{aligned}
% {}{^3}{M}{^2}(1,1,2),(2,3,3)&=0,\\
% {}{^3}{M}{^2}(1,2,3),(1,2,3)&=14c^2,\\
% {}{^3}{M}{^4}(1,1,2),(2,3,4),(3,5,5),(4,6,6)&=0,\\
% {}{^3}{M}{^4}(1,1,2),(2,3,4),(3,5,6),(4,5,6)&=0,\\
% {}{^3}{M}{^4}(1,2,3),(1,2,4),(3,5,6),(4,5,6)&=92c^4,\\
% {}{^3}{M}{^4}(1,2,3),(1,4,5),(2,5,6),(3,4,6)&=6c^4.
% \end{aligned}\end{equation}
What seem like three independent moments 
%actually also 
reduces to two because the three non-zero invariants become dependent via the relation 
%$2({}{^3}{M}{^4}\ldots (4,5,6)+{}{^3}{M}{^4}\ldots (3,4,6)) = ({}{^3}{M}{^2})^2$
$({}{^3}{M}{^2})^2 = \allowbreak 2({}{^3}{M}{^4}\ldots (4,5,6)+{}{^3}{M}{^4}\ldots (3,4,6))$
for spherical functions. 
%Fig.~\ref{f:langbein} shows a graph visualization of Langbein's homogeneous third-order invariants.

\subsection{Analysis of the Problem with Flexibility}\label{s:analysis}
We use the standard relations between Cartesian and spherical coordinates
\begin{eqnarray} \label{spherical}\begin{aligned} 
x = \begin{pmatrix} x_1\\ x_2\\ x_3\end{pmatrix} = 
r \begin{pmatrix}\sin(\theta)\cos(\phi)\\ \sin(\theta)\sin(\phi)\\ \cos(\theta)
\end{pmatrix}
=\|x\|\frac{x}{\|x\|}
=: r \hat r .
\end{aligned} \end{eqnarray}
In spherical coordinates, the 3D moment tensors of the unit ball 
take the form
\begin{equation} \label{sphericalMoment}\begin{aligned}
{}{^\ell}M &\overset{\eqref{momentTensor}}=\int_{B_1(0)} x^{\otimes \ell} f(x)\d^3 x
\\ &\overset{\eqref{spherical}}= \int_0^{2\pi}\int_0^{\pi}\int_0^1 x(r,\phi,\theta)^{\otimes \ell}f(r,\phi,\theta) r^2 \sin(\theta)\, d r\, d\theta\, d\phi
\\ &\overset{\eqref{spherical}}= \int_0^{1}\int_{S_1(0)} r^{\ell+2}\hat r^{\otimes \ell}f(r,\hat r)\, d r\, d^2\hat r.
\end{aligned}\end{equation}

If a function can be separated into its spherical and radial components, i.e., if $f(x)$ can be expressed as $f_r(r)f_s(\phi,\theta)=f_r(r)f_s(\hat r)$ for some $f_r$ and $f_s$, its moment tensor of order $\ell$ satisfies
\begin{equation} \label{radialSphericalFunction}\begin{aligned} 
 {}{^\ell}M &\overset{\eqref{sphericalMoment}}=  \int_0^{1}\int_{S_1(0)} r^{\ell+2}\hat r^{\otimes \ell}f_r(r)f_s(\hat r) \,d r\, d^2\hat r\\
 &= \int_0^1 r^{\ell+2}f_r(r)\, d r \int_{S_1(0)} \hat r^{\otimes \ell}f_s(\hat r)\, d^2\hat r.
\end{aligned} \end{equation}
and its traces become related to the moment tensor of order $\ell-2$ via
\begin{equation} \label{radialSphericalTrace}\begin{aligned} 
{}{^\ell}M^{(1,2)}_{i_3\ldots i_\ell} &%\overset{\eqref{contraction}}
= \sum_{i=1}^3{}{^\ell}M_{i,i,i_3\ldots i_\ell}\\
 &\overset{\eqref{momentTensor}}= \sum_{i=1}^3\int_{B_1(0)} x_i x_i x_{i_3}\ldots x_{i_\ell} f_r(r)f_s(\hat r)  \,d^3 x\\
 % &\quad = \int_{B_1(0)} \sum_{i=1}^3 x_i^2 x_{i_3}\ldots x_{i_\ell} f_r(r)f_s(\hat r)  \,d^3 x\\
 &\overset{\eqref{spherical}}= \int_{B_1(0)} r^2 x_{i_3}\ldots x_{i_\ell} f_r(r)f_s(\hat r)  \,d^3 x\\
 &\overset{\eqref{sphericalMoment}}= \int_0^{1}\int_{S_1(0)} r^4 x_{i_3}\ldots x_{i_\ell} f_r(r)f_s(\hat r)  \,d r\, d^2\hat r\\
 &\overset{\eqref{spherical}}= \int_0^{1}\int_{S_1(0)} r^4 r^{\ell-2}\hat r_{i_3}\ldots \hat r_{i_\ell} f_r(r)f_s(\hat r)  \,d r\, d^2\hat r\\
 % &\quad = \int_0^{1}\int_{S_1(0)} r^{\ell+2} \hat r_{i_3}\ldots \hat r_{i_\ell} f_r(r)f_s(\hat r)\, d r\, d^2\hat r\\
 % &\quad = \int_0^{1} r^{\ell+2}f_r(r)\, d r \int_{S_1(0)} \hat r_{i_3}\ldots \hat r_{i_\ell} f_s(\hat r)  \,d^2\hat r\\
 &\overset{\eqref{radialSphericalFunction}}=\frac{\int_0^{1} r^{\ell+2}f_r(r)\, d r}{\int_0^{1} r^{\ell}f_r(r)\, d r} {}{^{\ell-2}}M_{i_3\ldots i_\ell}.
\end{aligned} \end{equation}
Because of the symmetry of the moment tensors, this result for the contraction of the first two indices extends to the contraction between any pair of indices.
This means that homogeneous moment invariants that include a trace, i.e., the contraction of two indices within the same tensor, become dependent lower-order invariants. In our example from the previous subsection, all third-order invariants that start with a contraction, such as $(1,1,2)$, contain a trace. They become dependent on the first-order tensor and therefore on its invariants. Since the first-order tensor is zero, they all become zero, too. The first occurrence of this structure still contributes the information that we have a spherical-radial function, but none of the others contain any additional information. Together with the two non-vanishing moments, it yields only three degrees of information. This is not enough to recover the seven degrees of freedom (which result from the 10 degrees of freedom of a symmetric third rank tensor, less the three degrees of freedom of the rotation) that are needed to discriminate any two third-order functions.

\section{Irreducible Tensor Moment Bases}\label{s:theory}
The problem described in Sec.~\ref{s:problem} can be overcome by making use of the irreducibility of the spherical harmonics as the bases for moments~\cite{LD89}, but that comes with some inconveniences. First, in standard applications, functions are usually given in Cartesian coordinates and must be converted. Second, the spherical-harmonics bases do not form Cartesian moment tensors, and their products and contractions are computed with Clebsh-Gordan coefficients. While all of this is well-defined and well-established, it is inconvenient to implement, and unintuitive~\cite{coope1965irreducible}. Third, using spherical tensors does not automatically solve the problem of the redundancies in the set of possible contractions, requiring recasting Langbein's algorithm to the spherical basis.
We will instead make use of the advantages of spherical harmonics, without leaving the Cartesian basis, using the irreducible tensor decomposition.

\subsection{The Irreducible Tensor Decomposition} 
There is a well-known isomorphism between totally symmetric tensors and homogeneous polynomials, and between their subspaces of totally symmetric, traceless tensors and spherical harmonics~\cite{backus1970geometrical}. It allows us to identify the irreducible decomposition of polynomials into unique spherical harmonics with the irreducible decomposition of tensors into unique totally symmetric, traceless tensors. Totally symmetric, traceless tensors are also called irreducible tensors, irreducible Cartesian tensors~\cite{coope1965irreducible}, harmonic tensors~\cite{backus1970geometrical}, or deviators~\cite{hergl2020introduction}. 
We use the arguably most common term \textbf{irreducible tensor} and the symbol $H$ to refer to a totally symmetric, traceless tensor.

Every totally symmetric tensor can be decomposed into a unique series of irreducible tensors~\cite{backus1970geometrical}:
\begin{equation} \label{decomposition}\begin{aligned} 
{}{^\ell}{M}_{i_1 \ldots  i_\ell}
=
\sum_{k=0}^{\lfloor \ell/2 \rfloor}
{}{^\ell_{\ell-2k}}{H}_{(i_1 \ldots  i_{\ell-2k}} 
\underbrace{
   {\delta_{i_{\ell-2k+1}i_{\ell-2k+2}}
   \dots
   \delta_{i_{\ell-1}i_{\ell})}}
}_{k \text{ pairs of $\delta$}},
\end{aligned} \end{equation}
where we use parentheses $(i_1 \ldots  i_\ell)$ to indicate symmetrization:
\begin{equation} \begin{aligned} 
T_{(i_1 \ldots  i_\ell)} 
:= \frac1{\ell!}
\sum_{\pi \in S_\ell} 
T_{\pi(i_1)\dots\pi(i_\ell)},
\end{aligned} \end{equation}
where the sum is taken over all permutations $\pi$ in the symmetric group $S_\ell$. This ensures that $T_{(i_1 \ldots  i_\ell)}$ is fully symmetric in its $\ell$ indices.

The individual irreducible tensors are determined by removing their traces top down. 
If $\ell$ is even, we have $\frac{\ell}{2}$ components, and if $\ell$ is odd, we have $\frac{\ell-1}{2}$.

\begin{ex}\label{ex:irreducible}
Assume that we have moment tensors up to a maximal order $\ell_m=3$, then we have $6=1+1+2+2$ irreducible tensors total. For orders 0 and 1, the decompositions are trivial because the tensors are themselves traceless:
\begin{equation} \begin{aligned} 
{}{^0}M&={}{^0_0}H,&&
{}{^1}M_i&={}{^1_1}H_i.
\end{aligned} \end{equation}
For order 2, we get
\begin{equation} \begin{aligned} 
{}{^2}M_{ij}&={}{^2_2}H_{ij}+{}{_{0}^{2}}H\delta_{ij}
\end{aligned} \end{equation}
with a rank 2 and a rank 0 part
\begin{equation} \begin{aligned} 
{}{^2_2}H_{ij} &={}{^2}M_{ij} - \frac13{}{^2}M_{\iota\iota}\delta_{ij}\\
{}{_{0}^{2}}H &=\frac13{}{^2}M_{\iota\iota}.
\end{aligned} \end{equation}
%\begin{equation} \begin{aligned} 
%{}{^2_2}H_{ij} ={}{^2}M_{ij} - \frac13{}{^2}M_{\iota\iota}\delta_{ij}
%\quad\quad\quad
%{}{_{0}^{2}}H =\frac13{}{^2}M_{\iota\iota}.
%\end{aligned} \end{equation}
The factor $\frac13$ in front of the symmetrically subtracted trace ensures that the trace of the 3D tensor ${}{^2_2}H$ is zero because 
\begin{equation} \begin{aligned} 
\delta_{ij}{}{^2_2}H_{ij} &=\delta_{ij}{}{^2}M_{ij} - \frac13{}{^2}M_{\iota\iota}\delta_{ij}\delta_{ij}
% \\&={}{^2}M_{\iota\iota} - \frac13{}{^2}M_{\iota\iota}\sum_{i=1}^3\delta_{ii}
\\&={}{^2}M_{\iota\iota} - \frac13{}{^2}M_{\iota\iota}\cdot 3
\\&=0.
\end{aligned} \end{equation}
%\begin{equation} \begin{aligned} 
%\delta_{ij}{}{^2_2}H_{ij} =\delta_{ij}{}{^2}M_{ij} - \frac13{}{^2}M_{\iota\iota}\delta_{ij}\delta_{ij}
%% \\&={}{^2}M_{\iota\iota} - \frac13{}{^2}M_{\iota\iota}\sum_{i=1}^3\delta_{ii}
%={}{^2}M_{\iota\iota} - \frac13{}{^2}M_{\iota\iota}\cdot 3
%=0.
%\end{aligned} \end{equation}
For order 3, we get
\begin{equation}\label{decomp3} \begin{aligned} 
{}{^3}M_{ijk}&={}{^3_3}H_{ijk}+
{}{^3_{1}}{H}_{(i }
   \delta_{jk)}
={}{^3_3}H_{ijk}+{}{_{1}^{3}}H_{i}\delta_{jk}+{}{_{1}^{3}}H_{j}\delta_{ik}+{}{_{1}^{3}}H_{k}\delta_{ij},
\end{aligned} \end{equation}
with rank 3 and a rank 1 parts
\begin{equation} \begin{aligned} 
{}{^3_3}H_{ijk}&={}{^3}M_{ijk} - \frac15({}{^3}M_{i\iota\iota}\delta_{jk}+{}{^3}M_{\iota j\iota}\delta_{ik}+{}{^3}M_{\iota\iota k}\delta_{ij}),\\
{}{_{1}^{3}}H_{i} &= \frac15{}{^3}M_{i\iota\iota}.
\end{aligned} \end{equation}
Again, the factor $\frac15$ in front of the symmetrically subtracted trace ensures that the trace of the 3D tensor ${}{^3_3}H$ is zero because
\begin{equation} \begin{aligned} 
\delta_{ij}{}{^3_3}H_{ijk}&=\delta_{ij}{}{^3}M_{ijk} - \frac15({}{^3}M_{i\iota\iota}\delta_{ij}\delta_{jk}+{}{^3}M_{\iota j\iota}\delta_{ij}\delta_{ik}+{}{^3}M_{\iota\iota k}\delta_{ij}\delta_{ij})
\\&= {}{^3}M_{\iota\iota k} - \frac15({}{^3}M_{k\iota\iota}+{}{^3}M_{\iota k\iota}+{}{^3}M_{\iota\iota k}\cdot 3)
 \\&=0.
\end{aligned} \end{equation}
Note that we used the total symmetry. 
In total, two of the irreducible tensors are of zeroth rank, two are of first rank, one is of second rank, and one is of third rank.
\end{ex}

Bujack \etal pointed out that simultaneous invariants constitute a vulnerability of a basis to degenerate functions~\cite{bujack2022systematic}. But now we see that using the moment tensors as wholes is similar to using simultaneous invariants if we view a moment tensor as a combination of its irreducible tensors. We will now show how to overcome this vulnerability by following their general idea, but instead of applying it to the tensors themselves, we apply it to the irreducible tensors of the moment tensors. In other words, we will compute a moment basis from the elements of the irreducible tensor decomposition of the Cartesian moment tensors.
They construct bases that are flexible with respect to functions with vanishing moment tensors by using the homogeneous ones together with three simultaneous invariants of one selected non-vanishing tensor in combination with each other tensor. We will show that the same approach extends to constructing a basis from the irreducible tensors.

\subsection{Invariants of the Irreducible Tensors of a Single Moment Tensor}
\label{subsection:invariants-counts}
We will use the terms \textbf{homogeneous} and \textbf{simultaneous} as before to refer to invariants that contain only one moment tensor or multiple moment tensors, respectively, and introduce the terms \textbf{pure} and \textbf{mixed} to refer to invariants that contain only one irreducible tensor or multiple irreducible tensors, respectively.

\begingroup
\setlength{\tabcolsep}{5pt}
\renewcommand{\arraystretch}{1} % single spacing
\begin{table}%[ht]
 \begin{small}
 \begin{center}
\begin{tabular}{|c|c|c|c|c|c|c|} % or whatever you want for the non-header row
%\begin{tabular}{|p{1.0cm}|p{1.3cm}|p{1.3cm}|p{1.3cm}|p{1.4cm}|p{1.4cm}|p{1.3cm}|}
\hline
\multicolumn{1}{|>{\RaggedRight\arraybackslash}p{1.0cm}|}{order} & 
\multicolumn{1}{|>{\RaggedRight\arraybackslash}p{1.3cm}|}{moments of an irr.\ tensor} & 
\multicolumn{1}{|>{\RaggedRight\arraybackslash}p{1.3cm}|}{invariants of an irr.\ tensor} &  
\multicolumn{1}{|>{\RaggedRight\arraybackslash}p{1.3cm}|}{mom.\ of a mom.\ tensor} & 
\multicolumn{1}{|>{\RaggedRight\arraybackslash}p{1.4cm}|}{inv.\ of a moment tensor} & 
\multicolumn{1}{|>{\RaggedRight\arraybackslash}p{1.4cm}|}{pure inv.\  of a mom.\ tensor} & 
\multicolumn{1}{|>{\RaggedRight\arraybackslash}p{1.3cm}|}{mixed inv.\ needed} \\
\hline
%\multicolumn{1}{|>{\RaggedRight\arraybackslash}p{1.0cm}|}{\scriptsize{$\ell$}}  & 
{\scriptsize{$\ell$}} &
%\multicolumn{1}{|>{\RaggedRight\arraybackslash}p{1.3cm}|}
{\scriptsize{$2\ell+1$}} & 
%\multicolumn{1}{|>{\RaggedRight\arraybackslash}p{1.3cm}|}{\scriptsize{$2\ell -2$}}  &
{\scriptsize{$2\ell -2$}}  & 
\multicolumn{1}{|>{\Centering\arraybackslash}p{1.3cm}|}{\scriptsize{$(\ell+1)$ $(\ell+2)/2$}}    &
%{\scriptsize{$(\ell+1)$$(\ell+2)/2$}}    &
\multicolumn{1}{|>{\Centering\arraybackslash}p{1.4cm}|}{\scriptsize{$(\ell+1)(\ell+2)/2-3$}}    & 
\multicolumn{1}{|>{\Centering\arraybackslash}p{1.4cm}|}{\scriptsize{$(\ell^2+2-\ell\mod 2) /2$}} & 
\multicolumn{1}{|>{\Centering\arraybackslash}p{1.3cm}|}{\scriptsize{$(3\ell-6+\ell\mod 2)/2$}}    \\
\hline
{0} & 1    & \cellcolor{black!20}1   & 1   & \cellcolor{black!20}1 & 1   & \cellcolor{black!20}0 \\
{1} & 3    & \cellcolor{black!20}1   & 3   & \cellcolor{black!20}1 & 1   & \cellcolor{black!20}0 \\
{2} & 5    & 2   & 6   & 3 & 3   & 0 \\
{3} & 7    & 4   & 10  & 7 & 5   & 2 \\
{4} & 9    & 6   & 15  & 12    & 9   & 3 \\
{5} & 11   & 8   & 21  & 18    & 13  & 5 \\
% {6} & 13   & 10  & 28  & 25    & 19  & 6 \\
% {7} & 15   & 12  & 36  & 33    & 25  & 8 \\
% {8} & 17   & 14  & 45  & 42    & 33  & 9 \\
% {9} & 19   & 16  & 55  & 52    & 41  & 11    \\
\hline
   % &  & B-3 & & D-3   & & E-F   \\
 {Ref.}  & 
 {Lem.~\ref{l:dof}}  & 
 {Lem.~\ref{l:pure}}   & 
 {Lem.~\ref{l:dof2}} &  
 {Lem.~\ref{l:dof2}}  & 
 {Thm.~\ref{t:pure2}}   & 
 {Lem.~\ref{l:mixed}} \\
\hline
\end{tabular}
\end{center}
\end{small}
\caption{Summary of the degrees of freedom and numbers of invariants for irreducible tensors, moment tensors, and their irreducible tensor decompositions, as a function of the order. The gray cells indicate the exceptions to the general rule in the first row. 
%The last row contains the references to the respective lemmata and theorems.
\label{t:l}}
\end{table}
\endgroup

In these terms, we demonstrate that a basis of homogeneous invariants of a given moment tensor of order $\ell$ can be constructed from the pure invariants of its irreducible tensors, plus three mixed invariants combining a selected non-vanishing irreducible tensor with each of the other irreducible tensors from the irreducible tensor decomposition, Theorem~\ref{t:pure2}.
We show this in the remainder of this subsection by proving that the degrees of freedom add up to the total number of invariants; all proofs are given in \ref{appendix:count-proofs}. Table~\ref{t:l} lists the degrees of freedom and references the relevant lemmata and theorems.

\begin{restatable}[Independent moments of an irreducible tensor]{lem}{ldof}\label{l:dof}
The number of independent moments of an irreducible tensor $H$ of rank $\ell$ is $2\ell+1$.
\end{restatable}

\begin{restatable}[Independent moments of the decomposition]{lem}{ldoftwo}\label{l:dof2}
The number of independent moments of a moment tensor equals the number of independent moments of its irreducible tensor decomposition, namely $\frac{(\ell+1)(\ell+2)}2$.
\end{restatable}

\begin{restatable}[Invariants of an irreducible tensor]{lem}{lpure}\label{l:pure}
The number of independent invariants of an irreducible tensor $H$ of rank $\ell > 1$ is $2\ell-2$.
For $\ell=0$ and $\ell=1$, the number of independent pure invariants of an irreducible tensor is one.
\end{restatable}

\begin{restatable}[Pure invariants of a decomposition]{thm}{tpuretwo}\label{t:pure2}
The number of independent pure invariants of an irreducible tensor decomposition of a moment tensor of rank $\ell > 1$ is $\frac{\ell^2}2+1$ if $\ell$ is even and $\frac{\ell^2+1}2$ if $\ell$ is odd.
For $\ell=0$ and $\ell=1$, each have one pure invariant.
\end{restatable}

\begin{restatable}[Mixed homogeneous invariants of a decomposition]{lem}{lmixed}\label{l:mixed}
If we use all pure invariants of the irreducible tensors in the decomposition~\eqref{decomposition} of a moment tensor of order $\ell>1$, we are $\frac{3\ell-6}2$ degrees of freedom short for even $\ell$, and $\frac{3\ell-5}2$ for odd $\ell$, to get to the total number of independent homogeneous invariants.
%In analogy to the flexible basis by Bujack \etal, 
We can produce these missing degrees of freedom by choosing one rank $p_0>1$ such that ${}{^\ell_{p_0}}H\neq 0$ and adding three mixed invariants of it paired with any other irreducible tensor of rank $p>1$ if $\ell$ is even, or two mixed invariants of it paired with ${}{^\ell_{1}}H$ if $\ell$ is odd.
For $\ell=0$ and $\ell=1$, we need no mixed invariants.
\end{restatable}

The fact that we need three mixed invariants for $p>1$ and two for $p=1$ is not happenstance. These numbers are identical to the degrees of freedom of rotation in 3D and are necessary to couple the pure invariants of each pair orientationally to each other. We can think of it this way: without the mixed invariants, we can always construct functions that consist of two independent pure invariant parts. We can rotate these parts relative to each other, producing arbitrarily many new functions that cannot be distinguished by the pure invariants. The three mixed invariants couple the pure ones such that any relative rotation between them is not lost. 

% % In analogy to the flexible basis by Bujack \etal, we can produce these missing degrees of freedom by picking one rank $p_0>1$ such that ${}{^\ell_{p_0}}H\neq 0$ and adding 3 mixed invariants of it paired with any other irreducible tensor of rank $p>1$ if $\ell$ is even because
% % \begin{equation}
% % \sum_{q=2}^{\ell/2}3 = \frac{3p-6}2.
% % \end{equation}
% % If $\ell$ is odd, we can produce these missing degrees of freedom by picking one rank $p_0>1$ such that ${}{^\ell_{p_0}}H\neq 0$ and adding 3 mixed invariants of it paired with any other irreducible tensor of rank $p>1$ plus 2 mixed invariants of it paired with ${}{^\ell_{1}}H$ because
% % \begin{equation}
% % 2+\sum_{q=2}^{(\ell_m-1)/2}3 = \frac{3\ell-5}2.
% % \end{equation}

\subsection{Invariants of the Irreducible Tensors of the Moment Tensors}
Next, we will not only consider the irreducible tensor decomposition of a single moment tensor, but all irreducible tensors ${}{^{\ell}_{p}}H$, $p=o,\ell-2,\ldots $, from the decompositions of all moment tensors ${}{^{\ell}}M$ up to a maximum order $\ell\leq \ell_m\in \mathbb N$, simultaneously.

\begingroup
\setlength{\tabcolsep}{6pt}
\renewcommand{\arraystretch}{1} % single spacing
\begin{table}%[ht]
 \begin{small}
\begin{center}
\begin{tabular}{|c|c|c|c|c|} % or whatever you want for the non-header row
%\begin{tabular}{|p{0.7cm}|p{1.7cm}|p{2.2cm}|p{3cm}|p{2.3cm}|}\hline
\hline
\multicolumn{1}{|>{\RaggedRight\arraybackslash}p{0.7cm}|}{max.\ order} & 
\multicolumn{1}{|>{\RaggedRight\arraybackslash}p{1.8cm}|}{moments up to max.\ order} & 
\multicolumn{1}{|>{\RaggedRight\arraybackslash}p{2.3cm}|}{invariants up to max.\ order} & 
\multicolumn{1}{|>{\RaggedRight\arraybackslash}p{2.4cm}|}{pure inv.\ up to max.\ order} & 
\multicolumn{1}{|>{\RaggedRight\arraybackslash}p{2.3cm}|}{mixed invariants needed} \\
\hline
% 0  & 1 & 1 &\cellcolor{black!10} 1 &\cellcolor{black!10} 1 & \cellcolor{black!10} 1 &\cellcolor{black!10} 0   \\
% 1  & 3 & 4 &\cellcolor{black!10} 1 &                     2 & \cellcolor{black!10} 3 &\cellcolor{black!10} 0    \\\hline
% $\ell,\ell_m$&$(\ell+1)(\ell+2)/2$  &$(\ell+1)(\ell+2)/2-3$                 & $(\ell^2+^2_1)/2$             & $(\ell_{m}+1)(\ell_{m}+2)(\ell_{m}+3)/6-3$   & $2\ell_m^3+2\ell_m^2+10\ell_m+^12_9$       &   $(3\ell_m^2+4\ell_m-^{12}_{11})/4 $     \\\hline
% % (1)    & (2)         & (3)                            & (4)                      & (5)                    & (6)             & (7)                           & (8) \\\hline
%          & Lem.~\ref{l:dof2} & Lem.~\ref{l:dof2}$-3$              & Lem.~\ref{l:pure2}   & $\sum $Lem.~\ref{l:dof2} $-3$ & $\sum$ Lem.~\ref{l:pure2}          & Theorem~\ref{t:flexible}      
\multicolumn{1}{|>{\Centering\arraybackslash}p{0.7cm}|}{\scriptsize{$\ell_m$}} & 
\multicolumn{1}{|>{\Centering\arraybackslash}p{1.8cm}|}{\scriptsize{$(\ell_{m}+1)(\ell_{m}+2)\linebreak(\ell_{m}+3)/6$}} & 
\multicolumn{1}{|>{\Centering\arraybackslash}p{2.3cm}|}{\scriptsize{$(\ell_{m}+1)(\ell_{m}+2)\linebreak(\ell_{m}+3)/6-3$}} & \multicolumn{1}{|>{\Centering\arraybackslash}p{2.4cm}|}{\scriptsize{$(2\ell_m^3+3\ell_m^2+10o_m+12-3(\ell_m\mod2))/12$}} & 
\multicolumn{1}{|>{\Centering\arraybackslash}p{2.3cm}|}{\scriptsize{$(3\ell_m^2+4\ell_m-12+\linebreak(\ell_m\mod2))/4$}} \\
\hline
0     & 1       & \cellcolor{black!20}1  & 1   & \cellcolor{black!20}0    \\
1     & 4       & \cellcolor{black!20}2  & 2   & \cellcolor{black!20}0    \\
2     & 10  & 7     & 5      & 2       \\
3     & 20  & 17    & 10     & 7       \\
4     & 35  & 32    & 19     & 13      \\
5     & 56  & 53    & 32     & 21      \\
% 6     & 84  & 81    & 51     & 30      \\
% 7     & 120 & 117   & 76     & 41      \\
% 8     & 165 & 162   & 109    & 53      \\
% 9     & 220 & 217   & 150    & 67      \\
\hline
      % & Sum D      & D-3   & Sum F  & H-i     \\
Ref.  & 
Lem.~\ref{l:dof2}    &
Lem.~\ref{l:dof2} &
Thm.~\ref{t:pure2} &
Thm.~\ref{t:flexible}      
\\\hline
\end{tabular}
\end{center}
\end{small}
\caption{Summary of the degrees of freedom and numbers of invariants for moment tensors and their irreducible tensor decompositions depending on the maximal order. The gray cells indicate the exceptions from the general rule in the first row. 
% The last row contains the references to the respective lemmata and theorems.
\label{t:lm}}
\end{table}
\endgroup

In this subsection we show that a basis of invariants of moment tensors up to a given maximal order $\ell_m$ can be constructed from the pure invariants of the irreducible tensors of all tensors plus three mixed invariants between a selected non-vanishing irreducible tensor in combination with all other irreducible tensors from all irreducible tensor decompositions, Theorem~\ref{t:flexible}.  The proof is given in \ref{appendix:flexible}.
Table~\ref{t:lm} lists the degrees of freedom and references the relevant lemmata and theorems.

\begin{restatable}[Flexible basis]{thm}{tflexible}\label{t:flexible}
For any chosen irreducible tensor ${}{^{\ell_0}_{p_0}}H$ of rank $p_0>1$ and order $1<\ell_0\leq \ell_m$, we can find a basis of invariants using all pure invariants, three mixed invariants for each combination of ${}{^{\ell_0}_{p_0}}H$ with all other irreducible tensors ${}{^{\ell}_{p}}H$ with $p>1$, plus two mixed invariants for each combination with $p=1$. For $\ell_m=1$ the pure invariants are a basis.
\end{restatable}

\subsection{Example of Flexibility}
We now revisit the example from Sec.~\ref{l:example} and show that the irreducible tensor bases up to third order are able to discriminate the two functions in that example. Since the zeroth-to-second order moments of both functions vanish, it is sufficient to look at the third-order tensor. Their irreducible tensor decompositions become trivial because both are traceless:
\begin{equation} \begin{aligned} 
{}{^3}M_{ijk}&={}{^3_3}H_{ijk}+
{}{^3_{1}}{H}_{(i }
   \delta_{jk)}
={}{^3_3}H_{ijk}+0_{i}\delta_{jk} .
\end{aligned} \end{equation}
% \begin{equation} \begin{aligned} 
% {}{^3}M\overset{\eqref{decomposition}}={}{^\ell_\ell}H+{}{_{1}^{3}}H\delta \overset{\eqref{decomposition2}}= {}{^3}M + (0,0,0)^T\delta.
% \end{aligned} \end{equation}
Consequently, the pure invariant of the first-rank irreducible tensor and the two mixed invariants between ranks three and one are zero. As expected, the seven invariants of the third-order moment tensor reduce to the four invariants of the third-rank irreducible tensor,  Table~\ref{t:l}.
%.  Please cross check Table~\ref{t:l} for these numbers. 

Langbein's algorithm provides us with four independent invariants of the third-rank irreducible tensor. Fig.~\ref{f:irreducibleTensor} shows a graph of the homogeneous third-order invariants of the irreducible tensor decomposition. The first three are the same for both functions: $14\left(\frac{8\pi}{315}\right)^2$, $92\left(\frac{8\pi}{315}\right)^4$, and $32\left(\frac{8\pi}{315}\right)^6$,
% \begin{equation}\begin{aligned}\label{irreducibleTensor}
% {}{^3}{M}{^2}(1,2,3) (1,2,3)&=14\left(\frac{8\pi}{315}\right)^2\\
% {}{^3}{M}{^4}(1,2,3) (1,2,4) (3,5,6) (4,5,6)&=92\left(\frac{8\pi}{315}\right)^4\\
% {}{^3}{M}{^6}(1, 2, 3) (2, 3, 4) (1, 4, 5) (5, 6, 7) (7, 8, 9) (6, 8, 9)&=32\left(\frac{8\pi}{315}\right)^6,
% \end{aligned}\end{equation}
but the fourth one
% ,
% {\footnotesize
% \begin{equation}
% \begin{aligned}
% \label{irreducibleTensor2}
% {}^{3}M^{10}(1,2,3)(2,3,4)(1,4,5)(5,6,7)(6,7,8)(8,9,10)(9,10,11)(11,12,13)(13,14,15)(12,14,15)
% \end{aligned}
% \end{equation}
% }
distinguishes them, taking the values $1418\left(\frac{8\pi}{315}\right)^{10}$ for $f_1$ and $1152\left(\frac{8\pi}{315}\right)^{10}$ for $f_2$. The following ones are zero.

\begin{figure}%[ht]
\centering
{\includegraphics[width=0.14\linewidth]{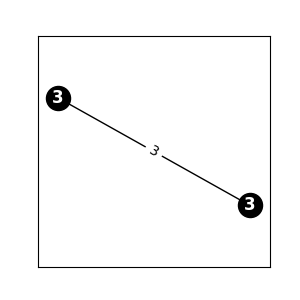}}\hfill
{\includegraphics[width=0.14\linewidth]{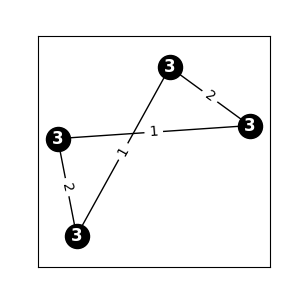}}\hfill
{\includegraphics[width=0.14\linewidth]{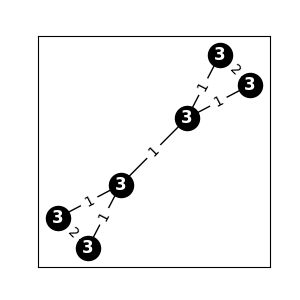}}\hfill
{\includegraphics[width=0.14\linewidth]{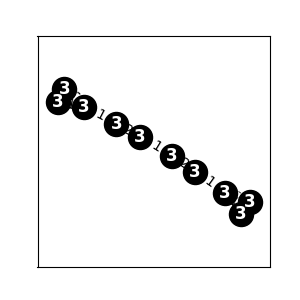}}\hfill
{\includegraphics[width=0.14\linewidth]{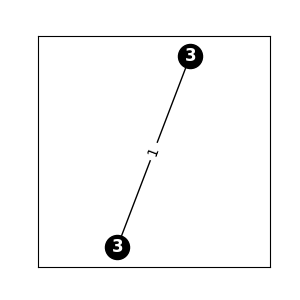}}\hfill
{\includegraphics[width=0.14\linewidth]{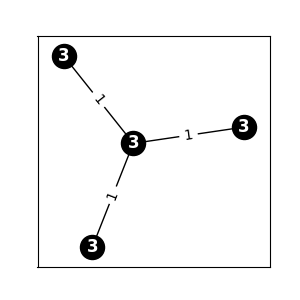}}\hfill
{\includegraphics[width=0.14\linewidth]{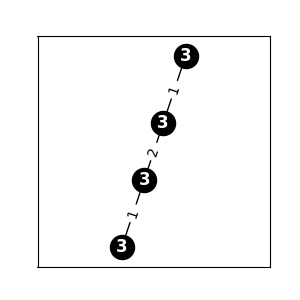}}\hfill\hfill
\caption{A set of independent homogeneous 3. order invariants in the basis formed by the irreducible tensor decomposition. The nodes all show order 3. The number of incident edges reveals their rank. The first four are the non-zero entries. The last three involve the rank 1 irreducible tensor, which is zero in this case. \label{f:irreducibleTensor}} 
\end{figure}

\section{The Generation of Optimal Flexible Descriptor Sets}\label{s:sets}
\begin{wrapfigure}{r}{0.5\textwidth}
  \centering
  \hspace{0.1in}
  \begin{minipage}[t]{0.21\textwidth}
    \centering
    \includegraphics[width=\linewidth]{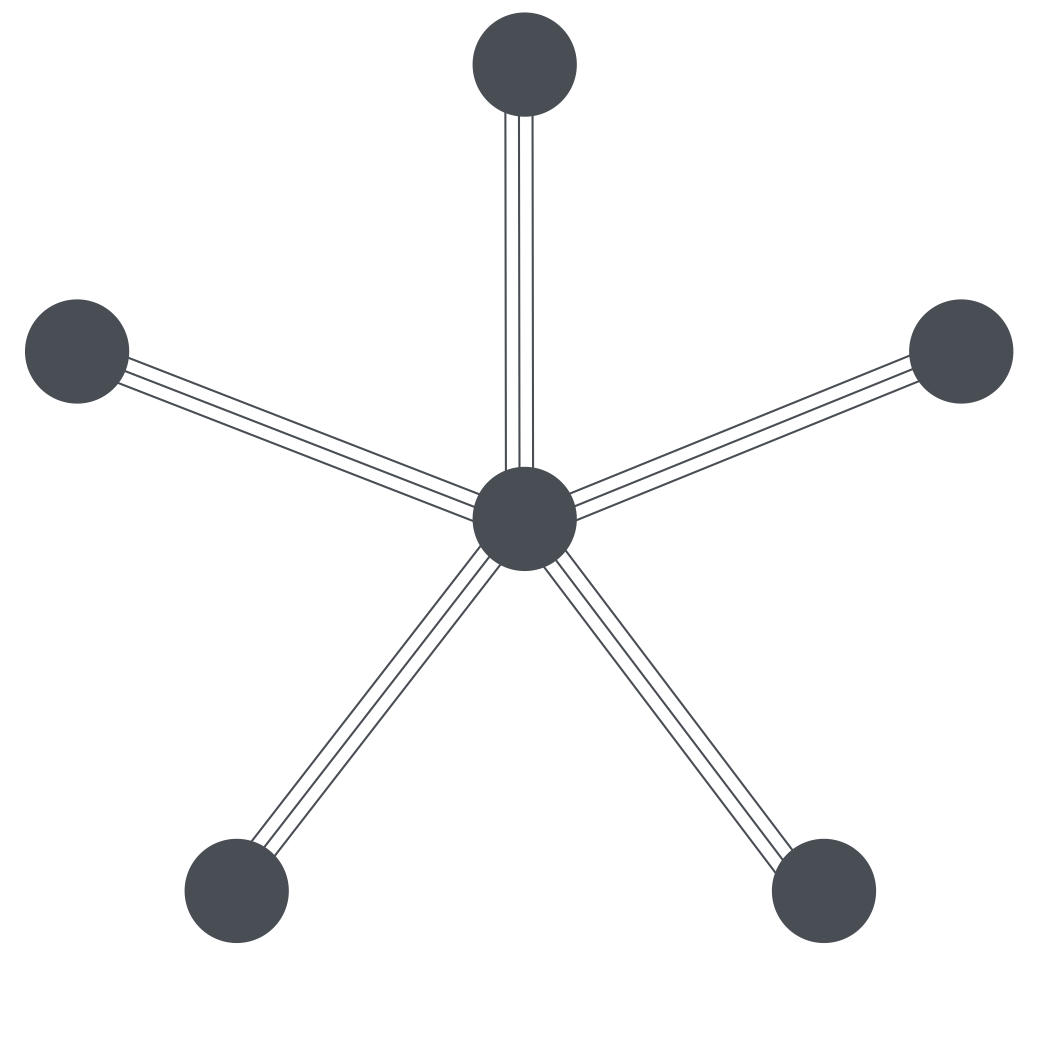}
    \captionof{figure}{A specific flexible set relies on a non-zero irreducible tensor.\label{f:featureSets1}}
    \label{f:specific}
  \end{minipage}
  %\hfill
  \hspace{0.1in}
  \begin{minipage}[t]{0.21\textwidth}
    \centering
    \includegraphics[width=\linewidth]{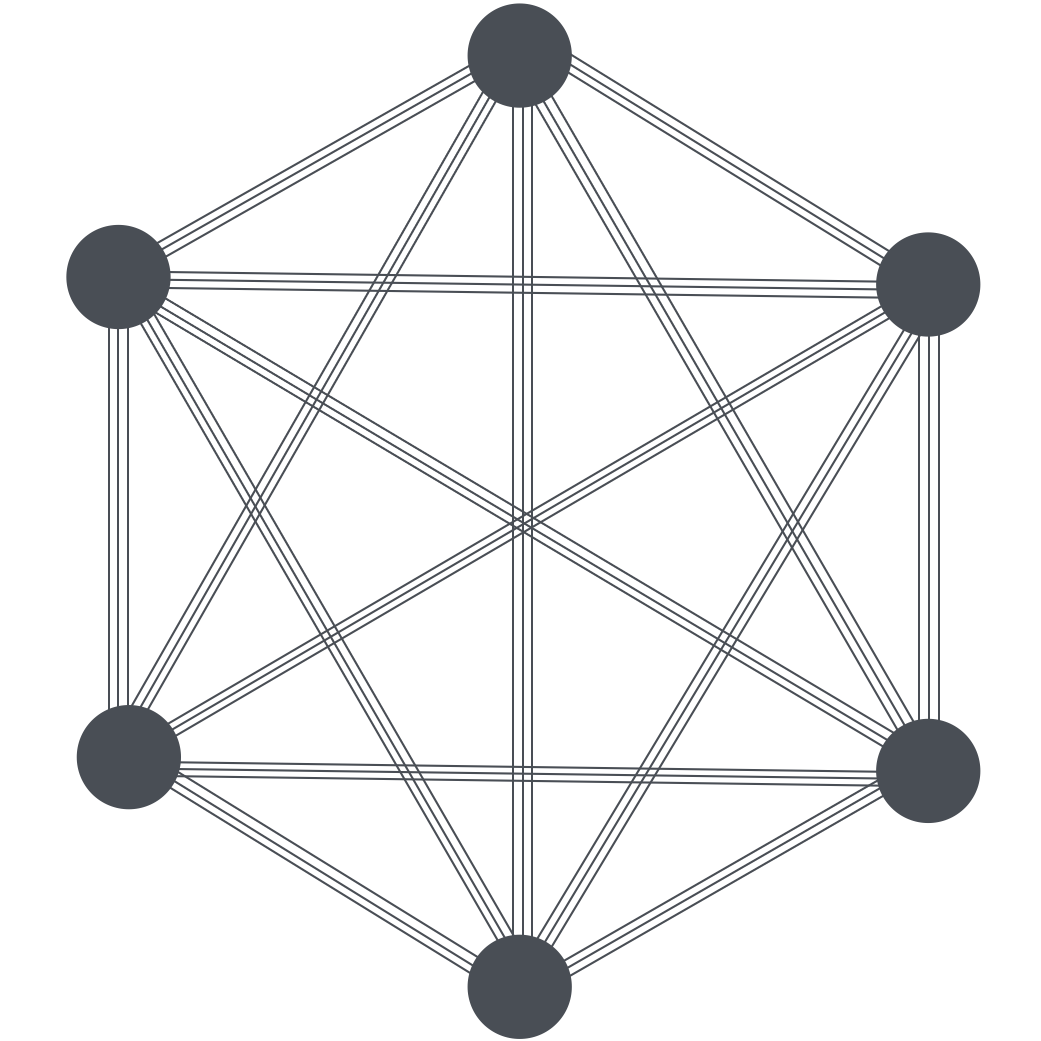}
    \captionof{figure}{A minimal flexible set is always complete.\label{f:featureSets2}}
    \label{f:minimal}
  \end{minipage}
\end{wrapfigure}
A useful way to think about the theory from the previous section, and in particular the main result, Thm.~\ref{t:flexible}, and the way in which we generate flexible descriptor sets, is that they consist of two parts. First, the pure invariants fully describe each irreducible tensor in an invariant way, and second, the mixed invariants anchor the mutual orientations between them.  Because the number of degrees of freedom of a 3D rotation is three, we typically need three mixed invariants for each anchor (with the exception of first order), and if we have $n$ irreducible tensors, then we need $n-1$ anchors. Fig.~\ref{f:featureSets1} and Fig.~\ref{f:featureSets2} illustrate two different ways in which this anchoring can be achieved, each optimal in its own way, as we will show.

\subsection{The Specific Flexible Basis}\label{s:specific}
There are applications in which a non-vanishing irreducible tensor can be preselected, e.g., in pattern detection applications once the pattern is known~\cite{bujack2017flexible}. 
% An irreducible tensor is robust if its norm is large and its rank is low. For example, the smallest rank-irreducible tensor that exceeds the average norm of all irreducible tensors present is a good candidate. 
It should have rank of at least two, though, to not fall into the category of exceptions for rank 0 and rank 1 tensors.   

The presence of a non-vanishing irreducible tensor allows setting the $n-1$ anchors by placing it as the center and connecting it to every other irreducible tensor, as illustrated in Fig.~\ref{f:specific}. 

The following algorithm computes a specific flexible basis if a non-vanishing irreducible tensor $D_r$ can be preselected. 
\begin{enumerate}
    \item Compute moment tensors up to a given maximal order $l_m$.
    \item Compute the irreducible tensoric decompositions of all tensors by iteratively subtracting the trace.
    % \item Enumerate them, i.e., choose an order in which the following steps are applied.
    \item For each irreducible tensor, apply Langbein's algorithm to its powers (sorted increasingly by exponent) to compute its independent pure invariants up to the number given in Table~\ref{t:l}, and add them to the basis.
    \item Select the non-zero irreducible tensor $D_r$.
    \item For each other irreducible tensor $D_i,i\neq m$, prefill the rows of the Jacobian with the derivatives of the pure invariants already selected for $D_i$ and $D_r$ and apply Langbein's algorithm to all products of $D_i$ and $D_r$ until three mixed invariants are found and add them to the basis.
\end{enumerate}

\begin{ex}\label{ex:22}
Let ${}{^0}M,{}{^1}M,{}{^2}M,{}{^3}M$ be the moment tensors up to order $l_m=3$. Example~\ref{ex:irreducible} shows the six irreducible tensors in the decompositions: ${}{^0_0}H$, ${}{^1_1}H$, ${}{^2_2}H$, ${}{_{0}^{2}}H$, ${}{^3_3}H$, ${}{_{1}^{3}}H$. There are always 10 pure invariants, depicted in Fig.~\ref{f:exPure}. They do not depend on the choice of a non-vanishing irreducible tensor. 
Let us assume a scenario in which the irreducible tensor ${}{^2_2}H$ is non-zero. Then we add the seven mixed invariants that contain ${}{^2_2}H$, depicted in Fig.~\ref{f:exMixed1}. This totals 17 invariants, which correctly matches the predicted number from Table~\ref{t:lm}. The complete list of invariants, with their contraction information as formulas, are given in \ref{appendix:example-mixed}.

\begin{figure}%[ht]
\centering
{\includegraphics[width=0.19\linewidth]{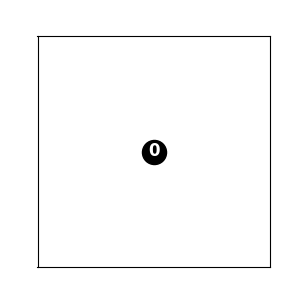}}\hfill
{\includegraphics[width=0.19\linewidth]{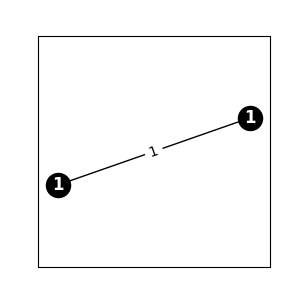}}\hfill
{\includegraphics[width=0.19\linewidth]{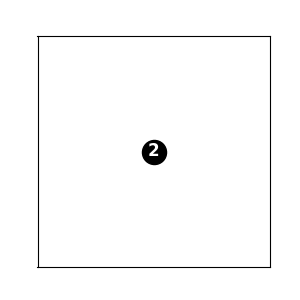}}\hfill
{\includegraphics[width=0.19\linewidth]{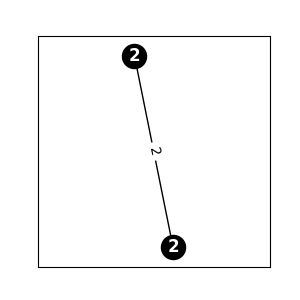}}\hfill
{\includegraphics[width=0.19\linewidth]{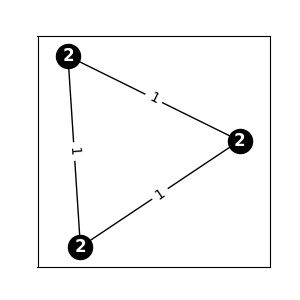}}\\
{\includegraphics[width=0.19\linewidth]{pics/pure_1_1_0.png}}\hfill
{\includegraphics[width=0.19\linewidth]{pics/pure_3_1_0.png}}\hfill
{\includegraphics[width=0.19\linewidth]{pics/pure_3_3_0.png}}\hfill
{\includegraphics[width=0.19\linewidth]{pics/pure_3_3_2.png}}\hfill
{\includegraphics[width=0.19\linewidth]{pics/pure_3_3_3.png}}
\caption{The 10 pure invariants up to $l_m=3$: one of ${}{^0}M={}{^0_0}H$, one of ${}{^1}M={}{^1_1}H$, three of ${}{^2}M$ (one of ${}{_{0}^{2}}H$ and two of ${}{^2_2}H$), and five of ${}{^3}M$ (one of ${}{_{1}^{3}}H$ and four of ${}{_{3}^{3}}H$).\label{f:exPure}} 
\end{figure}

\begin{figure}%[ht]
\centering
{\includegraphics[width=0.14\linewidth]{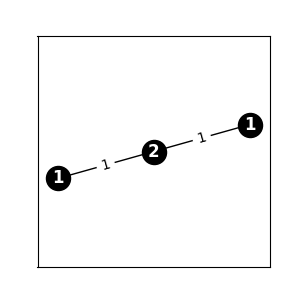}}\hfill
{\includegraphics[width=0.14\linewidth]{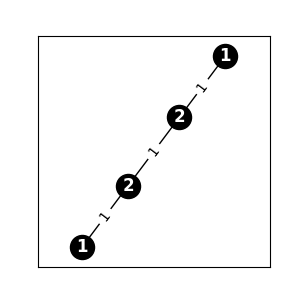}}\hfill
{\includegraphics[width=0.14\linewidth]{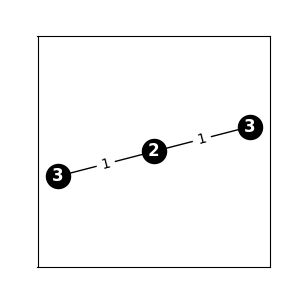}}\hfill
{\includegraphics[width=0.14\linewidth]{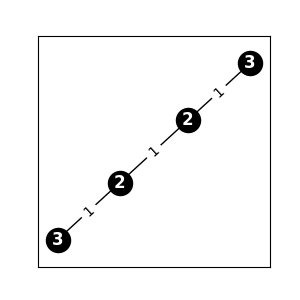}}\hfill
{\includegraphics[width=0.14\linewidth]{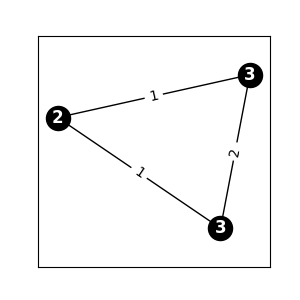}}\hfill
{\includegraphics[width=0.14\linewidth]{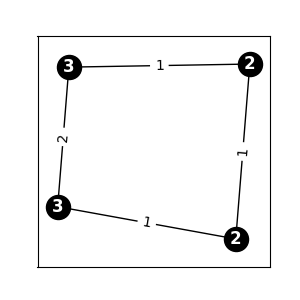}}\hfill
{\includegraphics[width=0.14\linewidth]{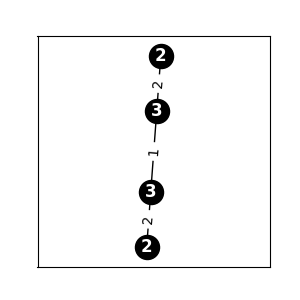}}
\caption{The seven mixed invariants up to $l_m=3$ centered around ${}{^2_2}H$: two with ${}{^1_1}H$, two with ${}{_{1}^{3}}H$, and three with ${}{_{3}^{3}}H$.\label{f:exMixed1}} 
\end{figure}
\end{ex}

% \begin{ex}\label{ex:33}
% If, on the other hand, we had chosen ${}{^3_3}H$ to be the robustly non-zero irreducible tensor, we would have added the seven mixed invariants containing it, depicted in Fig.~\ref{f:exMixed2}, to the pure invariants from Fig.~\ref{f:exPure}.

% \begin{figure}[ht]

% \centering
% {\includegraphics[width=0.14\linewidth]{pics/mixed_1_1_3_3_0.png}}\hfill
% {\includegraphics[width=0.14\linewidth]{pics/mixed_1_1_3_3_1.png}}\hfill
% {\includegraphics[width=0.14\linewidth]{pics/mixed_3_1_3_3_0.png}}\hfill
% {\includegraphics[width=0.14\linewidth]{pics/mixed_3_1_3_3_1.png}}\\ \hfill
% {\includegraphics[width=0.14\linewidth]{pics/mixed_2_2_3_3_0.png}}\hfill
% {\includegraphics[width=0.14\linewidth]{pics/mixed_2_2_3_3_1.png}}\hfill
% {\includegraphics[width=0.14\linewidth]{pics/mixed_2_2_3_3_2.png}}\hfill\hfill
% \caption{The seven mixed invariants up to $l_m=3$ centered around ${}{^3_3}H$: 2 with ${}{^1_1}H$, 2 with ${}{_{1}^{3}}H$, and 3 with ${}{_{2}^{2}}H$.\label{f:exMixed2}} 
% \end{figure}
% \end{ex}

\subsection{The Minimal Flexible Set}\label{s:minimal}
There are applications that do not allow preselecting one irreducible tensor, such as clustering~\cite{bujack2015clustering} and machine learning~\cite{allen2025optimal}, that deal with a large set of input functions rather than one specific one. To guarantee flexibility for any input function, Bujack \etal suggested using the smallest set that guarantees full flexibility, instead of a basis, which they called it an overcomplete set~\cite{bujack2022systematic}. We will call it the \emph{minimal flexible set} to avoid confusion with other overcomplete sets such as the atomic cluster expansion set (ACE)~\cite{AtomicClusterExpansion}. The minimal flexible set is not a basis in the proper sense because it is dependent, but it is the smallest set that is complete for any input function, containing pairwise combinations of all moment tensors. We now extend this idea from moment tensors to their individual irreducible components.

In the anchoring metaphor, the minimal flexible set corresponds to mutually anchoring every irreducible tensor to every other irreducible tensor, illustrated in Fig.~\ref{f:minimal}. No matter which irreducible tensors vanish, the mutual orientation information is guaranteed to be preserved between the non-zero irreducible tensors.

To build a minimal flexible set, we combine each irreducible tensor with each other irreducible tensor. Pairing two tensors of ranks greater than one always has three independent invariants. Pairing a tensor of rank greater than one with a tensor of rank one yields two independent invariants. Pairing two first-rank tensors yields one independent invariant, Table~\ref{t:lm}. 

For $l_m$ even, we have $\frac {l_m} 2$ even-rank tensors larger than zeroth rank and $\frac {l_m} 2$ odd-rank tensors.
If $l$ is even, the irreducible decomposition has $\frac l2$ components larger than zeroth rank, and if $l$ is odd, it has $\frac{l-1}2$ of rank larger than one and one of rank one. This means we have $n_1=\frac {l_m} 2$ irreducible tensors of first rank, and 
\begin{equation}
n_2=\sum_{q=1}^{\frac {l_m}2}\frac{2q}2+\sum_{q=1}^{\frac {l_m}2}\frac{(2q+1)-1}2=\frac{l_m(l_m + 2)}4.
\end{equation}
The total number of invariants in the minimal flexible set is then the number of pure invariants (from the third column of Table~\ref{t:lm}), plus one times the combination of every first-rank irreducible tensor with every other first-rank tensor, plus twice the combination of every first-rank irreducible tensor with every tensor of rank greater than one, plus three times the combination of all irreducible tensors with rank greater than one with each other, which totals
\begin{equation}\begin{aligned}
&\frac{2l_m^3+3l_m^2+10l_m+12}{12} +\frac{n_1(n_1-1)}2+2n_1n_2+3\frac{n_2(n_2-1)}2 \\= &\frac1{96}(9l_m^4 + 76l_m^3 + 84l_m^2 - 16l_m + 96).
\end{aligned}\end{equation}
The behavior with odd $l_m$ is similar. Here we have $\frac {l_m-1} 2$ even-rank tensors larger than zeroth rank and $\frac {l_m-1} 2$ odd-rank tensors larger than first rank. Therefore, we get $n_1=\frac {l_m+1} 2$ and $n_2=((l_m - 1)*(l_m + 1))/4$, which together with the pure invariants gives
\begin{equation}\begin{aligned}
&\frac{2l_m^3+3l_m^2+10l_m+9}{12} +\frac{n_1(n_1-1)}2+2n_1n_2+3\frac{n_2(n_2-1)}2 
\\= &\frac1{96}(9l_m^4 + 40l_m^3 + 6l_m^2 + 56l_m + 81).
\end{aligned}\end{equation}
This shows that the minimal flexible set between irreducible tensors constitutes a computational cost going from $\Theta(l_m^3)$, as shown in the third column of Table~\ref{t:lm}, to $\Theta(l_m^4)$, and so might not be feasible for all applications.
The first values for $l_m=0,\ldots ,5$ for the specific flexible basis are $1,2,7,17,32,53$; the numbers for the minimal flexible set are $1,2,12,22,89,116$.

\begin{ex}\label{ex:minimal}
Assume a scenario in which none of the irreducible tensors are guaranteed to always be non-zero. Then we add all individually-independent sets of mixed invariants between two irreducible tensors to the 10 pure invariants from Fig.~\ref{f:exPure}. The 12 mixed invariants of the minimal flexible set are depicted in Fig.~\ref{f:exMixed3}. This leads to a total of 22 invariants that allow flexibility, i.e., five additional invariants more than a basis.
The complete list of invariants with their contractions as formulas are given in~\ref{appendix:example-minimal}.

\begin{figure}%[ht]
\centering
{\includegraphics[width=0.16\linewidth]{pics/mixed_1_1_2_2_0.png}}\hfill
{\includegraphics[width=0.16\linewidth]{pics/mixed_1_1_2_2_1.png}}\hfill
{\includegraphics[width=0.16\linewidth]{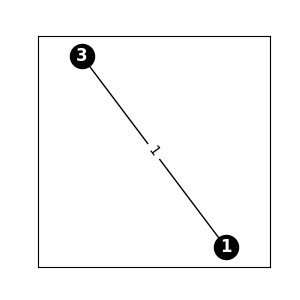}}\hfill
{\includegraphics[width=0.16\linewidth]{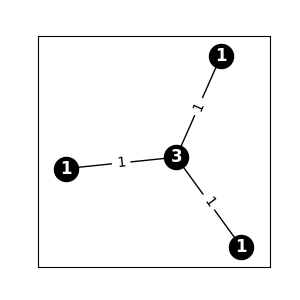}}\hfill
{\includegraphics[width=0.16\linewidth]{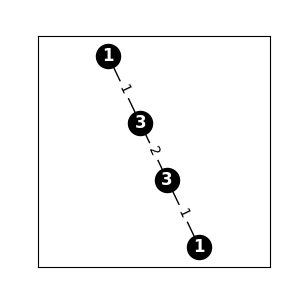}}\hfill
{\includegraphics[width=0.16\linewidth]{pics/mixed_2_2_3_3_0.png}}\hfill
{\includegraphics[width=0.16\linewidth]{pics/mixed_2_2_3_3_1.png}}\hfill
{\includegraphics[width=0.16\linewidth]{pics/mixed_2_2_3_3_0.png}}\hfill
{\includegraphics[width=0.16\linewidth]{pics/mixed_2_2_3_3_1.png}}\hfill
{\includegraphics[width=0.16\linewidth]{pics/mixed_2_2_3_3_2.png}}\hfill
{\includegraphics[width=0.16\linewidth]{pics/mixed_3_1_3_3_0.png}}\hfill
{\includegraphics[width=0.16\linewidth]{pics/mixed_3_1_3_3_1.png}}
\caption{The 12 pairwise-independent mixed invariants of the minimal flexible set up to order $l_m=3$: one between ${}{^1_1}H$ and ${}{_{1}^{3}}H$, two between ${}{^1_1}H$ and ${}{_{2}^{2}}H$, two between ${}{_{1}^{3}}H$ and ${}{_{2}^{2}}H$, two between ${}{_{1}^{1}}H$ and ${}{_{3}^{3}}H$, two between ${}{_{1}^{3}}H$ and ${}{_{3}^{3}}H$, and three between ${}{_{2}^{2}}H$ and ${}{_{3}^{3}}H$.\label{f:exMixed3}} 
\end{figure}
\end{ex}

\section{Cartesian Spherical Moments}\label{s:spherical}
In the previous sections we have seen that the irreducible tensors can form flexible bases of moment invariants that are invulnerable to vanishing tensors and functions that separate into a spherical and a radial part. 

Even though it is now possible to treat spherical functions, i.e., $f(\phi,\theta)=f(\hat r)$, in our flexible descriptor sets, the volumetric moments are not meant to be applied to spherical functions. The flexible descriptor sets are supposed to ensure that the full discriminative power is maintained if a volumetric function happens to degenerate to a spherical or a spherical-radial function. If it is clear from the outset that spherical functions are considered exclusively, such as in machine learning atomic potentials~\cite{allen2025optimal}, then spherical moments should be preferred to the volumetric moments because it saves computation during integration. Also, the complexity of the flexible descriptor sets can be reduced a-priori in this setting. 
% Even though that means that they can handle purely spherical functions $f(\phi,\theta)=f(\hat r)$, there is a more efficient way to work in cases where it is explicitly known that only spherical functions are treated, such as in machine learning atomic potentials~\cite{allen2025optimal}. In particular, (1) there is no need to integrate across the whole 3D volume and (2) there is no need to carry along the full irreducible decomposition. 
In this section, we derive this tailored theory of Cartesian spherical moment invariants.

\begin{defn}\label{d:sph_mom_tensor}
Let $f(\hat r):S_1(0)\to \mathbb R$ be a spherical scalar function over the unit sphere $S_1(0)\subset \mathbb R^3$. The Cartesian \textbf{spherical moment tensor} ${}{^\ell}{\hat M}$ of order $\ell\in\N$ is defined by
\begin{equation} \label{d:sphericalMoment}\begin{aligned}
{}{^\ell}{\hat M} &:= \int_{S_1(0)} \hat r^{\otimes \ell}f(\hat r) \d^2\hat r,\\
{}{^\ell}{\hat M}_{i_1\ldots i_\ell}&=\int_{S_1(0)}\hat r_{i_1}\ldots\hat r_{i_\ell}f(x)\d^2\hat r.
\end{aligned}\end{equation}
\end{defn}
Note that for spherically symmetric volumetric functions, spherical moments and the classical 3D moments are not identical, rather they differ by a factor of
\begin{equation}\label{volumetricSpherical} \begin{aligned}
{}{^\ell}M &\overset{\eqref{sphericalMoment}}=\int_0^{1}\int_{S_1(0)} r^{\ell+2}\hat r^{\otimes \ell}f(\hat r)\,dr\,d^2\hat r\\
 &=\int_0^{1}r^{\ell+2}H r\int_{S_1(0)} \hat r^{\otimes \ell}f(\hat r)\, d^2\hat r\\
 &\overset{\eqref{d:sphericalMoment}}= \frac1{\ell+3}{}{^\ell}{\hat M}.
\end{aligned}\end{equation}
This relation lets us translate all of our findings to the invariants of spherical moments. 

\subsection{Irreducible Tensor Moments Bases}
Eq.~\eqref{radialSphericalTrace} shows that the traces of moment tensors of spherical-radial functions are dependent on lower-rank tensors. For spherical functions, this simplifies to
\begin{equation}\label{sphericalTrace} \begin{aligned} 
{}{^\ell}M^{(1,2)}_{i_3\ldots i_\ell} &\overset{\eqref{radialSphericalTrace}}=
\frac{\ell+1}{\ell+3} {}{^{\ell-2}}M_{i_3\ldots i_\ell},
\end{aligned} \end{equation}
and by using the relation between volumetric and spherical moments, we even get the identity for the traces of spherical moment tensors,
\begin{equation} \label{traceIdentity}\begin{aligned} 
{}{^\ell}{\hat M}^{(1,2)}_{i_3\ldots i_\ell}&\overset{\eqref{volumetricSpherical}}
=(\ell+3){}{^\ell}M^{(1,2)}_{i_3\ldots i_\ell}\\
&\overset{\eqref{sphericalTrace}}
=\frac{\ell+1}{\ell+3} (\ell+3){}{^{\ell-2}}M_{i_3\ldots i_\ell}\\
&\overset{\eqref{volumetricSpherical}}
={}{^{\ell-2}}{\hat M}_{i_3\ldots i_\ell}.
\end{aligned} \end{equation}
Consequently, we can omit all traces from the descriptor sets from the outset.

\begin{ex}\label{ex:irreducibleTensors2}
Example~\ref{ex:irreducible} showed the six irreducible tensors in the decompositions for volumetric moments up to $\ell_m=3$: ${}{^0_0}H,\;{}{^1_1}H,\;{}{^2_2}H,\;{}{_{0}^{2}}H,\;{}{^3_3}H,\;{}{_{1}^{3}}H$.
% ,{}{^4_4}H,{}{_{2}^{4}}H,{}{_{0}^{4}}H. 
For spherical moments, we only need to consider four: ${}{^0_0}H,\;{}{^1_1}H,\;{}{^2_2}H,\;{}{^3_3}H$, one for each order because of the identity in Equation~\eqref{traceIdentity}:
\begin{equation} \begin{aligned} 
{}{_{0}^{2}}H^{(1,2)} &= {}{^0_0}H,&&
{}{_{1}^{3}}H^{(1,2)} &= {}{^1_1}H.
% {}{_{2}^{4}}H &= {}{^2_2}H,\\
% {}{_{0}^{4}}H &= {}{^4_2}H = {}{^4_0}H.
\end{aligned} \end{equation}
\end{ex}

Because the degrees of freedom of the volumetric and spherical moments for a spherical function are the same, the results from Section~\ref{s:theory} can be adapted to the case of spherical functions as summarized in Table~\ref{t:spherical}. The lower number of independent moments that comes from the lack of a radial dependence adds up exactly to the smaller number of invariants that come from omitting the irreducible tensors whose orders differ from the orders of their moment tensor. This means that the flexible sets can be generated in exactly the same way as for the volumetric functions described in Section~\ref{s:sets} after dropping the reduced-order irreducible tensors from the initial set.

\begingroup
\setlength{\tabcolsep}{5pt}
\begin{table}%[ht]
 %\begin{small}
  \begin{center}
\renewcommand{\arraystretch}{1.0} % single spacing
\begin{tabular}{|c|c|c|c|c|c|c|} % or whatever you want for the non-header row
\hline
\multicolumn{1}{|>{\RaggedRight\arraybackslash}p{1.20cm}|}{order $\ell$ or max.\ order $\ell_m$} &
\multicolumn{1}{>{\RaggedRight\arraybackslash}p{1.29cm}|}{moments of order $\ell$} &
\multicolumn{1}{>{\RaggedRight\arraybackslash}p{1.28cm}|}{moments up to $\ell_m$} &
\multicolumn{1}{>{\RaggedRight\arraybackslash}p{1.25cm}|}{hom.\ inv.\ of order $\ell$} &
\multicolumn{1}{>{\RaggedRight\arraybackslash}p{1.31cm}|}{hom.\ inv.\ up to $\ell_m$ } &
\multicolumn{1}{>{\RaggedRight\arraybackslash}p{1.42cm}|}{invariants up to order $\ell_m$} &
\multicolumn{1}{>{\RaggedRight\arraybackslash}p{1.55cm}|}{mixed inv.\ needed up to $\ell_m$} \\
\hline
\small{$\ell,\ell_m$} & 
\small{$2\ell+1$} & 
\small{$(\ell_m+1)^2$} & 
\small{$2\ell-2$} & 
\small{$\ell_m^2-\ell_m+2$} & 
\small{$\ell_m^2+2\ell_m-2$} & 
\small{$3\ell_m-4$} \\
\hline
0  & 1 & 1 &\cellcolor{black!20} 1 &\cellcolor{black!20} 1 & \cellcolor{black!20} 1 &\cellcolor{black!20} 0  \\
1  & 3 & 4 &\cellcolor{black!20} 1 &                     2 & \cellcolor{black!20} 3 &\cellcolor{black!20} 0  \\
2  & 5 & 9                     & 2                     & 4                      & 6                    & 2    \\
3  & 7 & 16                    & 4                     & 8                      & 13                   & 5    \\
4  & 9 & 25                    & 6                     & 14                     & 22                   & 8   \\
5  & 11& 36                    & 8                     & 22                     & 33                   & 11   \\
%6  & 13& 49                    & 10                    & 32                     & 46                   & 14   \\\hline
\hline
{Ref.}& 
{Lem.~\ref{l:dof}} & 
{Lem.~\ref{l:dof}} & 
{Lem.~\ref{l:pure}} & 
{Lem.~\ref{l:pure}} & 
{Lem.~\ref{l:dof}} & 
{Lem.~\ref{l:dof},~\ref{l:pure}}\\
\hline
\end{tabular}
%\end{small}
\caption{Summary of the degrees of freedom and numbers of invariants for a spherical function. The grey cells indicate the exceptions to the rule on the bottom.\label{t:spherical}}
\end{center}
\end{table}
\endgroup

\subsection{Descriptor Sets}
A specific flexible basis for spherical functions is generated analogously to Section~\ref{s:specific} by taking all pure invariants and adding three mixed invariants paired with the chosen robust irreducible tensor.

\begin{ex}\label{ex:spherical22}
We construct an example of a specific flexible basis analogously to Example~\ref{ex:22}. Let ${}{^\ell}{\hat M}$ be the spherical moment tensors up to order $\ell\leq \ell_m=3$.  Example~\ref{ex:irreducibleTensors2} shows the four full-order irreducible tensors in the decompositions: ${}{^0_0}H$, ${}{^1_1}H$, ${}{^2_2}H$, ${}{^3_3}H$. Therefore, there are always eight pure invariants as depicted in Fig.~\ref{f:exSphericalPure}, and they do not depend on the choice of a robust irreducible tensor. 
Assume a scenario in which the irreducible tensor ${}{^2_2}H$ is non-zero. Then we add the five mixed invariants that contain ${}{^2_2}H$, depicted in Fig.~\ref{f:exSphericalMixed1}. This totals 13 invariants, which correctly matches the predicted number from Table~\ref{t:spherical}. The complete list of invariants with their contraction information as formulas is given in \ref{appendix:example-spherical}.

%\begin{figure}%[ht]
%\centering
%{\includegraphics[width=0.12\linewidth]{pics/pure_0_0_0.png}}\hfill
%{\includegraphics[width=0.12\linewidth]{pics/pure_1_1_0.png}}\hfill
%{\includegraphics[width=0.12\linewidth]{pics/pure_2_2_0.png}}\hfill
%{\includegraphics[width=0.12\linewidth]{pics/pure_2_2_1.png}}\hfill
%{\includegraphics[width=0.12\linewidth]{pics/pure_3_3_0.png}}\hfill
%{\includegraphics[width=0.12\linewidth]{pics/pure_3_3_1.png}}\hfill
%{\includegraphics[width=0.12\linewidth]{pics/pure_3_3_2.png}}\hfill
%{\includegraphics[width=0.12\linewidth]{pics/pure_3_3_3.png}}
%\caption{The eight pure invariants of a spherical function up to $\ell_m=3$: 1 of ${}{^0_0}H$, 1 of ${}{^1_1}H$,  2 of ${}{^2_2}H$, and 4 of ${}{_{3}^{3}}H$. \label{f:exSphericalPure}} 
%\end{figure}

\begin{figure}%[ht]
\centering
{\includegraphics[width=0.12\linewidth]{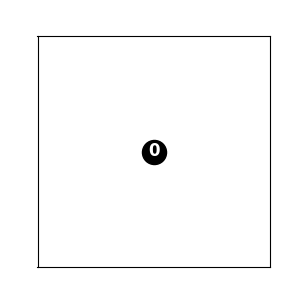}}\hfill
{\includegraphics[width=0.12\linewidth]{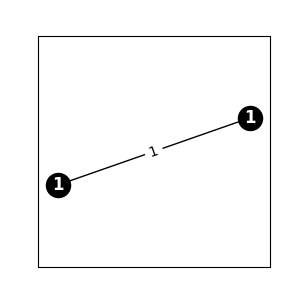}}\hfill
{\includegraphics[width=0.12\linewidth]{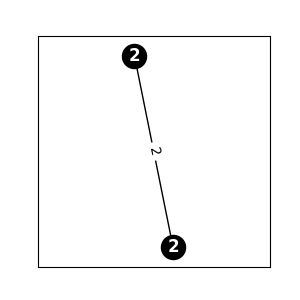}}\hfill
{\includegraphics[width=0.12\linewidth]{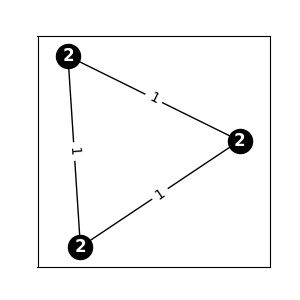}}\hfill
{\includegraphics[width=0.12\linewidth]{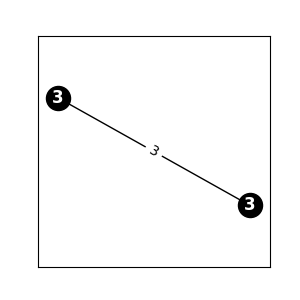}}\hfill
{\includegraphics[width=0.12\linewidth]{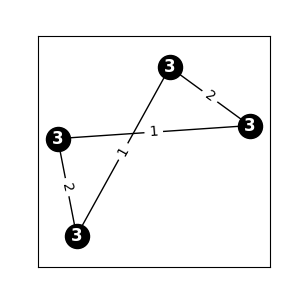}}\hfill
{\includegraphics[width=0.12\linewidth]{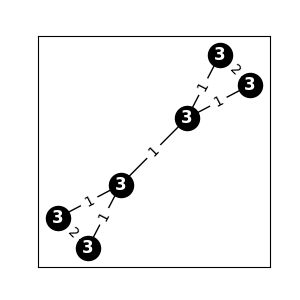}}\hfill
{\includegraphics[width=0.12\linewidth]{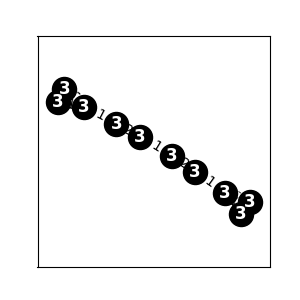}}
\caption{The eight pure invariants of a spherical function up to $\ell_m=3$: two of ${}{^0_0}H$, one of ${}{^1_1}H$, two of ${}{^2_2}H$, and four of ${}{_{3}^{3}}H$. \label{f:exSphericalPure}} 
\end{figure}

\begin{figure}%[ht]
\centering
{\includegraphics[width=0.19\linewidth]{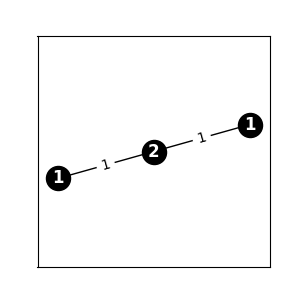}}\hfill
{\includegraphics[width=0.19\linewidth]{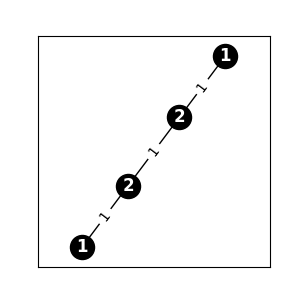}}\hfill
{\includegraphics[width=0.19\linewidth]{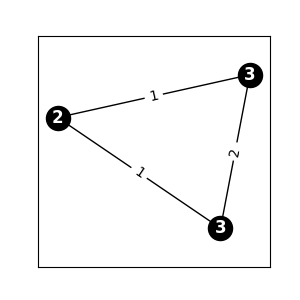}}\hfill
{\includegraphics[width=0.19\linewidth]{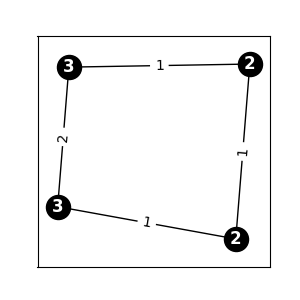}}\hfill
{\includegraphics[width=0.19\linewidth]{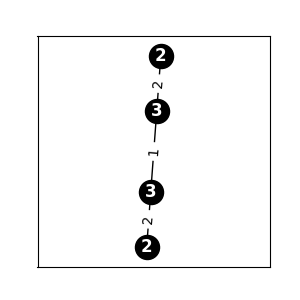}}
\caption{The five mixed invariants of a spherical function up to $\ell_m=3$ centered around ${}{^2_2}H$: two with ${}{^1_1}H$ and three with ${}{_{3}^{3}}H$. \label{f:exSphericalMixed1}} 
\end{figure}
\end{ex}

A minimal flexible set for spherical functions is generated analogously to Section~\ref{s:minimal} by taking all pure invariants and adding three mixed invariants for all pairs of irreducible tensors. 

For spherical functions, because the irreducible decompositions have only one component that needs to be considered, for any $\ell_m>0$ we have $n_1=1$ irreducible tensor of first rank and $n_2=\ell_m-1$ irreducible tensors of rank greater than one, giving a total of 
\begin{equation}
\ell_m^2 + 2\ell_m - 2 + 2n_1n_2 + 3\frac{n_2(n_2 - 1)}2 = \frac{5}{2}\ell_m^2-\frac{1}{2}\ell_m-1.
\end{equation}
The first values for $\ell_m=0,\ldots,5$ for the specific flexible basis are $1$, $3$, $6$, $13$, $22$, $33$, and the numbers for the minimal flexible set are $1$, $3$, $8$, $20$, $37$, $59$.
The computational overhead of using the minimal flexible set over the specific flexible basis for spherical functions is significantly less than for volumetric functions. With $\frac{5}{2}\ell_m^2-\frac{1}{2}\ell_m-1$ descriptors for the minimal flexible set and $\ell_m^2+2\ell_m-2$ (from column 6 of Table~\ref{t:spherical}) for the specific flexible basis; both fall into $\Theta(\ell_m^2)$.

\begin{ex}\label{ex:minimalSpherical}
We construct an example of a minimal flexible set analogously to Example~\ref{ex:minimal} for a
scenario in which no irreducible tensor can be guaranteed to never vanish. We add all individually independent sets of mixed invariants between two irreducible tensors to the eight pure invariants from Fig.~\ref{f:exSphericalPure}. The seven mixed invariants of the minimal flexible set are depicted in Fig.~\ref{f:exSphericalMixed3}. This leads to a total of 15 invariants that allow minimal flexibility, i.e., two additional invariants compared to a basis.
The complete list of invariants with the contraction information as formulas is given in~\ref{appendix:example-minimalSpherical}.

\begin{figure}%[ht]
\centering
{\includegraphics[width=0.14\linewidth]{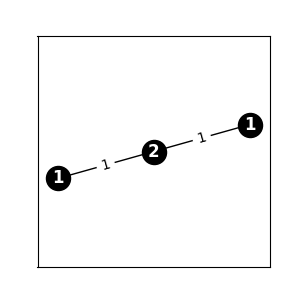}}\hfill
{\includegraphics[width=0.14\linewidth]{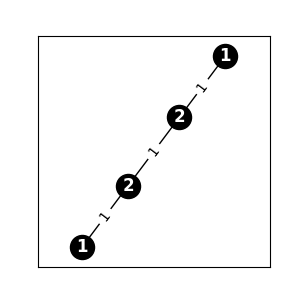}}\hfill
{\includegraphics[width=0.14\linewidth]{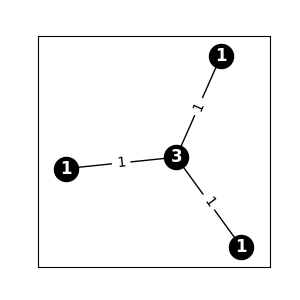}}\hfill
{\includegraphics[width=0.14\linewidth]{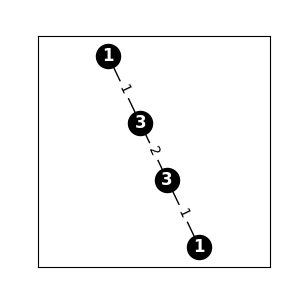}}\hfill
{\includegraphics[width=0.14\linewidth]{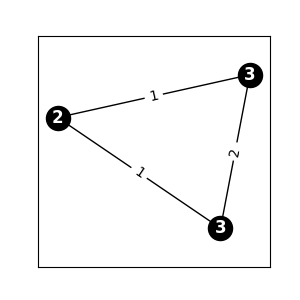}}\hfill
{\includegraphics[width=0.14\linewidth]{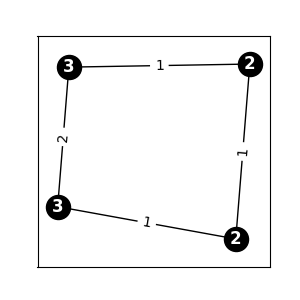}}\hfill
{\includegraphics[width=0.14\linewidth]{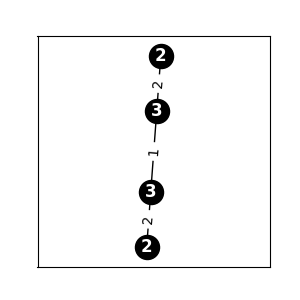}} \hfill\hfill
\caption{The seven pairwise independent mixed invariants of the minimal flexible set of a spherical function up to order $\ell_m=3$: two between ${}{^1_1}H$ and ${}{_{2}^{2}}H$, two between ${}{_{1}^{1}}H$ and ${}{_{3}^{3}}H$, and three between ${}{_{2}^{2}}H$ and ${}{_{3}^{3}}H$. \label{f:exSphericalMixed3}} 
\end{figure}
\end{ex}

\section{Application}
We showed that prior moment-invariant bases are vulnerable to functions that separate into radial and spherical components. As proven in Sec.~\ref{s:analysis}, the invariants for such functions depend primarily on low-order moments, making these bases non-flexible and potentially causing them to vanish entirely. We complement this theory with two qualitative, real-data experiments. First we demonstrate that moment vanishing can occur in real data, even when the underlying functions are not explicitly spherical. In these cases, our proposed irreducible tensor moment bases (Sec.~\ref{s:sets}) successfully distinguish the patterns, whereas previous bases do not. Second, we show that for spherical data, our irreducible tensor moment bases (Sec.~\ref{s:spherical}) provide significant advantages, while the previously proposed bases cannot distinguish between, e.g., different configurations of a methane molecule in a hydrogen environment. These qualitative experiments illustrate how our theoretical findings manifest in practice. 

\subsection{Viscous Fingers}
The Viscous Fingers dataset from the standard ParaView sample data library models~\cite{ahrens2005paraview} captures the viscous fingering instability that occurs when a less viscous fluid injected into a more viscous one in a porous medium or narrow gap. Instead of smooth displacement, the interface destabilizes into branched, finger-like structures. These patterns are central to petroleum engineering, hydrology, and materials science. We use this dataset as a controlled source to evaluate our invariants under realistic conditions.

We demonstrate the vanishing problem using the real pattern shown in Figure~\ref{fig:vanishing}(a). Moment tensors are computed up to third order. The trace of ${^3}M$ for this pattern is nearly zero. For reference, Figure~\ref{fig:vanishing}(b) shows a pattern with an exactly zero trace, which is visually indistinguishable from the original. For completeness, we also report the $L_2$ distances between all patterns. 

To illustrate how the dependence on traces affects prior bases, we construct an adversarial pattern by rotating only the linear component ${^1}M$. The rotated versions of the original and zero-trace patterns are shown in Figures~\ref{fig:vanishing}(c) and~\ref{fig:vanishing}(d), respectively. Prior bases incorrectly consider the adversarial pattern (c) to be closer to the original than to the zero-trace pattern (b), whereas our proposed invariants place (c) an order of magnitude farther from the original than (b). We report these relative distances for comparison, as the absolute distances between invariants are not directly meaningful.

\begin{figure}%[!t]
    \centering
    \includegraphics[width=0.5\linewidth]{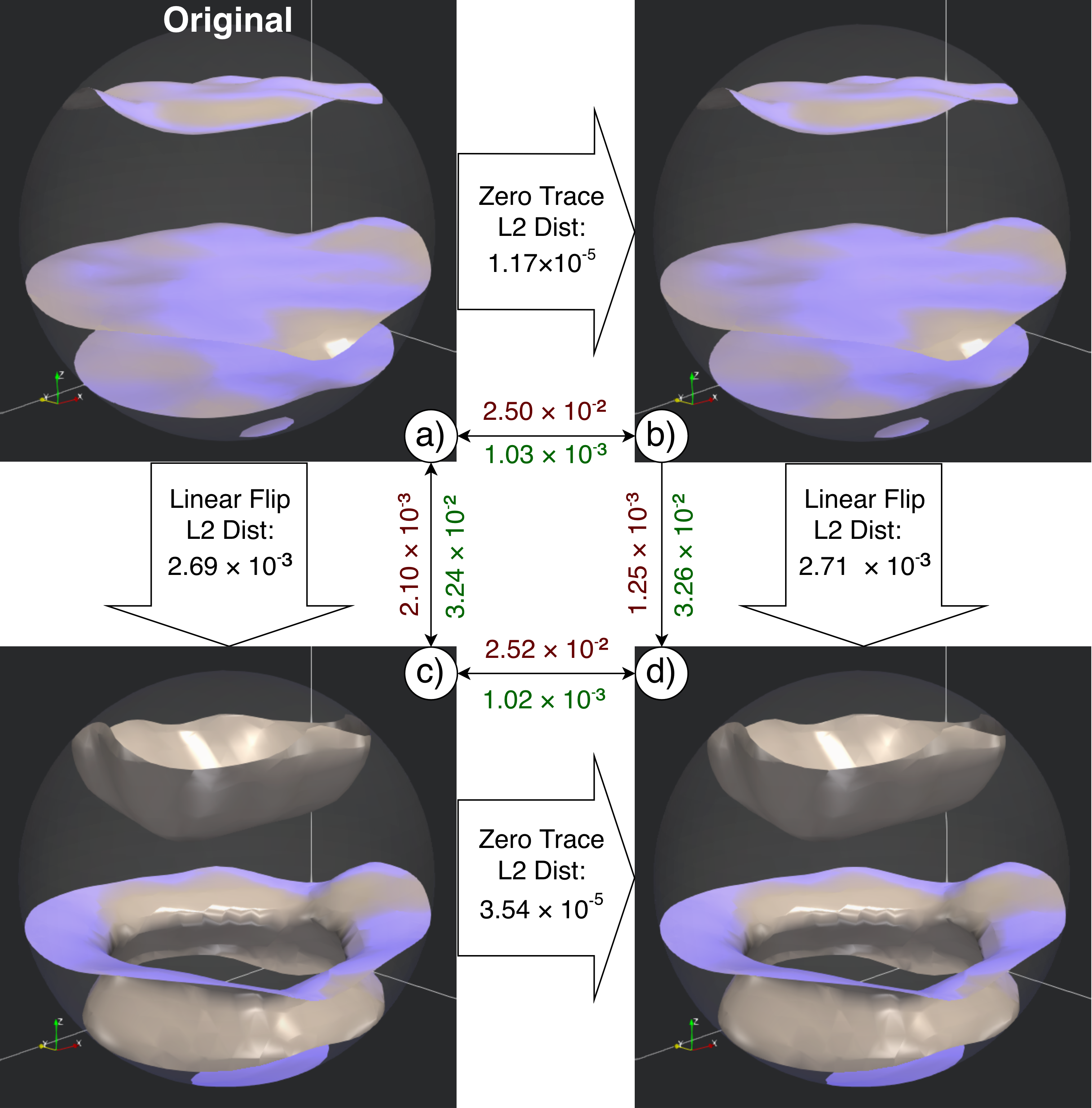}
    \caption{Isocontour visualization of the vanishing problem on a pattern in the Viscous Fingers dataset (run: 2, timestep: 64, center location: (22,31,23) radius: 13, array: density, iso-value: 0.085). 
\textbf{(a)} Original pattern with a nearly zero trace of the third-order moment ${}^3M$.
\textbf{(b)} Pattern with an exactly zero trace, visually indistinguishable from (a). 
\textbf{(c)} Adversarial pattern obtained by rotating only the linear component ${}^1M$ of the original pattern. 
\textbf{(d)} The corresponding rotated version of the zero-trace pattern. }
    \label{fig:vanishing}
\end{figure}

This experiment confirms that the vanishing problem is not merely theoretical but can occur in experimental data. The nearly-vanishing third-order moments observed in the Viscous Fingers patterns cause prior bases to misinterpret structurally distinct configurations as similar. In contrast, the proposed basis remains stable and discriminative, correctly capturing these differences. These results demonstrate its practical advantage even when the data are not explicitly spherical.

\begin{figure}%[h]
    \centering
    \includegraphics[width=0.75\linewidth]{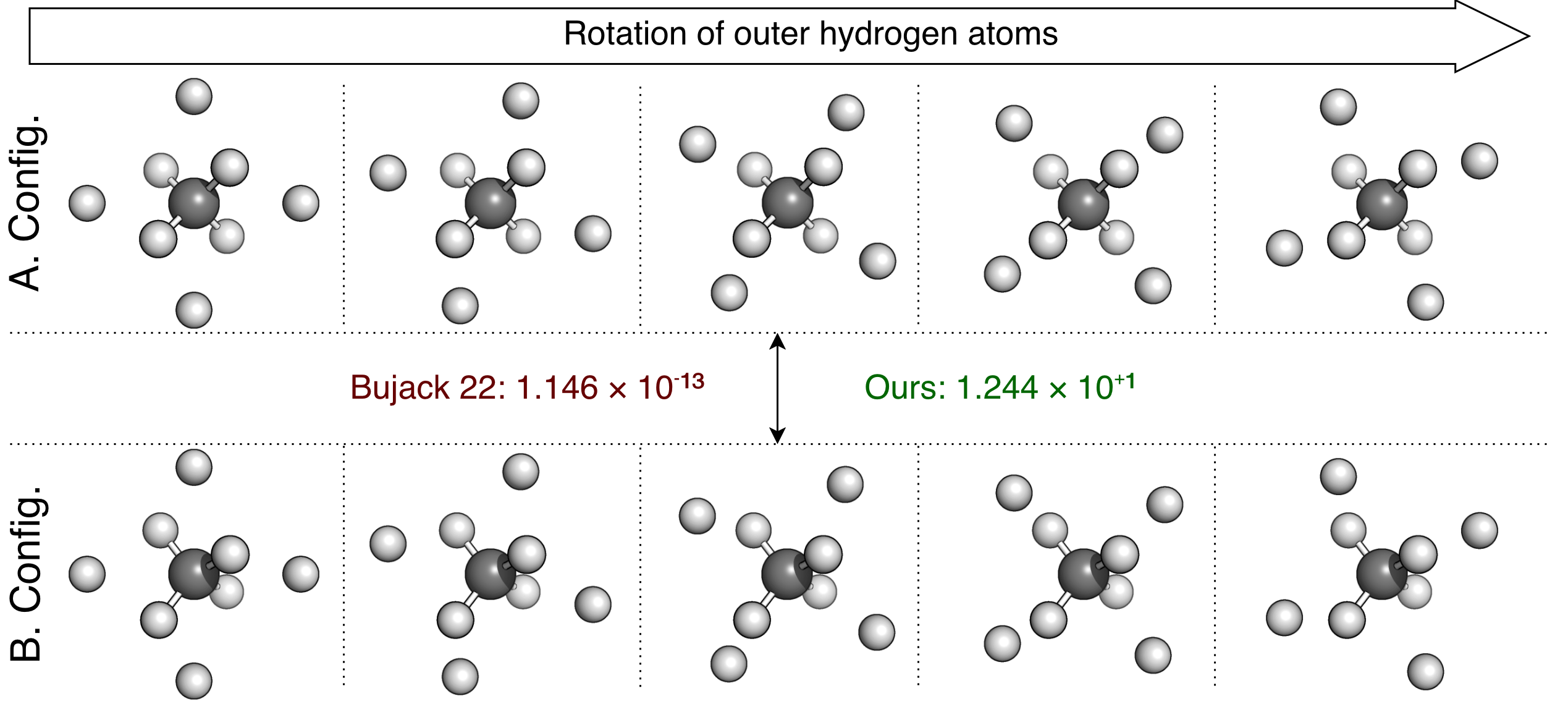}
    \caption{Manifold of two distinct configurations (A and B) of a methane molecule in a hydrogen environment. Each row shows several representatives of each manifold. Each column shows 2 configurations that are indiscriminable by the prior basis. The final column gives the maximum atomic force difference (2.65 eV/\AA\ for A and 3.00 eV/\AA\  for B). Despite these physical differences, prior invariant bases map each configuration pair A and B to the same representation.}
    \label{fig:methan}
\end{figure}

\subsection{Methane Molecule Configurations}
Our practical example of spherical-radial functions models molecules, where the lack of distinguishability has practical consequences. The functions could represent atomic configurations such as those in Fig.~\ref{fig:methan}, where a methane molecule, with a tetrahedral arrangement of hydrogen atoms, 
is surrounded by other hydrogen atoms. The two configurations shown are physically distinct, and if we compute the forces acting on them we find that the largest force on configuration A is 2.65 eV/\AA, whereas for configuration B it is 3.00 eV/\AA. However, if we were to search large molecular or materials datasets using the prior basis, these configurations would be incorrectly identified as identical.

If such a basis were used to remove rotationally identical structures from a dataset, this erroneous conflation could result in the removal of genuinely distinct configurations. Even more seriously, in the context of interatomic potential models, treating the two configurations as the same would impose the constraint that they must have identical energies and forces, which could significantly reduce the accuracy of the resulting model.

\subsection{Interatomic Potentials}
Understanding how material and chemical systems behave under different conditions (e.g., temperatures and pressures) is of interest in numerous fields in the natural sciences~\cite{Ponder2003}. One way to do this is to perform computational simulations that explicitly describe the atoms present in a material and model how the energy of a system changes with the atomic positions~\cite{Frenkel2002}. The energy of a system can then be used to find the structures that are preferred under different conditions. For example, under a range of high pressures, a tetrahedral arrangement (as diamond) is the lowest-energy configuration that can be assumed by carbon atoms.

For simulation, a model is needed to describe the interactions between different atoms and molecules. Atoms interact with one another according to various physical phenomena, such as bonding between atoms or electrical charges attracting/repelling one another~\cite{Frenkel2002}. Machine learning has been used to model these interactions because it provides a faster solution than traditional quantum mechanical approaches; these models are known as machine-learning interatomic potentials (MLIPs)~\cite{Behler2016,kulichenko2021rise}.

A fundamental requirement when building MLIPs is ensuring that physical invariants are satisfied~\cite{Behler2016}. Translations, rotations, and permutational symmetries must be respected by the underlying model. Building physical invariants into a model is not only important to ensure that non-physical behavior does not occur, but also to increase the performance of the model. The required invariants can be built into a model in various ways, and in recent years some of the MLIP community has moved towards using tensor products~\cite{Shapeev2016,Musaelian2023}. However, how to build the most effective descriptor set from tensor products remains an open question. This is an important consideration as smaller sets result in faster models, and removing redundancy can improve extrapolation performance.

In a companion paper~\cite{allen2025optimal}, the theory from this paper is used to generate a flexible set of invariants as the descriptors to a state-of-the-art MLIP neural network~\cite{Chigaev2023}, specifically into the Hierarchically Interacting Particle Neural Network (HIP-NN), yielding a variant called the Hierarchically Interacting Particle Higher-Order Polynomial Neural Network (HIP-HOP-NN) that can include arbitrarily-high body-order information in a single interaction layer~\cite{lubbers2018hierarchical}. HIP-HOP-NN outperforms state-of-the-art MLIPs neural networks, such as HIP-NN and HIP-NN-TS~\cite{Chigaev2023}, which use incomplete feature sets. In the highlighted case of a large dataset of methane configurations, the HIP-HOP-NN model was found to be roughly twice as accurate as HIP-NN-TS and roughly twenty times as accurate as HIP-NN when using one layer of message passing. 

\section{Discussion and Conclusion}
We have presented a systematic approach to finding flexible, complete, and independent sets of moment invariants with respect to orthogonal transformations by using the irreducible tensor decomposition in Cartesian coordinates.

The set of all tensor contractions~\cite{bujack2026} is complete, we have shown that it is always possible to construct a basis using (1) all pure invariants, i.e., invariants that are built from a single irreducible tensor, and (2) mixed invariants with no more than two different irreducible tensors and with one of the factors fixed per pattern. We have also introduced the \emph{general flexible set}, the smallest set that remains complete for any input function, which can be applied if no pattern is known a priori.

We have shown that our approach is superior to previous methods based on spherical harmonics because they do not provide a means to eliminate higher-order dependencies~\cite{suk20153d}, and to previous methods based on Cartesian tensors because they are vulnerable to functions that can be separated into its spherical and radial components~\cite{bujack2022systematic}.
Based on empirical evidence, we hypothesize that it is fully flexible to any form of degeneracy in the input function, such as rotational symmetries, spherical functions, functions with vanishing moments, but we have yet to prove this.

We have shown that this improvement is relevant in practical applications. In particular in machine learning interatomic potentials, the functions are purely spherical, for which the Cartesian method loses its completeness. The integration of flexible, complete, and independent descriptor sets outperforms state-of-the-art models~\cite{allen2025optimal}.

An open-source implementation for the generation of the invariant feature sets is available at \url{github.com/lanl/rotation-invariant-neural-networks}.

\section*{Acknowledgments}
We thank Kei Davis for feedback and Pieter Swart for helpful discussions. This work was supported in part by the U.S. Department of Energy through the LANL Laboratory Directed Research Development Program project number 20250145ER, and the Czech Science Foundation grant No.\ GA24-10069S. The authors used ChatGPT and Cursor for reformulation and error detection and fully corrected the final content. This document is published as LANL LA-UR-25-22957. 

%% Loading bibliography style file
%\bibliographystyle{model1-num-names}
\bibliographystyle{cas-model2-names}

% Loading bibliography database
\bibliography{cas-refs}

@article{ehrenborg1993apolarity,
  title={Apolarity and canonical forms for homogeneous polynomials},
  author={Ehrenborg, Richard and Rota, Gian-Carlo},
  journal={European Journal of Combinatorics},
  volume={14},
  number={3},
  pages={157--181},
  year={1993},
  publisher={Elsevier}
}

@article{beecken2013algebraic,
  title={Algebraic independence and blackbox identity testing},
  author={Beecken, Malte and Mittmann, Johannes and Saxena, Nitin},
  journal={Information and Computation},
  volume={222},
  pages={2--19},
  year={2013},
  publisher={Elsevier}
}

@article{coope1965irreducible,
  title={Irreducible cartesian tensors},
  author={Coope, JAR and Snider, RF and McCourt, FR},
  journal={The Journal of Chemical Physics},
  volume={43},
  number={7},
  pages={2269--2275},
  year={1965},
  publisher={American Institute of Physics}
}

@article{jacobi1841determinantibus,
  title={De Determinantibus functionalibus},
  author={Jacobi, Carl Gustav Jacob},
  journal={Journal f{\"u}r die reine und angewandte Mathematik (Crelles Journal)},
  volume={1841},
  number={22},
  pages={319--359},
  year={1841},
  publisher={De Gruyter Berlin, New York}
}

@inproceedings{tsai2020approaches,
  title={Approaches for In Situ Computation of Moments in a Data-Parallel Environment.},
  author={Tsai, Karen C and Bujack, Roxana and Geveci, Berk and Ayachit, Utkarsh and Ahrens, James P},
  booktitle={EGPGV@ Eurographics/EuroVis},
  pages={57--68},
  year={2020}
}

@article{hickman2012geometric,
  title={Geometric moments and their invariants},
  author={Hickman, Mark S},
  journal={Journal of Mathematical Imaging and Vision},
  volume={44},
  number={3},
  pages={223--235},
  year={2012},
  publisher={Springer}
}

@inproceedings{LH09,
  title={A generalization of moment invariants on {2D} vector fields to tensor fields of arbitrary order and dimension},
  author={Langbein, Max and Hagen, Hans},
  booktitle={International Symposium on Visual Computing},
  pages={1151--1160},
  year={2009},
  organization={Springer}
}

@inproceedings{suk2004graph,
  title={Graph method for generating affine moment invariants},
  author={Suk, Tomas and Flusser, Jan},
  booktitle={Pattern Recognition, 2004. ICPR 2004. Proceedings of the 17th International Conference on},
  volume={2},
  pages={192--195},
  year={2004},
  organization={IEEE}
}

@article{lubbers2018hierarchical,
  title={Hierarchical modeling of molecular energies using a deep neural network},
  author={Lubbers, Nicholas and Smith, Justin S and Barros, Kipton},
  journal={The Journal of chemical physics},
  volume={148},
  number={24},
  year={2018},
  publisher={AIP Publishing}
}

@article{ahrens2025ecp,
  title={The ECP ALPINE project: In situ and post hoc visualization infrastructure and analysis capabilities for exascale},
  author={Ahrens, James and Arienti, Marco and Ayachit, Utkarsh and Bennett, Janine and Binyahib, Roba and Biswas, Ayan and Bremer, Peer-Timo and Brugger, Eric and Bujack, Roxana and others},
  journal={The International Journal of High Performance Computing Applications},
  volume={39},
  number={1},
  pages={32--51},
  year={2025},
  publisher={SAGE Publications Sage UK: London, England}
}

@ARTICLE{SukFlu:AMIgraph,
  AUTHOR =       "Tom\'{a}\v{s} Suk and Jan Flusser",
  TITLE =        "Affine Moment Invariants Generated by Graph Method",
  JOURNAL =      "Pattern Recognition",
  YEAR =         "2011",
  VOLUME =       "44",
  NUMBER =       "9",
  PAGES =        "2047-2056"
}

@article{suk20153d,
  title={{3D} rotation invariants by complex moments},
  author={Suk, Tom{\'a}{\v{s}} and Flusser, Jan and Boldy{\v{s}}, Ji{\v{r}}{\'\i}},
  journal={Pattern Recognition},
  volume={48},
  number={11},
  pages={3516--3526},
  year={2015},
  publisher={Elsevier}
}

@book{grinfeld2013introduction,
  title={Introduction to Tensor Analysis and the Calculus of Moving Surfaces},
  author={Grinfeld, P.},
  isbn={9781461478676},

  year={2013},
  publisher={Springer New York}
}

@book{bowen2008introduction,
  title={Introduction to Vectors and Tensors},
  author={Bowen, R.M. and Wang, C.},
  isbn={9780486469140},
  lccn={2008024877},
  series={Dover books on mathematics},

  year={2008},
  publisher={Dover Publications}
}

@ARTICLE{BH95,
  author = {Gilles Burel and Hugues Henocq},
  title = {{3D Invariants and their Application to Object Recognition}},
  journal = {Signal procesing},
  year = {1995},
  volume = {45},
  pages = {1-22},
  number = {1},
  publisher = {Elsevier}
}

@INCOLLECTION{Cant96,
  author = {Nikolaos Canterakis},
  title = {{Complete moment invariants and pose determination for orthogonal
	transformations of 3D objects}},
  booktitle = {Mustererkennung 1996, 18. DAGM Symposium, Informatik aktuell},
  publisher = {Springer},
  year = {1996},
  pages = {339-350},

}

@article{DN77,
  title={Pattern matching under affine transformations},
  author={Dirilten, Hudai and Newman, Thomas G},
  journal={Computers, IEEE Transactions on},
  volume={100},
  number={3},
  pages={314--317},
  year={1977},
  publisher={IEEE}
}

@ARTICLE{Flu02,
  author = {Jan Flusser},
  title = {{On the inverse problem of rotation moment invariants}},
  journal = {Pattern Recognition},
  year = {2002},
  volume = {35},
  pages = {3015-3017}
}

@ARTICLE{Flu00,
  author = {Jan Flusser},
  title = {{On the independence of rotation moment invariants}},
  journal = {Pattern Recognition},
  year = {2000},
  volume = {33},
  pages = {1405-1410},
  number = {9}
}

@BOOK{FSZ16,
  title = {{2D and 3D Image Analysis by Moments}},
  publisher = {John Wiley \& Sons},
  year = {2016},
  author = {Flusser, J. and Zitova, B. and Suk, T.},
  isbn = {978-1-119-03935-8}
}

@ARTICLE{Hu62,
  author = {Ming-Kuei Hu},
  title = {{Visual pattern recognition by moment invariants}},
  journal = {IRE Transactions on Information Theory},
  year = {1962},
  volume = {8},
  pages = {179-187},
  number = {2},
  issn = {0096-1000}
}

@inproceedings{xu20063,
  title={3-D surface moment invariants},
  author={Xu, Dong and Li, Hua},
  booktitle={18th International Conference on Pattern Recognition (ICPR'06)},
  volume={4},
  pages={173--176},
  year={2006},
  organization={IEEE}
}

@article{backus1970geometrical,
  title={A geometrical picture of anisotropic elastic tensors},
  author={Backus, George},
  journal={Reviews of geophysics},
  volume={8},
  number={3},
  pages={633--671},
  year={1970},
  publisher={Wiley Online Library}
}

@ARTICLE{LD89,
  author = {C.H. Lo and H.S. Don},
  title = {{3-D Moment Forms: Their Construction and Application to Object Identification
	and Positioning}},
  journal = {IEEE Trans. Pattern Anal. Mach. Intell.},
  year = {1989},
  volume = {11},
  pages = {1053-1064},
  number = {10},
  address = {Los Alamitos, CA, USA},

  issn = {0162-8828},
  publisher = {IEEE Computer Society}
}

@ARTICLE{PCO85,
  author = {Ziha Pinjo and David Cyganski and John A Orr},
  title = {{Determination of 3-D object orientation from projections}},
  journal = {Pattern Recognition Letters},
  year = {1985},
  volume = {3},
  pages = {351-356},
  number = {5},

  issn = {0167-8655},

}

@INCOLLECTION{Suk2011tensor,
  author = {Suk, Tomas and Flusser, Jan},
  title = {{Tensor Method for Constructing 3D Moment Invariants}},
  booktitle = {Computer Analysis of Images and Patterns},
  publisher = {Springer Berlin, Heidelberg},
  year = {2011},
  volume = {6855},
  series = {Lecture Notes in Computer Science},
  pages = {212-219}
}

@inproceedings{KFR03,
  author = {Michael Kazhdan and Thomas Funkhouser and Szymon Rusinkiewicz},
  title = {{Rotation Invariant Spherical Harmonic Representation of 3D Shape Descriptors}},
  booktitle = {Symposium on Geometry Processing},
  volume={6},
  pages={156--164},
  year = {2003}
}

@article{hergl2020introduction,
  title={An introduction to the deviatoric tensor decomposition in three dimensions and its multipole representation},
  author={Hergl, Chiara and Nagel, Thomas and Scheuermann, Gerik},
  journal={arXiv preprint arXiv:2009.11723},
  year={2020}
}

@inproceedings{langbein2009generalization,
  title={A generalization of moment invariants on 2d vector fields to tensor fields of arbitrary order and dimension},
  author={Langbein, Max and Hagen, Hans},
  booktitle={International Symposium on Visual Computing},
  pages={1151--1160},
  year={2009},
  organization={Springer}
}

@article{kostkova2019affine,
  title={Affine invariants of vector fields},
  author={Kostkov{\'a}, Jitka and Suk, Tom{\'a}{\v{s}} and Flusser, Jan},
  journal={IEEE Transactions on Pattern Analysis and Machine Intelligence},
  volume={43},
  number={4},
  pages={1140--1155},
  year={2019},
  publisher={IEEE}
}

@inproceedings {bujack2015clustering,
booktitle = {Eurographics Conference on Visualization (EuroVis) - Short Papers},
editor = {E. Bertini and J. Kennedy and E. Puppo},
title = {{Clustering Moment Invariants to Identify Similarity within 2D Flow Fields}},
author = {Bujack, Roxana and Kasten, Jens and Natarajan, Vijay and Scheuermann, Gerik and Joy, Kenneth I.},
year = {2015},
publisher = {The Eurographics Association},
pages = {31--35}
}

@InProceedings{bujack2017tensor,
author="Bujack, Roxana
and Hagen, Hans",
editor="Schultz, Thomas
and {\"O}zarslan, Evren
and Hotz, Ingrid",
title="Moment Invariants for Multi-Dimensional Data",
booktitle="Modeling, Analysis, and Visualization of Anisotropy",
year="2017",
publisher="Springer International Publishing",
address="Cham",
pages="43--64",
isbn="978-3-319-61358-1"
}

@inproceedings{bujack2017flexible,
  title={Flexible Moment Invariant Bases for {2D} Scalar and Vector Fields},
  author={Roxana Bujack and Jan Flusser},
  booktitle={Proceedings of International Conference in Central Europe on Computer Graphics, Visualization and Computer Vision (WSCG)},
    pages = {11--20},
  year={2017}
}

@article{bujack2022systematic,
  title={Systematic generation of moment invariant bases for 2D and 3D tensor fields},
  author={Bujack, Roxana and Zhang, Xinhua and Suk, Tom{\'a}{\v{s}} and Rogers, David},
  journal={Pattern Recognition},
  volume={123},
  pages={108313},
  year={2022},
  publisher={Elsevier}
}

@inproceedings{flusser2023affine,
  title={Affine Moment Invariants of Tensor Fields},
  author={Flusser, Jan and Suk, Tom{\'a}{\v{s}} and L{\'e}bl, Mat{\v{e}}j and Bujack, Roxana and Ibrahim, Ibrahim},
  booktitle={Scandinavian Conference on Image Analysis},
  pages={299--313},
  year={2023},
  organization={Springer}
}

@incollection{Ponder2003,
author = "Ponder, Jay W. and Case, David A.",
booktitle = "{Protein Simulations}",
pages = "27--85",
publisher = "Academic Press",
series = "{Advances in Protein Chemistry}",
title = "{Force Fields for Protein Simulations}",
volume = "66",
year = "2003"
}

@book{Frenkel2002,
title = {Understanding Molecular Simulation (Second Edition)},
publisher = {Academic Press},
year = {2002},
author = {Daan Frenkel and Berend Smit}
}

@article{Behler2016,
author = {Behler,J\"org },
title = {Perspective: Machine learning potentials for atomistic simulations},
journal = {J. Chem. Phys},
volume = {145},
number = {17},
pages = {170901},
year = {2016},
}

@article{kulichenko2021rise,
	author = {Kulichenko, Maksim and Smith, Justin S. and Nebgen, Benjamin and Li, Ying Wai and Fedik, Nikita and Boldyrev, Alexander I. and Lubbers, Nicholas and Barros, Kipton and Tretiak, Sergei},
	journal = {The Journal of Physical Chemistry Letters},
	month = {07},
	number = {26},
	OPTpages = {6227--6243},
	title = {The Rise of Neural Networks for Materials and Chemical Dynamics},
	volume = {12},
	year = {2021}}

@article{Musaelian2023,
  title    = "Learning local equivariant representations for large-scale
              atomistic dynamics",
  author   = "Musaelian, Albert and Batzner, Simon and Johansson, Anders and
              Sun, Lixin and Owen, Cameron J and Kornbluth, Mordechai and
              Kozinsky, Boris",
  journal  = "Nat Commun",
  volume   =  14,
  number   =  1,
  year     =  2023,
}

@article{Shapeev2016,
author = {Shapeev, Alexander V.},
title = {Moment Tensor Potentials: A Class of Systematically Improvable Interatomic Potentials},
journal = {Multiscale Model. Simul.},
volume = {14},
number = {3},
pages = {1153-1173},
year = {2016}
}

@article{allen2025optimal,
author={Allen, Alice E. A.
and Shinkle, Emily
and Bujack, Roxana
and Lubbers, Nicholas},
title={Optimal invariant sets for atomistic machine learning},
journal={npj Computational Materials},
year={2026},
month={Jan},
day={17},
issn={2057-3960},
doi={10.1038/s41524-025-01948-0},
url={https://doi.org/10.1038/s41524-025-01948-0}
}

@article{Chigaev2023,
    author = {Chigaev, Michael and Smith, Justin S. and Anaya, Steven and Nebgen, Benjamin and Bettencourt, Matthew and Barros, Kipton and Lubbers, Nicholas},
    title = "{Lightweight and effective tensor sensitivity for atomistic neural networks}",
    journal = {J. Chem. Phys.},
    volume = {158},
    number = {18},
    pages = {184108},
    year = {2023},
    month = {05},
    issn = {0021-9606},
}

@article{AtomicClusterExpansion,
  title = {Atomic cluster expansion for accurate and transferable interatomic potentials},
  author = {Drautz, Ralf},
  journal = {Physical Review B},
  volume = {99},
  issue = {1},
  pages = {014104},
  numpages = {15},
  year = {2019},
  month = {Jan},
  publisher = {American Physical Society},


}

@article{ahrens2005paraview,
title={Paraview: An end-user tool for large data visualization},
author={Ahrens, James and Geveci, Berk and Law, Charles},
journal={The visualization handbook},
volume={717},
number={8},
year={2005}
}

@article{
brehmer2025,
title={Does equivariance matter at scale?},
author={Johann Brehmer and S{\"o}nke Behrends and Pim De Haan and Taco Cohen},
journal={Transactions on Machine Learning Research},
issn={2835-8856},
year={2025},
url={https://openreview.net/forum?id=wilNute8Tn},
note={}
}

@inproceedings{
karella2026,
title={On Making \${SE}(2)\$-Invariant Networks Optimal},
author={Tomas Karella and Emily Shinkle and Alice E. A. Allen and Pieter Swart and Nicholas Lubbers and Roxana Bujack},
booktitle={The Fortieth Annual Conference on Neural Information Processing Systems},
year={2026},
url={https://openreview.net/forum?id=3gDFvmTZXb}
}

@article{bujack2026,
  author       = {Bujack, Roxana Barbara and Karella, Tomas and Cerezo de la Roca, Marco Vinicio Sebastian and DeSantis, Derek Freeman and Shinkle, Emily Suzanne and Allen, Alice Elisabeth Anastasia and Lubbers, Nicholas Edward Williams},
  title        = {Completeness and Discriminability of Tensor-Contraction Moment Invariants},
  institution  = {Los Alamos National Laboratory (LANL), Los Alamos, NM (United States)},
  doi          = {10.2172/17120801},
  url          = {https://www.osti.gov/biblio/17120801},
  place        = {United States},
  year         = {2026},
  month        = {10}}
\clearpage
\appendix
\appendix

\section{Langbein's Algorithm}
\label{appendix:langbein}
This appendix summarizes Langbein's algorithm and explains how it is applied in our method. For a definitive description, readers are directed to the original  paper~\cite{langbein2009generalization}.
\begin{enumerate}
\item Populate all moment tensors up to a designated maximal order $\ell_m$ with random numbers.
\item Determine all zeroth-rank contractions of all of their resulting products up to a specified maximum number of factors $p_\mathit{m}$.
\item Derive their partial derivatives with respect to all moments up to $\ell_m$ as a row vector.
\item Initialize a matrix with the number of columns equivalent to the number of moments and zero rows.
\item For each invariant until the maximum number of potential invariants is reached, append the row vector of the partial derivatives of the subsequent invariant into the matrix if it augments the rank of the matrix.
\end{enumerate}

It is worth noting that when invariance regarding translation and scaling is essential in an application, usually the zeroth and first moments are normalized, thereby voiding their potential for creating rotation invariants~\cite{FSZ16}. We incorporate them in this document for clarity, but Langbein and Hagen initially omitted them in their work.

\section{Number of Invariants of the Irreducible Tensors}
\label{appendix:count-proofs}
This appendix collects the full statements and proofs of the lemmata and theorems referenced in Subsection~\ref{subsection:invariants-counts}.

\ldof*
\begin{proof}
It follows immediately from Def.~\ref{d:mom_tensor} and the commutativity of multiplication that moment tensors of scalar functions are totally symmetric. Therefore, a moment tensor of order $\ell$ has $3^\ell$ entries but only $\frac{(\ell+1)(\ell+2)}2$ degrees of freedom. 
Because of the symmetry, it has exactly one trace, which is a symmetric tensor of order $\ell-2$, i.e., it has $\frac{(\ell-1)\ell}{2}$ degrees of freedom. As a result, an irreducible tensor, i.e., a traceless symmetric tensor, has  $\frac{(\ell+1)(\ell+2)-(\ell-1)\ell}{2}=2\ell+1$ independent moments.
\end{proof}

\ldoftwo*
\begin{proof}
We have seen in the proof of Lem.~\ref{l:dof} that a moment tensor of order $\ell$ has $(\ell+1)(\ell+2)/2$ degrees of freedom. Now we consider the irreducible tensors in its decomposition~\eqref{decomposition}. 
If $\ell$ is even, we get from the decomposition~\eqref{decomposition}, Lem.~\ref{l:dof}, and straightforward calculation, the number of degrees of freedom
\begin{equation}
\sum_{p=0}^{\ell/2}\Bigl(2(2p)+1\Bigr)=\frac12(\ell+1)(\ell+2).
\end{equation}
If $\ell$ is odd, we get from the decomposition~\eqref{decomposition}, Lem.~\ref{l:dof}, and straightforward calculation, the number of degrees of freedom
\begin{equation}
\sum_{p=0}^{(\ell-1)/2}\Bigl(2(2p+1)+1\Bigr)=\frac12(\ell+1)(\ell+2).
\end{equation}
\end{proof}

\lpure*
\begin{proof}
The number of possible independent invariants is the number of independent moments minus the three degrees of freedom of a rotation. In general the degrees of freedom of a 3D rotation is three, which concludes the proof for Lemma~\ref{l:dof}
\begin{equation}
2\ell+1-3=2\ell-2.
\end{equation}
The exceptions $\ell=0$ and $\ell=1$ follow because a zeroth-rank irreducible tensor is a scalar and therefore rotationally invariant. So for it the rotation has no degree of freedom. Further, a first-order irreducible tensor is a vector, for which a 3D rotation has only two degrees of freedom.
\end{proof}

\tpuretwo*
\begin{proof}
We know from Lem.~\ref{l:pure} that an irreducible tensor of rank $p>1$ has $2p-2$ pure invariants, and for $p=0$ or $p=1$, one pure invariant each.
If $\ell$ is even, the decomposition~\eqref{decomposition} has one irreducible tensor of rank $p=0$ and the others are $p>1$, which with $p=2i$ leads to the number of degrees of freedom
\begin{equation}
1+\sum_{i=1}^{\ell/2}\Bigl(2(2i)-2\Bigr) = \frac{\ell^2}2+1.
\end{equation}
If $\ell$ is odd, the decomposition~\eqref{decomposition} has one irreducible tensor of rank $p=1$ and the others are $p>1$, which with $p=2i+1$ leads to the number of degrees of freedom
\begin{equation}
1+\sum_{i=1}^{(\ell-1)/2}\Bigl(2(2i+1)-2\Bigr) =\frac{\ell^2+1}2.
\end{equation}

For $\ell=0$ and $\ell=1$, the moment tensors are irreducible tensors themselves and their norms are automatically their homogeneous, as well as their pure, invariants.
\end{proof}

\lmixed* 
\begin{proof}
We know that the moment tensor of order $\ell>1$ has $\frac12(\ell+1)  (\ell+2)-3$ independent invariants. Theorem~\ref{t:pure2} shows that this differs from the number of degrees of freedom of the pure invariants by
\begin{equation}
\frac12(\ell+1)  (\ell+2)-3 - \left(\frac{\ell^2}2+1\right) = \frac{3\ell-6}2
\end{equation}
if $\ell$ is even.
If we choose one rank $p_0>1$ and add 3 mixed invariants of it paired with any other irreducible tensor of rank $p>1$, we get to this number exactly because 
\begin{equation}
\sum_{q=2}^{\ell/2}3 = \frac{3\ell-6}2.
\end{equation}
If $\ell$ is odd, Theorem~\ref{t:pure2} shows that this differs from the number of degrees of freedom of the pure invariants by
\begin{equation}
\frac12(\ell+1)  (\ell+2)-3 - \left(\frac{\ell^2+1}2\right) = \frac{3\ell-5}2.
\end{equation}
If we choose one rank $p_0>1$ and add three mixed invariants of it paired with any other irreducible tensor of rank $p>1$ and two mixed invariants of it paired with $p=1$, we get to this number exactly because 
\begin{equation}
2+\sum_{q=2}^{(\ell_m-1)/2}3 = \frac{3\ell-5}2.
\end{equation}

For $\ell=0$ and $\ell=1$, we have one homogeneous and one pure invariant each and therefore do not need additional mixed ones.
\end{proof}
\section{Flexible basis}
\label{appendix:flexible}
\tflexible*
\begin{proof}
% For each order $\ell>1$, we have $(\ell+1)(\ell+2)/2$ independent moments. This adds up to
% \begin{equation}
% \sum_{\ell=0}^{\ell_m}(\ell+1)(\ell+2)/2 = (\ell_{m}+1)(\ell_{m}+2)(\ell_{m}+3)/6.
% \end{equation}
% That means, we have $(\ell_{m}+1)(\ell_{m}+2)(\ell_{m}+3)/6-3$ independent invariants up to the maximum order $\ell_m$.
We know from Theorem~\ref{t:pure2} that we have $\frac{\ell^2}2+1$ pure invariants if $\ell$ is even, and $\frac{\ell^2+1}2$ if $\ell$ is odd.
First assume $\ell_m$ is even. Then we have $\frac{\ell_m}2+1$ even and $\frac{\ell_m}2$ odd order tensors. Subtracting the number of their pure invariants from the total number of invariants, we get the number of mixed invariants needed
{\small\begin{equation}\begin{aligned}
\underset{\text{inv.\ of all tensors}}{\underbrace{\sum_{\ell=0}^{\ell_m}\frac{(\ell+1)(\ell+2)}2- 3}}  - 
\underset{\text{pure inv.\ of even tensors}}{\underbrace{\sum_{i=0}^{\ell_m/2} \frac{(2i)^2+2}2}} - 
\underset{\text{pure inv.\ of odd tensors}}{\underbrace{\sum_{i=0}^{\ell_m/2-1} \frac{(2i+1)^2+1}2}} = \frac{3\ell_m^2+4\ell_m-12}{4}.
\end{aligned}\end{equation}}
On the other hand, for even orders $\ell$, we have a certain number of irreducible tensors, one of them of rank 0, and for odd orders $\ell$, we have $\frac{\ell+1}2$ irreducible tensors, one of them of rank 1. That leads to a total of 
\begin{equation}
\sum_{i=0}^{\ell_m/2} 1 = \frac {\ell_m}2+1
\end{equation}
irreducible tensors of rank 0,
\begin{equation}
\sum_{i=0}^{\ell_m/2} 1 = \frac {\ell_m}2
\end{equation}
irreducible tensors of rank 1, and a total of
\begin{equation}
\sum_{i=0}^{\ell_m/2} i + \sum_{i=0}^{\ell_m/2} i = \frac {\ell_m^2}4
\end{equation}
irreducible tensors of rank larger than 1 up to $\ell_m$.
Selecting one of the latter, removing it, and pairing it, gives us two mixed invariants for each pairing with first order and three mixed invariants for each pairing with the remaining higher-order irreducible tensors, which exactly coincides with the number of independent invariants needed
\begin{equation}
2\frac {\ell_m}2 + 3 \left(\frac {\ell_m^2}4-1\right) = \frac{3\ell_m^2+4\ell_m-12}{4}.
\end{equation}

If $\ell_m$ is odd, the same arguments lead to $\frac{3\ell_m^2+4\ell_m-11}{4}$ invariants needed.
% Now assume $\ell_m$ is odd. Then, we have $\frac{\ell_m+1}2$ even and odd order tensors. Subtracting these from the total number of invariants, we get
% {\small\begin{equation}\begin{aligned}
% \underset{\text{inv.\ of all tensors}}{\underbrace{\sum_{\ell=0}^{\ell_m}(\ell+1)(\ell+2)/2- 3}}  
%  - \underset{\text{pure inv.\ of even tensors}}{\underbrace{\sum_{i=0}^{(\ell_m-1)/2} \frac{(2i)^2+2}2}}
%  - \underset{\text{pure inv.\ of odd tensors}}{\underbrace{\sum_{i=0}^{(\ell_m-1)/2} \frac{(2i+1)^2+1}2}} 
%   = \frac{3\ell_m^2+4\ell_m-11}{4}.
% \end{aligned}\end{equation}}
% On the other hand, we have a certain number of irreducible tensors for even orders $\ell$, one of them of rank 0, and for odd orders $\ell$, one of them of rank 1. That leads to a total of 
% \begin{equation}
% \sum_{i=0}^{(\ell_m-1)/2} 1 = \frac {\ell_m+1}2
% \end{equation}
% irreducible tensors of rank 0,
% \begin{equation}
% \sum_{i=0}^{\ell_m/2} 1 =\frac {\ell_m+1}2
% \end{equation}
% irreducible tensors of rank 1, and a total of
% \begin{equation}
% \sum_{i=0}^{\ell_m/2} i + \sum_{i=0}^{\ell_m/2} i = \frac {\ell_m^2-1}4
% \end{equation}
% irreducible tensors of rank larger than 1 up to $\ell_m$.
% Selecting one of the latter, removing it, and pairing it, gives us two mixed invariants for each pairing with first order and three mixed invariants for each pairing with the remaining higher-order irreducible tensors, which again exactly coincides with the number of independent invariants needed:
% \begin{equation}
% 2\frac {\ell_m+1}2 + 3 \left(\frac {\ell_m^2-1}4-1\right) = \frac{3\ell_m^2+4\ell_m-11}{4}.
% \end{equation}
%
The exception for $\ell_m=1$ follows from the exception in Lem.~\ref{l:mixed}.
\end{proof}

\section{Formulas for Example~\ref{ex:22}}
\label{appendix:example-mixed}
In this section we provide the the complete list of invariants with
their contraction information as formulas for Example~\ref{ex:22}, in
which the irreducible tensor ${}{^2_2}H$ is non-zero.
Let ${}{^0}M,{}{^1}M,{}{^2}M,{}{^3}M$ be the
moment tensors up to order $l_m=3$. Example~\ref{ex:irreducible} shows
the six irreducible tensors in the decompositions:
${}{^0_0}H,{}{^1_1}H,{}{^2_2}H,{}{_{0}^{2}}H,{}{^3_3}H,{}{_{1}^{3}}H$.

There are always 10 pure invariants, which do not depend on the choice
of the robust irreducible tensor. We then add the seven mixed
invariants that contain ${}{^2_2}H$. This totals 17 invariants,
which correctly matches the predicted number from Table~\ref{t:lm}.

\subsection{Pure Invariants: 1+1+3+5}

\begin{equation}
   \raisebox{-0.4\height}{\includegraphics[height=6\baselineskip]{pics/pure_0_0_0.png}}
   \begin{aligned}
      \quad
      {}{^0_0}H
   \end{aligned}
\end{equation}

\begin{equation}
   \raisebox{-0.4\height}{\includegraphics[height=6\baselineskip]{pics/pure_1_1_0.png}}
   \begin{aligned}
      % {}{^1_1}H^{(1,2)}
      \quad
      {}{^1_1}H^2(1)  (1)
   \end{aligned}
\end{equation}

\begin{equation}
   \raisebox{-0.4\height}{\includegraphics[height=6\baselineskip]{pics/pure_2_0_0.png}}
   \begin{aligned}
      \quad
      {}{^2_0}H
   \end{aligned}
\end{equation}

\begin{equation}
   \raisebox{-0.4\height}{\includegraphics[height=6\baselineskip]{pics/pure_2_2_0.png}}
   \begin{aligned}
      \quad
      {}{^2_2}H^2(1,2)(1,2)
   \end{aligned}
\end{equation}

\begin{equation}
   \raisebox{-0.4\height}{\includegraphics[height=6\baselineskip]{pics/pure_2_2_1.png}}
   \begin{aligned}
      \quad
      {}{^2_2}H^3(1,2)(2,3)(1,3)
   \end{aligned}
\end{equation}

% \subsubsection{Pure 3 Invariants: 4}
\begin{equation}
   \raisebox{-0.4\height}{\includegraphics[height=6\baselineskip]{pics/pure_3_1_0.png}}
   \begin{aligned}
      \quad
      {}{^3_1}H^2(1)(1)
   \end{aligned}
\end{equation}

\begin{equation}
   \raisebox{-0.4\height}{\includegraphics[height=6\baselineskip]{pics/pure_3_3_0.png}}
   \begin{aligned}
      \quad
      {}{^3_3}H^2(1,2,3)(1,2,3)
   \end{aligned}
\end{equation}

\begin{equation}
   \raisebox{-0.4\height}{\includegraphics[height=6\baselineskip]{pics/pure_3_3_1.png}}
   \begin{aligned}
      \quad
      {}{^3_3}H^4(1,2,3)  (1,2,4)  (3,5,6)  (4,5,6)
   \end{aligned}
\end{equation}

\begin{equation}
   \raisebox{-0.4\height}{\includegraphics[height=6\baselineskip]{pics/pure_3_3_2.png}}
   \begin{aligned}
      \quad
      {}{^3_3}H^6(1, 2, 3)  (2, 3, 4)  (1, 4, 5)  (5, 6, 7)  (7, 8, 9)  (6, 8, 9)
   \end{aligned}
\end{equation}

\begin{equation}
   \raisebox{-0.4\height}{\includegraphics[height=6\baselineskip]{pics/pure_3_3_3.png}}
   \begin{aligned}
      \quad
      {}{^3_3}H^{10}(1, 2, 3)  (2, 3, 4)  (1, 4, 5)  (5, 6, 7)  (6, 7, 8)  (8, 9, 10) \\
      (9, 10, 11)  (11, 12, 13)  (13, 14, 15)  (12, 14, 15)
   \end{aligned}
\end{equation}

\subsection{Mixed Invariants: 2+2+3}
% \begin{figure}[ht]
% \centering
% {\includegraphics[width=0.14\linewidth]{pics/mixed_1_1_2_2_0.png}}\hfill
% {\includegraphics[width=0.14\linewidth]{pics/mixed_1_1_2_2_1.png}}\hfill
% {\includegraphics[width=0.14\linewidth]{pics/mixed_3_1_2_2_0.png}}\hfill
% {\includegraphics[width=0.14\linewidth]{pics/mixed_3_1_2_2_1.png}}\hfill
% {\includegraphics[width=0.14\linewidth]{pics/mixed_2_2_3_3_0.png}}\hfill
% {\includegraphics[width=0.14\linewidth]{pics/mixed_2_2_3_3_1.png}}\hfill
% {\includegraphics[width=0.14\linewidth]{pics/mixed_2_2_3_3_2.png}}\hfill
% \caption{Mixed Invariants: 2+2+3} 
% \end{figure}

\begin{equation}
   \raisebox{-0.4\height}{\includegraphics[height=6\baselineskip]{pics/mixed_1_1_2_2_0.png}}
   \begin{aligned}
      \quad
      {}{^1_1}H^2{}{^2_2}H(1)(2)(1,2)
   \end{aligned}
\end{equation}

\begin{equation}
   \raisebox{-0.4\height}{\includegraphics[height=6\baselineskip]{pics/mixed_1_1_2_2_1.png}}
   \begin{aligned}
      \quad
      {}{^1_1}H^2{}{^2_2}H^2(1)(2)(1,3)(2,3)
   \end{aligned}
\end{equation}

\begin{equation}
   \raisebox{-0.4\height}{\includegraphics[height=6\baselineskip]{pics/mixed_3_1_2_2_0.png}}
   \begin{aligned}
      \quad
      {}{^3_1}H^2{}{^2_2}H(1)(2)(1,2)
   \end{aligned}
\end{equation}

\begin{equation}
   \raisebox{-0.4\height}{\includegraphics[height=6\baselineskip]{pics/mixed_3_1_2_2_1.png}}
   \begin{aligned}
      \quad
      {}{^3_1}H^2{}{^2_2}H^2(1)(2)(1,3)(2,3)
   \end{aligned}
\end{equation}

\begin{equation}
   \raisebox{-0.4\height}{\includegraphics[height=6\baselineskip]{pics/mixed_2_2_3_3_0.png}}
   \begin{aligned}
      \quad
      {}{^2_2}H{}{^3_3}H^2(1, 2)  (2, 3, 4)  (1, 3, 4)
   \end{aligned}
\end{equation}

\begin{equation}
   \raisebox{-0.4\height}{\includegraphics[height=6\baselineskip]{pics/mixed_2_2_3_3_1.png}}
   \begin{aligned}
      \quad
      {}{^2_2}H^2{}{^3_3}H^2(1, 2)  (2, 3)  (1, 4, 5)  (3, 4, 5)
   \end{aligned}
\end{equation}

\begin{equation}
   \raisebox{-0.4\height}{\includegraphics[height=6\baselineskip]{pics/mixed_2_2_3_3_2.png}}
   \begin{aligned}
      \quad
      {}{^2_2}H^2{}{^3_3}H^2(1, 2)  (3, 4)  (1, 2, 5)  (3, 4, 5)
   \end{aligned}
\end{equation}

\section{Formulas for Example~\ref{ex:minimal}}
\label{appendix:example-minimal}
The formulas for the minimal flexible set of Example~\ref{ex:minimal}
are exactly the ones in \ref{a:minimal} up to order $l_m=3$.

\section{Formulas for Example~\ref{ex:spherical22}}
\label{appendix:example-spherical}
In this section, we provide the complete list of invariants with their
contraction information as formulas for Example~\ref{ex:spherical22},
in which the irreducible tensor ${}{^2_2}H$ is non-zero
for a spherical function. Example~\ref{ex:irreducibleTensors2} showed
the four full-order irreducible tensors in the decompositions:
${}{^0_0}H,{}{^1_1}H,{}{^2_2}H,{}{^3_3}H$.

There are always eight pure invariants, which do not depend on the
choice of the robust irreducible tensor. Then we add the five mixed
invariants that contain ${}{^2_2}H$. This totals 13 invariants,
which correctly matches the predicted number from
Table~\ref{t:spherical}.

\subsection{Pure Invariants: 1+1+2+4}
\begin{equation}
   \raisebox{-0.4\height}{\includegraphics[height=6\baselineskip]{pics/pure_0_0_0.png}}
   \begin{aligned}
      \quad
      {}{^0_0}H
   \end{aligned}
\end{equation}

\begin{equation}
   \raisebox{-0.4\height}{\includegraphics[height=6\baselineskip]{pics/pure_1_1_0.png}}
   \begin{aligned}
      % {}{^1_1}H^{(1,2)}
      \quad
      {}{^1_1}H^2(1)  (1)
   \end{aligned}
\end{equation}

\begin{equation}
   \raisebox{-0.4\height}{\includegraphics[height=6\baselineskip]{pics/pure_2_2_0.png}}
   \begin{aligned}
      \quad
      {}{^2_2}H^2(1,2)(1,2)
   \end{aligned}
\end{equation}

\begin{equation}
   \raisebox{-0.4\height}{\includegraphics[height=6\baselineskip]{pics/pure_2_2_1.png}}
   \begin{aligned}
      \quad
      {}{^2_2}H^3(1,2)(2,3)(1,3)
   \end{aligned}
\end{equation}

\begin{equation}
   \raisebox{-0.4\height}{\includegraphics[height=6\baselineskip]{pics/pure_3_3_0.png}}
   \begin{aligned}
      \quad
      {}{^3_3}H^2(1,2,3)(1,2,3)
   \end{aligned}
\end{equation}

\begin{equation}
   \raisebox{-0.4\height}{\includegraphics[height=6\baselineskip]{pics/pure_3_3_1.png}}
   \begin{aligned}
      \quad
      {}{^3_3}H^4(1,2,3)  (1,2,4)  (3,5,6)  (4,5,6)
   \end{aligned}
\end{equation}

\begin{equation}
   \raisebox{-0.4\height}{\includegraphics[height=6\baselineskip]{pics/pure_3_3_2.png}}
   \begin{aligned}
      \quad
      {}{^3_3}H^6(1, 2, 3)  (2, 3, 4)  (1, 4, 5)  (5, 6, 7)  (7, 8, 9)  (6, 8, 9)
   \end{aligned}
\end{equation}

\begin{equation}
   \raisebox{-0.4\height}{\includegraphics[height=6\baselineskip]{pics/pure_3_3_3.png}}
   \begin{aligned}
      \quad
      {}{^3_3}H^{10}(1, 2, 3)  (2, 3, 4)  (1, 4, 5)  (5, 6, 7)  (6, 7, 8)  (8, 9, 10) \\
      (9, 10, 11)  (11, 12, 13)  (13, 14, 15)  (12, 14, 15)
   \end{aligned}
\end{equation}

% \centering
% {\includegraphics[width=0.12\linewidth]{pics/pure_0_0_0.png}}\hfill
% {\includegraphics[width=0.12\linewidth]{pics/pure_1_1_0.png}}\hfill
% {\includegraphics[width=0.12\linewidth]{pics/pure_2_2_0.png}}\hfill
% {\includegraphics[width=0.12\linewidth]{pics/pure_2_2_1.png}}\hfill
% {\includegraphics[width=0.12\linewidth]{pics/pure_3_3_0.png}}\hfill
% {\includegraphics[width=0.12\linewidth]{pics/pure_3_3_1.png}}\hfill
% {\includegraphics[width=0.12\linewidth]{pics/pure_3_3_2.png}}\hfill
% {\includegraphics[width=0.12\linewidth]{pics/pure_3_3_3.png}}
% \caption{The 8 pure invariants of a spherical function up to $l_m=3$: 1 of ${}{^0_0}H$, 1 of ${}{^1_1}H$,  2 of ${}{^2_2}H$, and 4 of ${}{_{3}^{3}}H$).\label{f:exSphericalPure}} 
% \end{figure}

% \begin{figure}[ht]
% \centering
% {\includegraphics[width=0.19\linewidth]{pics/mixed_1_1_2_2_0.png}}\hfill
% {\includegraphics[width=0.19\linewidth]{pics/mixed_1_1_2_2_1.png}}\hfill
% {\includegraphics[width=0.19\linewidth]{pics/mixed_2_2_3_3_0.png}}\hfill
% {\includegraphics[width=0.19\linewidth]{pics/mixed_2_2_3_3_1.png}}\hfill
% {\includegraphics[width=0.19\linewidth]{pics/mixed_2_2_3_3_2.png}}\hfill
% \caption{The 5 mixed invariants of a spherical function up to $l_m=3$ centered around ${}{^2_2}H$: two with ${}{^1_1}H$ and three with ${}{_{3}^{3}}H$.\label{f:exSphericalMixed1}} 
% \end{figure}
% \end{ex}
\subsection{Mixed Invariants: 2+3}

\begin{equation}
   \raisebox{-0.4\height}{\includegraphics[height=6\baselineskip]{pics/mixed_1_1_2_2_0.png}}
   \begin{aligned}
      \quad
      {}{^1_1}H^2{}{^2_2}H(1)(2)(1,2)
   \end{aligned}
\end{equation}

\begin{equation}
   \raisebox{-0.4\height}{\includegraphics[height=6\baselineskip]{pics/mixed_1_1_2_2_1.png}}
   \begin{aligned}
      \quad
      {}{^1_1}H^2{}{^2_2}H^2(1)(2)(1,3)(2,3)
   \end{aligned}
\end{equation}

\begin{equation}
   \raisebox{-0.4\height}{\includegraphics[height=6\baselineskip]{pics/mixed_2_2_3_3_0.png}}
   \begin{aligned}
      \quad
      {}{^2_2}H{}{^3_3}H^2(1, 2)  (2, 3, 4)  (1, 3, 4)
   \end{aligned}
\end{equation}

\begin{equation}
   \raisebox{-0.4\height}{\includegraphics[height=6\baselineskip]{pics/mixed_2_2_3_3_1.png}}
   \begin{aligned}
      \quad
      {}{^2_2}H^2{}{^3_3}H^2(1, 2)  (2, 3)  (1, 4, 5)  (3, 4, 5)
   \end{aligned}
\end{equation}

\begin{equation}
   \raisebox{-0.4\height}{\includegraphics[height=6\baselineskip]{pics/mixed_2_2_3_3_2.png}}
   \begin{aligned}
      \quad
      {}{^2_2}H^2{}{^3_3}H^2(1, 2)  (3, 4)  (1, 2, 5)  (3, 4, 5)
   \end{aligned}
\end{equation}

\section{Formulas for Example~\ref{ex:minimalSpherical}}
\label{appendix:example-minimalSpherical}
The formulas for the minimal flexible set of
Example~\ref{ex:minimalSpherical} are exactly the ones in
\ref{a:minimalSpherical} up to order $l_m=3$.

\section{Spherical Harmonics and Irreducible Tensors}
There is a well-known relation between totally symmetric tensors and
spherical harmonics~\cite{coope1965irreducible, backus1970geometrical,
   hergl2020introduction}.

Given a 3D symmetric tensor $S$ of order $\ell$, the polynomial
generated by $S$ is defined as
\begin{equation}\label{polynomial}
   \begin{aligned}
      p_S(x)={}{^\ell}S_{i_1\ldots i_\ell}x_{i_1}\ldots x_{i_\ell},
   \end{aligned}
\end{equation}
a homogeneous polynomial of degree $\ell$.  There is a one-to-one
mapping between $S$ and its generated polynomial
\begin{equation}\label{polynomial2}
   \begin{aligned}
      S_{i_1\ldots i_\ell}=\frac{1}{\ell!}\partial x_{i_1}\ldots \partial x_{i_\ell} p_S(x).
   \end{aligned}
\end{equation}
In other words, there is an isomorphism between the linear space of
homogeneous polynomials $\mathcal P^\ell$ of degree $\ell$ and the
linear space of totally symmetric tensors $\mathcal S^\ell$ of order
$\ell$.

Homogeneous polynomials can be interpreted as functions on the
sphere. Harmonic homogeneous polynomials, i.e., polynomials whose
second derivative vanishes,
\begin{equation}\label{harmonic}
   \begin{aligned}
      \nabla^2p(x)=\Delta p(x)= \sum_{i=1}^3\frac{p(x)}{\partial x_{i}^2}=0,
   \end{aligned}
\end{equation}
are called \textbf{spherical harmonics}.  Every homogeneous polynomial
${}{^\ell}p$ of degree $\ell$ can be decomposed into a unique
series of spherical harmonics $h$
\begin{equation}
   \begin{aligned}
      {}{^\ell}p(x)={}{^\ell}h(x)+r^2{}{^{\ell-2}}h(x)+\ldots \, .
   \end{aligned}
\end{equation}

Applying the relation~\eqref{polynomial} to a spherical harmonic
$p_H(x)=h(x)$, we see that their corresponding tensors are irreducible
tensors $H$. Their traces vanish because
\begin{equation}
   \begin{aligned}
      H{^{(1,2)}}_{i_1\ldots i_\ell} & = \sum_{i=1}^3H_{iii_3\ldots i_\ell}
      \\&\overset{\eqref{polynomial2}}= \sum_{i=1}^3\frac{1}{\ell!}\partial x_{i}\partial x_{i}\partial x_{i_3}\ldots \partial x_{i_\ell} p_H(x)
      \\&\overset{\eqref{harmonic}}=\frac{1}{\ell!}\partial x_{i_3}\ldots \partial x_{i_\ell}\Delta p_H(x)
      \\&\overset{\eqref{harmonic}}=0.
   \end{aligned}
\end{equation}
Using the spherical harmonic decomposition of a homogeneous polynomial
and the above isomorphism, we see that every symmetric tensor can be
decomposed into a unique series of irreducible tensors
\begin{equation}
   \begin{aligned}
      {}{^\ell}{M}_{i_1 \dots i_\ell}
      ={}{^\ell_{\ell}}{H}_{i_1 i_\ell}
      +{}{^\ell_{\ell-2}}{H}_{(i_1 \dots i_{\ell-2}} \delta_{i_{\ell-1}i_{\ell})}
      +{}{^\ell_{\ell-4}}{H}_{(i_1 \dots i_{\ell-4}}
      \delta_{i_{\ell-3}i_{\ell-2}}\delta_{i_{\ell-1}i_{\ell})}+\cdots.
   \end{aligned}
\end{equation}
The parentheses $(i_1 \ldots  i_\ell)$ indicate symmetrization
\begin{equation}
   \begin{aligned}
      T_{(i_1 \dots i_\ell)}
      := \frac1{\ell!}
      \sum_{\pi \in S_\ell}
      T_{\pi(i_1)\dots\pi(i_\ell)},
   \end{aligned}
\end{equation}
where the sum is taken over all permutations $\pi$ in the symmetric group $S_\ell$.

\section{The Minimal Flexible Set}\label{a:minimal}
Even though using the minimal flexible set is not the most efficient
solution, it is the absolutely bulletproof solution. Further, it
allows a faster derivation of a basis. Instead of applying Langbein's
algorithm, a basis can be derived from it by removing mixed invariants
until the correct number from Table~\ref{t:lm} is achieved. As we have
seen, this can be done by selecting a robust irreducible tensor and
removing every mixed invariant that does not contain
it. Alternatively, mixed invariants can be removed based on their
orders, numbers of factors, or structure. Using low orders and factors
allows faster evaluation and greater robustness.

Here we present a minimal flexible set up to order six.  The
subsections structure the invariants by order. For each we provide the
irreducible decomposition with the number of summands displayed in the
sub-subsections' names, the pure invariants with the number of them,
the mixed homogeneous invariants with the number of them, and the
mixed simultaneous invariants with the number of them. The numbers
that were expected from Tables~\ref{t:l} and~\ref{t:lm} have always
been found exactly by our algorithm. We also provide the graph
visualization for the invariants.  Note the patterns of mixed
invariants of irreducible tensors of the same orders: it does not
matter which moment tensors they were originally derived from, the
shape of the invariants is always the same. For example, the mixed
homogeneous invariants of moment order three have the same structure
as the mixed simultaneous invariants between any third-rank
irreducible tensor with any first-rank irreducible tensor. Therefore
it is sufficient to look only at the highest-rank irreducible tensor
in each decomposition and its simultaneous invariants to understand
the structure of all invariants. These numbers are depicted in bold
text if the structure is new.

An interesting observation is that the third-order irreducible tensor
is thus far the only one that we have encountered that required an
exponent as large as 10.

An open-source implementation for the generation of the invariant feature sets and the graph visualizations of this paper is publicly available at \url{github.com/lanl/rotation-invariant-neural-networks}.

\subsection{Order 0}

\subsubsection{Irreducible Decomposition: 1}
\begin{equation}
   \begin{aligned}
      {}{^0}M={}{^0_0}H
   \end{aligned}
\end{equation}
\subsubsection{Pure Invariants: \textbf{1}}
\begin{equation}
   \raisebox{-0.4\height}{\includegraphics[height=6\baselineskip]{pics/pure_0_0_0}}
   \begin{aligned}
      \quad
      {}{^0_0}H
   \end{aligned}
\end{equation}

% \subsubsection{mixed^ Invariants: 0}
% \subsubsection{Mixed Homogeneous Invariants: 0}
% \subsubsection{Mixed Simultaneous Invariants: 0 }

\subsection{Order 1}

\subsubsection{Irreducible Decomposition: 1}
\begin{equation}
   \begin{aligned}
      {}{^1}M_i & ={}{^1_1}H_i
   \end{aligned}
\end{equation}

\subsubsection{Pure Invariants: \textbf{1}}
\begin{equation}
   \raisebox{-0.4\height}{\includegraphics[height=6\baselineskip]{pics/pure_1_1_0}}
   \begin{aligned}
      % {}{^1_1}H^{(1,2)}
      \quad
      {}{^1_1}H^2(1)  (1)
   \end{aligned}
\end{equation}

% \subsubsection{mixed^ Invariants: 0}
% \subsubsection{Mixed Homogeneous Invariants: 0}
% \subsubsection{Mixed Simultaneous Invariants: 0 }

\subsection{Order 2}

\subsubsection{Irreducible Decomposition: 2}
\begin{equation}
   \begin{aligned}
      {}{^2}M_{ij} & ={}{^2_2}H_{ij}+{}{_{0}^{2}}H\delta_{ij} \\
      % {}{^2_2}H_{ij} &={}{^2}M_{ij} - \frac13{}{^2}M_{\iota\iota}\delta_{ij}\\
      % {}{_{0}^{2}}H &=\frac13{}{^2}M_{\iota\iota}
   \end{aligned}
\end{equation}
% \begin{equation}
% \begin{aligned} 
% {}{^2}M_{ij}&={}{^2_2}H_{ij}+{}{_{0}^{2}}H\delta_{ij}\\
% {}{^2_2}H_{ij} &={}{^2}M_{ij} - \frac13{}{^2}M_{\iota\iota}\delta_{ij}\\
% {}{_{0}^{2}}H &=\frac13{}{^2}M_{\iota\iota}
% \end{aligned} 
% \end{equation}
% \begin{equation}
% \begin{aligned} 
% {}{^2}M_{(ij)}&={}{^2_2}H_{(ij)}+{}{_{0}^{2}}H\delta_{(ij)}\\
% {}{^2_2}H_{(ij)} &={}{^2}M_{(ij)} - \frac13{}{^2}M_{\iota\iota}\delta_{(ij)}\\
% {}{_{0}^{2}}H &=\frac13{}{^2}M_{\iota\iota}
% \end{aligned} 
% \end{equation}
% adds up
% \begin{equation}
% \begin{aligned} 
% {}{^2}M_{(ij)}&={}{^2_2}H_{(ij)}+{}{_{0}^{2}}H\delta_{(ij)}
% \\ &={}{^2}M_{(ij)} - \frac13{}{^2}M_{\iota\iota}\delta_{(ij)} 
% + \frac13{}{^2}M_{\iota\iota}\delta_{(ij)}
% % \\0&=0
% \end{aligned} 
% \end{equation}
% and has zero trace
% \begin{equation}
% \begin{aligned} 
% \delta_{(ij}{}{^2_2}H_{ij)} &= {}{^2_2}H_{\iota\iota}
% \\&={}{^2}M_{\iota\iota} - \frac13{}{^2}M_{\iota\iota}\delta_{\iota\iota}
% \\&={}{^2}M_{\iota\iota} - \frac13{}{^2}M_{\iota\iota}3
% \\&=0
% \end{aligned} 
% \end{equation}
% % \begin{equation}
% \begin{aligned} 
% % \delta_{(ij}{}{^2_2}H_{ij)} &=\delta_{(ij}{}{^2}M_{ij)} - \frac13{}{^2}M_{\iota\iota}\delta_{(ij}\delta_{ij)}
% % \\&={}{^2}M_{\iota\iota} - \frac13{}{^2}M_{\iota\iota}3
% % \\&=0
% % \end{aligned} 
% \end{equation}
% or 

\subsubsection{Pure Invariants: 1+\textbf{2}}
\begin{equation}
   \raisebox{-0.4\height}{\includegraphics[height=6\baselineskip]{pics/pure_2_0_0}}
   \begin{aligned}
      \quad
      {}{^2_0}H
   \end{aligned}
\end{equation}

% \subsubsection{Pure 2 Invariants: 2}
\begin{equation}
   \raisebox{-0.4\height}{\includegraphics[height=6\baselineskip]{pics/pure_2_2_0.png}}
   \begin{aligned}
      \quad
      {}{^2_2}H^2(1,2)(1,2)
   \end{aligned}
\end{equation}

\begin{equation}
   \raisebox{-0.4\height}{\includegraphics[height=6\baselineskip]{pics/pure_2_2_1.png}}
   \begin{aligned}
      \quad
      {}{^2_2}H^3(1,2)(2,3)(1,3)
   \end{aligned}
\end{equation}

% \subsubsection{Mixed Homogeneous Invariants: 0}

\subsubsection{Mixed Simultaneous 1 2 Invariants: \textbf{2}}
\begin{equation}
   \raisebox{-0.4\height}{\includegraphics[height=6\baselineskip]{pics/mixed_1_1_2_2_0.png}}
   \begin{aligned}
      \quad
      {}{^1_1}H^2{}{^2_2}H(1)(2)(1,2)
   \end{aligned}
\end{equation}

\begin{equation}
   \raisebox{-0.4\height}{\includegraphics[height=6\baselineskip]{pics/mixed_1_1_2_2_1.png}}
   \begin{aligned}
      \quad
      {}{^1_1}H^2{}{^2_2}H^2(1)(2)(1,3)(2,3)
   \end{aligned}
\end{equation}

\subsection{Order 3}

\subsubsection{Irreducible Decomposition: 2}
\begin{equation}
   \begin{aligned}
      {}{^3}M_{ijk} & ={}{^3_3}H_{ijk}+
      {}{^3_{1}}{H}_{(i }
      \delta_{jk)}
      % \\&={}{^3_3}H_{ijk}+\frac13({}{_{1}^{3}}H_{i}\delta_{jk}+{}{_{1}^{3}}H_{j}\delta_{ik}+{}{_{1}^{3}}H_{k}\delta_{ij})\\
      % {}{^3_3}H_{ijk}&={}{^3}M_{ijk} - \frac15({}{^3}M_{i\iota\iota}\delta_{jk}+{}{^3}M_{\iota j\iota}\delta_{ik}+{}{^3}M_{\iota\iota k}\delta_{ij})\\
      % {}{_{1}^{3}}H_{i} &= \frac15{}{^3}M_{i\iota\iota}.
   \end{aligned}
\end{equation}
% \begin{equation}
% \begin{aligned} 
% {}{^3}M_{ijk}&={}{^3_3}H_{ijk}+
% {}{^3_{1}}{H}_{(i }
%    \delta_{jk)}
% \\&={}{^3_3}H_{ijk}+\frac13({}{_{1}^{3}}H_{i}\delta_{jk}+{}{_{1}^{3}}H_{j}\delta_{ik}+{}{_{1}^{3}}H_{k}\delta_{ij})\\
% {}{^3_3}H_{ijk}&={}{^3}M_{ijk} - \frac15({}{^3}M_{i\iota\iota}\delta_{jk}+{}{^3}M_{\iota j\iota}\delta_{ik}+{}{^3}M_{\iota\iota k}\delta_{ij})\\
% {}{_{1}^{3}}H_{i} &= \frac15{}{^3}M_{i\iota\iota}.
% \end{aligned} 
% \end{equation}

% \begin{equation}
% \begin{aligned} 
% {}{^3}M_{(ijk)}&={}{^3_3}H_{(ijk)}+{}{^3_{1}}{H}_{(i }\delta_{jk)}
% \\
% {}{^3_3}H_{(ijk)}&={}{^3}M_{(ijk)} - \frac15{}{^3}M_{(i\iota\iota}\delta_{jk)}
% \\
% {}{_{1}^{3}}H_{(i)} &= \frac15{}{^3}M_{(i\iota\iota)}.
% \end{aligned} 
% \end{equation}
% adds up
% and has zero trace
% \begin{equation}
% \begin{aligned} 
% \delta_{(ij}{}{^3_3}H_{ijk)}&=
% \delta_{(ij}{}{^3}M_{ijk)} - \frac15{}{^3}M_{(i\iota\iota}\delta_{ij}\delta_{jk)}
% \\&={}{^3}M_{(\iota\iota k)} - \frac15{}{^3}M_{(i\iota\iota}\delta_{ij}\delta_{jk)}
% \end{aligned} 
% \end{equation}

\subsubsection{Pure Invariants: 1+\textbf{4}}
\begin{equation}
   \raisebox{-0.4\height}{\includegraphics[height=6\baselineskip]{pics/pure_3_1_0.png}}
   \begin{aligned}
      \quad
      {}{^3_1}H^2(1)(1)
   \end{aligned}
\end{equation}

% \subsubsection{Pure 3 Invariants: 4}
\begin{equation}
   \raisebox{-0.4\height}{\includegraphics[height=6\baselineskip]{pics/pure_3_3_0.png}}
   \begin{aligned}
      \quad
      {}{^3_3}H^2(1,2,3)(1,2,3)
   \end{aligned}
\end{equation}

\begin{equation}
   \raisebox{-0.4\height}{\includegraphics[height=6\baselineskip]{pics/pure_3_3_1.png}}
   \begin{aligned}
      \quad
      {}{^3_3}H^4(1,2,3)  (1,2,4)  (3,5,6)  (4,5,6)
   \end{aligned}
\end{equation}

\begin{equation}
   \raisebox{-0.4\height}{\includegraphics[height=6\baselineskip]{pics/pure_3_3_2.png}}
   \begin{aligned}
      \quad
      {}{^3_3}H^6(1, 2, 3)  (2, 3, 4)  (1, 4, 5)  (5, 6, 7)  (7, 8, 9)  (6, 8, 9)
   \end{aligned}
\end{equation}

\begin{equation}
   \raisebox{-0.4\height}{\includegraphics[height=6\baselineskip]{pics/pure_3_3_3.png}}
   \begin{aligned}
      \quad
      {}{^3_3}H^{10}(1, 2, 3)  (2, 3, 4)  (1, 4, 5)  (5, 6, 7)  (6, 7, 8)  (8, 9, 10) \\ (9, 10, 11)  (11, 12, 13)  (13, 14, 15)  (12, 14, 15)
   \end{aligned}
\end{equation}

\subsubsection{Mixed Homogeneous Invariants: \textbf{2}}
\begin{equation}
   \raisebox{-0.4\height}{\includegraphics[height=6\baselineskip]{pics/mixed_3_1_3_3_0.png}}
   \begin{aligned}
      \quad
      {}{^3_1}H^3{}{^3_3}H(1)(2)(3)(1, 2, 3)
   \end{aligned}
\end{equation}

\begin{equation}
   \raisebox{-0.4\height}{\includegraphics[height=6\baselineskip]{pics/mixed_3_1_3_3_1}}
   \begin{aligned}
      \quad
      {}{^3_1}H^2{}{^3_3}H^2(1)  (2)  (1, 3, 4)  (2, 3,4)
   \end{aligned}
\end{equation}

\subsubsection{Mixed Simultaneous 1 3 Invariants: \textbf{1}+2}
\begin{equation}
   \raisebox{-0.4\height}{\includegraphics[height=6\baselineskip]{pics/mixed_1_1_3_1_0}}
   \begin{aligned}
      \quad
      {}{^1_1}H{}{^3_1}H(1)(1)
   \end{aligned}
\end{equation}

\begin{equation}
   \raisebox{-0.4\height}{\includegraphics[height=6\baselineskip]{pics/mixed_1_1_3_3_0.png}}
   \begin{aligned}
      \quad
      {}{^1_1}H^3{}{^3_3}H(1)(2)(3)(1, 2, 3)
   \end{aligned}
\end{equation}

\begin{equation}
   \raisebox{-0.4\height}{\includegraphics[height=6\baselineskip]{pics/mixed_1_1_3_3_1.png}}
   \begin{aligned}
      \quad
      {}{^1_1}H^2{}{^3_3}H^2(1)  (2)  (1,3,4)  (2,3,4)
   \end{aligned}
\end{equation}

\subsubsection{Mixed Simultaneous 2 3 Invariants: 2+\textbf{3}}
\begin{equation}
   \raisebox{-0.4\height}{\includegraphics[height=6\baselineskip]{pics/mixed_3_1_2_2_0.png}}
   \begin{aligned}
      \quad
      {}{^3_1}H^2{}{^2_2}H(1)(2)(1,2)
   \end{aligned}
\end{equation}

\begin{equation}
   \raisebox{-0.4\height}{\includegraphics[height=6\baselineskip]{pics/mixed_3_1_2_2_1.png}}
   \begin{aligned}
      \quad
      {}{^3_1}H^2{}{^2_2}H^2(1)(2)(1,3)(2,3)
   \end{aligned}
\end{equation}

\begin{equation}
   \raisebox{-0.4\height}{\includegraphics[height=6\baselineskip]{pics/mixed_2_2_3_3_0.png}}
   \begin{aligned}
      \quad
      {}{^2_2}H{}{^3_3}H^2(1, 2)  (2, 3, 4)  (1, 3, 4)
   \end{aligned}
\end{equation}

\begin{equation}
   \raisebox{-0.4\height}{\includegraphics[height=6\baselineskip]{pics/mixed_2_2_3_3_1.png}}
   \begin{aligned}
      \quad
      {}{^2_2}H^2{}{^3_3}H^2(1, 2)  (2, 3)  (1, 4, 5)  (3, 4, 5)
   \end{aligned}
\end{equation}

\begin{equation}
   \raisebox{-0.4\height}{\includegraphics[height=6\baselineskip]{pics/mixed_2_2_3_3_2.png}}
   \begin{aligned}
      \quad
      {}{^2_2}H^2{}{^3_3}H^2(1, 2)  (3, 4)  (1, 2, 5)  (3, 4, 5)
   \end{aligned}
\end{equation}

\subsection{Order 4}

\subsubsection{Irreducible Decomposition: 3}
\begin{equation}
   \begin{aligned}
      {}{^4}M_{ijkl} & =
      {}{^4_4}H_{ijkl}
      +{}{^4_{2}}{H}_{(ij }  \delta_{kl)}
      +{}{^4_{0}}{H}   \delta_{(ij}\delta_{kl)}
   \end{aligned}
\end{equation}

\subsubsection{Pure Invariants: 1+2+\textbf{6}}
\begin{equation}
   \raisebox{-0.4\height}{\includegraphics[height=6\baselineskip]{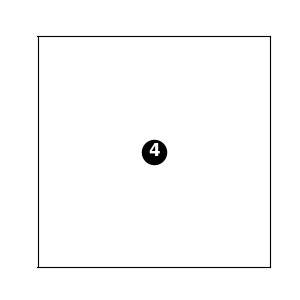}}
   \begin{aligned}
      \quad
      {}{^4_0}H
   \end{aligned}
\end{equation}

\begin{equation}
   \raisebox{-0.4\height}{\includegraphics[height=6\baselineskip]{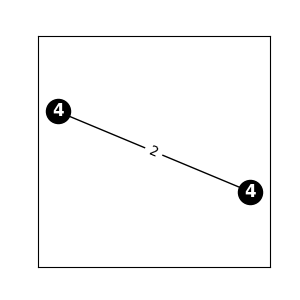}}
   \begin{aligned}
      \quad
      {}{^4_2}H^2(1,2)(1,2)
   \end{aligned}
\end{equation}

\begin{equation}
   \raisebox{-0.4\height}{\includegraphics[height=6\baselineskip]{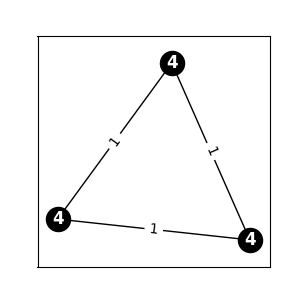}}
   \begin{aligned}
      \quad
      {}{^4_2}H^3(1,2)(2,3)(1,3)
   \end{aligned}
\end{equation}

\begin{equation}
   \raisebox{-0.4\height}{\includegraphics[height=6\baselineskip]{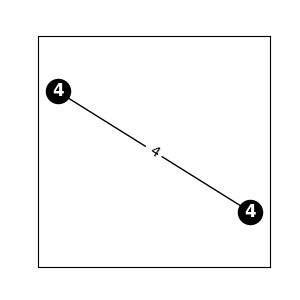}}
   \begin{aligned}
      \quad
      {}{^4_4}H^2(1, 2, 3, 4)  (1, 2, 3, 4)
   \end{aligned}
\end{equation}

\begin{equation}
   \raisebox{-0.4\height}{\includegraphics[height=6\baselineskip]{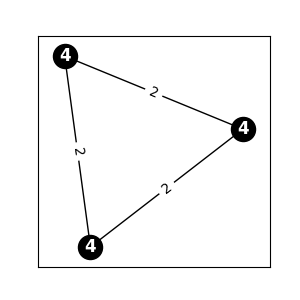}}
   \begin{aligned}
      \quad
      {}{^4_4}H^3(1, 2, 3, 4)  (3, 4, 5, 6)  (1, 2, 5, 6)
   \end{aligned}
\end{equation}

\begin{equation}
   \raisebox{-0.4\height}{\includegraphics[height=6\baselineskip]{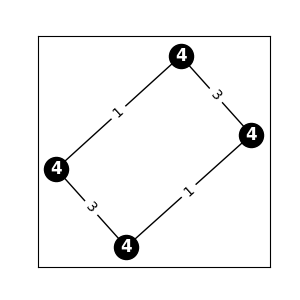}}
   \begin{aligned}
      \quad
      {}{^4_4}H^4(1, 2, 3, 4)  (2, 3, 4, 5)  (1, 6, 7, 8)  (5, 6, 7, 8)
   \end{aligned}
\end{equation}

\begin{equation}
   \raisebox{-0.4\height}{\includegraphics[height=6\baselineskip]{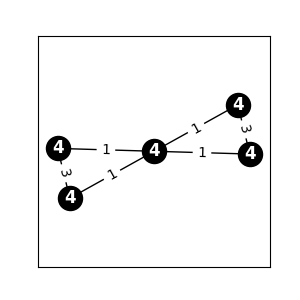}}
   \begin{aligned}
      \quad
      {}{^4_4}H^5(1, 2, 3, 4)  (2, 3, 4, 5)  (1, 5, 6, 7)  (7, 8, 9, 10)  (6, 8, 9, 10)
   \end{aligned}
\end{equation}

\begin{equation}
   \raisebox{-0.4\height}{\includegraphics[height=6\baselineskip]{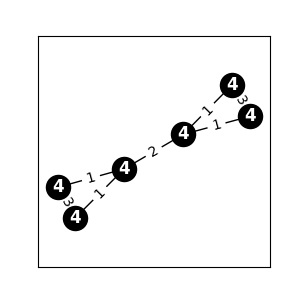}}
   \begin{aligned}
      \quad
      {}{^4_4}H^6(1, 2, 3, 4)
      (2, 3, 4, 5)
      (1, 5, 6, 7)
      (6, 7, 8, 9) \\
      (9, 10, 11, 12)
      (8, 10, 11, 12)
   \end{aligned}
\end{equation}

\begin{equation}
   \raisebox{-0.4\height}{\includegraphics[height=6\baselineskip]{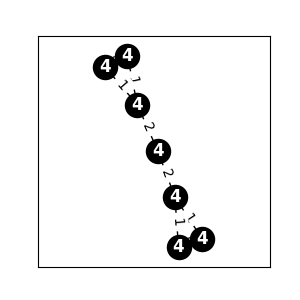}}
   \begin{aligned}
      \quad
      {}{^4_4}H^7(1, 2, 3, 4)
      (2, 3, 4, 5)
      (1, 5, 6, 7)
      (6, 7, 8, 9) \\
      (8, 9, 10, 11)
      (11, 12, 13, 14)
      (10, 12, 13, 14)
   \end{aligned}
\end{equation}

\subsubsection{Mixed Homogeneous Invariants: \textbf{3}}
\begin{equation}
   \raisebox{-0.4\height}{\includegraphics[height=6\baselineskip]{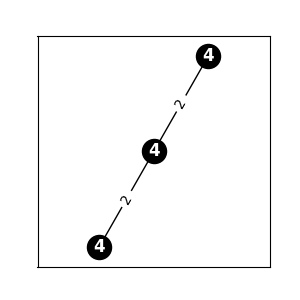}}
   \begin{aligned}
      \quad
      {}{^4_2}H^2{}{^4_4}H(1, 2)  (3, 4)  (1, 2, 3, 4)
   \end{aligned}
\end{equation}

\begin{equation}
   \raisebox{-0.4\height}{\includegraphics[height=6\baselineskip]{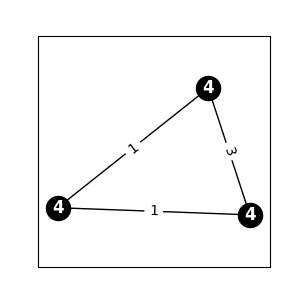}}
   \begin{aligned}
      \quad
      {}{^4_2}H{}{^4_4}H^2(1, 2)  (2, 3, 4, 5)  (1, 3, 4, 5)
   \end{aligned}
\end{equation}

\begin{equation}
   \raisebox{-0.4\height}{\includegraphics[height=6\baselineskip]{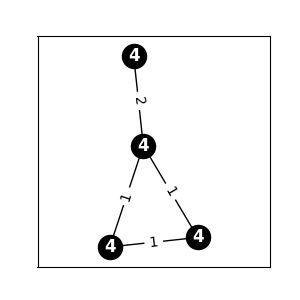}}
   \begin{aligned}
      \quad
      {}{^4_2}H^3{}{^4_4}H(1, 2)  (2, 3)  (4, 5)  (1, 3, 4, 5)
   \end{aligned}
\end{equation}

\subsubsection{Mixed Simultaneous 1 4 Invariants: \textbf{2}}
\begin{equation}
   \raisebox{-0.4\height}{\includegraphics[height=6\baselineskip]{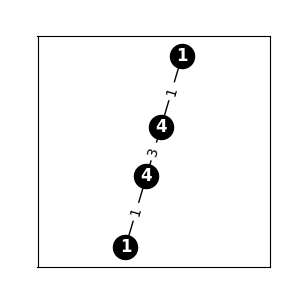}}
   \begin{aligned}
      \quad
      {}{^1_1}H^2{}{^4_4}H^2(1)  (2)  (1, 3, 4, 5)  (2, 3, 4, 5)
   \end{aligned}
\end{equation}

\begin{equation}
   \raisebox{-0.4\height}{\includegraphics[height=6\baselineskip]{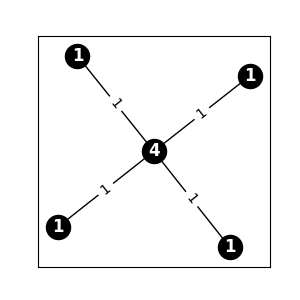}}
   \begin{aligned}
      \quad
      {}{^1_1}H^4{}{^4_4}H(1)  (2)  (3)  (4)  (1, 2, 3, 4)
   \end{aligned}
\end{equation}

\subsubsection{Mixed Simultaneous 2 4 Invariants: \textbf{3}+3}
\begin{equation}
   \raisebox{-0.4\height}{\includegraphics[height=6\baselineskip]{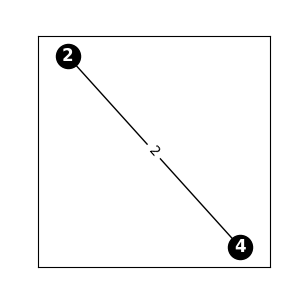}}
   \begin{aligned}
      \quad
      {}{^2_2}H{}{^4_2}H(1,2)(1,2)
   \end{aligned}
\end{equation}

\begin{equation}
   \raisebox{-0.4\height}{\includegraphics[height=6\baselineskip]{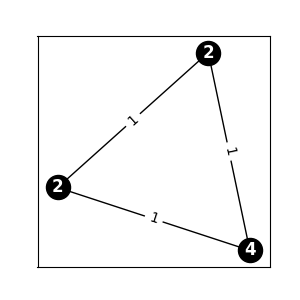}}
   \begin{aligned}
      \quad
      {}{^2_2}H^2{}{^4_2}H(1,2)(2,3)(1,3)
   \end{aligned}
\end{equation}

\begin{equation}
   \raisebox{-0.4\height}{\includegraphics[height=6\baselineskip]{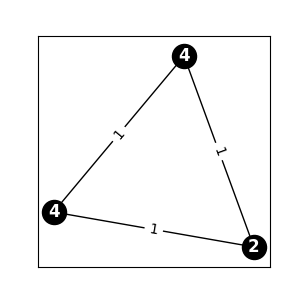}}
   \begin{aligned}
      \quad
      {}{^2_2}H{}{^4_2}H^2(1,2)(2,3)(1,3)
   \end{aligned}
\end{equation}

\begin{equation}
   \raisebox{-0.4\height}{\includegraphics[height=6\baselineskip]{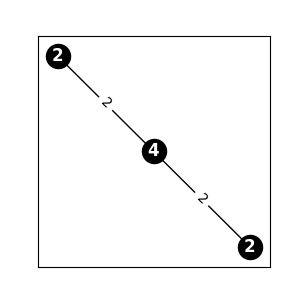}}
   \begin{aligned}
      \quad
      {}{^2_2}H^2{}{^4_4}H(1, 2)  (3, 4)  (1, 2, 3, 4)
   \end{aligned}
\end{equation}

\begin{equation}
   \raisebox{-0.4\height}{\includegraphics[height=6\baselineskip]{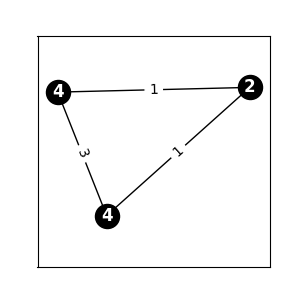}}
   \begin{aligned}
      \quad
      {}{^2_2}H{}{^4_4}H^2(1, 2)  (2, 3, 4, 5)  (1, 3, 4, 5)
   \end{aligned}
\end{equation}

\begin{equation}
   \raisebox{-0.4\height}{\includegraphics[height=6\baselineskip]{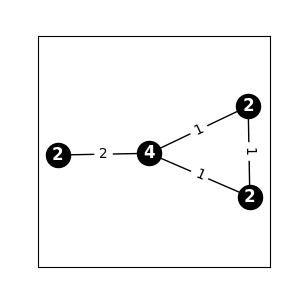}}
   \begin{aligned}
      \quad
      {}{^2_2}H^3{}{^4_4}H(1, 2)  (2, 3)  (4, 5)  (1, 3, 4, 5)
   \end{aligned}
\end{equation}

\subsubsection{Mixed Simultaneous 3 4 Invariants: 2+2+3+\textbf{3}}
\begin{equation}
   \raisebox{-0.4\height}{\includegraphics[height=6\baselineskip]{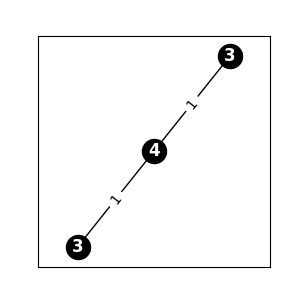}}
   \begin{aligned}
      \quad
      {}{^3_1}H^2{}{^4_2}H(1)(2)(1,2)
   \end{aligned}
\end{equation}

\begin{equation}
   \raisebox{-0.4\height}{\includegraphics[height=6\baselineskip]{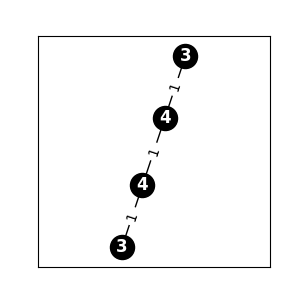}}
   \begin{aligned}
      \quad
      {}{^3_1}H^2{}{^4_2}H^2(1)(2)(1,3)(2,3)
   \end{aligned}
\end{equation}

\begin{equation}
   \raisebox{-0.4\height}{\includegraphics[height=6\baselineskip]{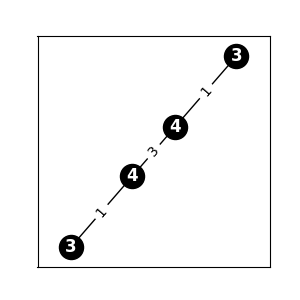}}
   \begin{aligned}
      \quad
      {}{^3_1}H^2{}{^4_4}H^2(1)  (2)  (1, 3, 4, 5)  (2, 3, 4, 5)
   \end{aligned}
\end{equation}

\begin{equation}
   \raisebox{-0.4\height}{\includegraphics[height=6\baselineskip]{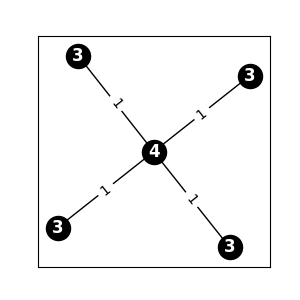}}
   \begin{aligned}
      \quad
      {}{^3_1}H^4{}{^4_4}H(1)  (2)  (3)  (4)  (1, 2, 3, 4)
   \end{aligned}
\end{equation}

\begin{equation}
   \raisebox{-0.4\height}{\includegraphics[height=6\baselineskip]{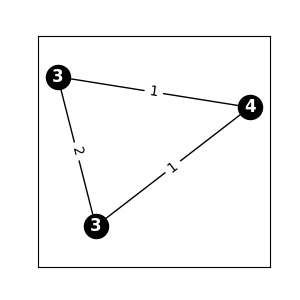}}
   \begin{aligned}
      \quad
      {}{^4_2}H{}{^3_3}H^2(1, 2)  (2, 3, 4)  (1, 3, 4)
   \end{aligned}
\end{equation}

\begin{equation}
   \raisebox{-0.4\height}{\includegraphics[height=6\baselineskip]{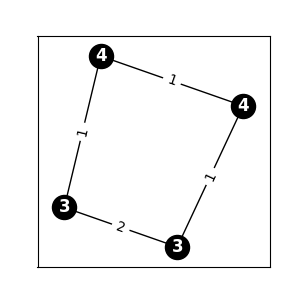}}
   \begin{aligned}
      \quad
      {}{^4_2}H^2{}{^3_3}H^2(1, 2)  (2, 3)  (1, 4, 5)  (3, 4, 5)
   \end{aligned}
\end{equation}

\begin{equation}
   \raisebox{-0.4\height}{\includegraphics[height=6\baselineskip]{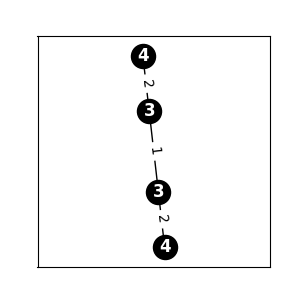}}
   \begin{aligned}
      \quad
      {}{^4_2}H^2{}{^3_3}H^2(1, 2)  (3, 4)  (1, 2, 5)  (3, 4, 5)
   \end{aligned}
\end{equation}

\begin{equation}
   \raisebox{-0.4\height}{\includegraphics[height=6\baselineskip]{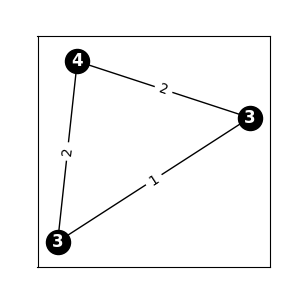}}
   \begin{aligned}
      \quad
      {}{^3_3}H^2{}{^4_4}H(1, 2, 3)  (3, 4, 5)  (1, 2, 4, 5)
   \end{aligned}
\end{equation}

\begin{equation}
   \raisebox{-0.4\height}{\includegraphics[height=6\baselineskip]{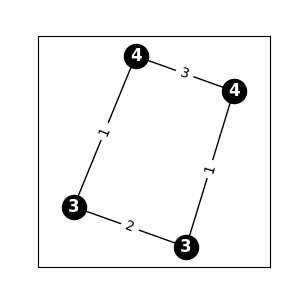}}
   \begin{aligned}
      \quad
      {}{^3_3}H^2{}{^4_4}H^2(1, 2, 3)  (2, 3, 4)  (1, 5, 6, 7)  (4, 5, 6, 7)
   \end{aligned}
\end{equation}

\begin{equation}
   \raisebox{-0.4\height}{\includegraphics[height=6\baselineskip]{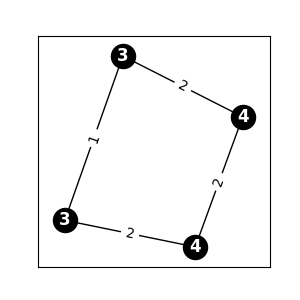}}
   \begin{aligned}
      \quad
      {}{^3_3}H^2{}{^4_4}H^2(1, 2, 3)  (3, 4, 5)  (1, 2, 6, 7)  (4, 5, 6, 7)
   \end{aligned}
\end{equation}

\subsection{Order 5}

\subsubsection{Irreducible Decomposition: 3}
\begin{equation}
   \begin{aligned}
      {}{^5} M_{ijklm}
       & =
      {}{^5_5} H_{ijklm}
      +
      {}{5_3} H_{(ijk}\delta_{lm)}
      +
      {}{^5_1} H_{(i}\delta_{jk}\delta_{lm)}
   \end{aligned}
\end{equation}

\subsubsection{Pure Invariants: 1+4+\textbf{8}}
\begin{equation}
   \raisebox{-0.4\height}{\includegraphics[height=6\baselineskip]{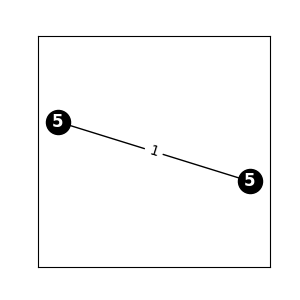}}
   \begin{aligned}
      \quad
      {}{^5_1}H^2(1)(1)
   \end{aligned}
\end{equation}

\begin{equation}
   \raisebox{-0.4\height}{\includegraphics[height=6\baselineskip]{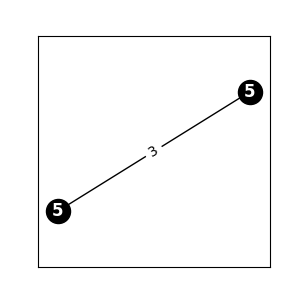}}
   \begin{aligned}
      \quad
      {}{^5_3}H^2(1,2,3)(1,2,3)
   \end{aligned}
\end{equation}

\begin{equation}
   \raisebox{-0.4\height}{\includegraphics[height=6\baselineskip]{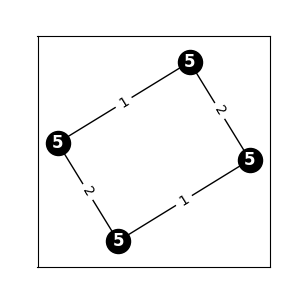}}
   \begin{aligned}
      \quad
      {}{^5_3}H^4(1,2,3)  (1,2,4)  (3,5,6)  (4,5,6)
   \end{aligned}
\end{equation}

\begin{equation}
   \raisebox{-0.4\height}{\includegraphics[height=6\baselineskip]{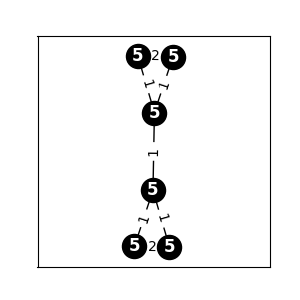}}
   \begin{aligned}
      \quad
      {}{^5_3}H^6(1, 2, 3)  (2, 3, 4)  (1, 4, 5)  (5, 6, 7)  (7, 8, 9)  (6, 8, 9)
   \end{aligned}
\end{equation}

\begin{equation}
   \raisebox{-0.4\height}{\includegraphics[height=6\baselineskip]{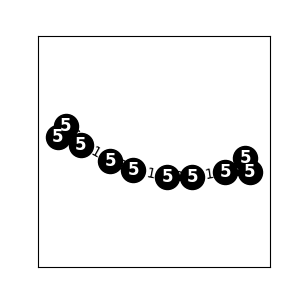}}
   \begin{aligned}
      \quad
      {}{^5_3}H^{10}(1, 2, 3)  (2, 3, 4)  (1, 4, 5)  (5, 6, 7)  (6, 7, 8)  (8, 9, 10) \\ (9, 10, 11)  (11, 12, 13)  (13, 14, 15)  (12, 14, 15)
   \end{aligned}
\end{equation}

\begin{equation}
   \raisebox{-0.4\height}{\includegraphics[height=6\baselineskip]{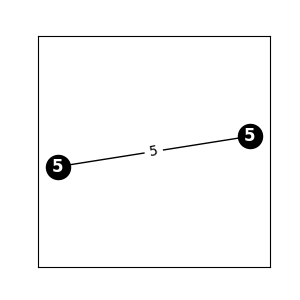}}
   \begin{aligned}
      \quad
      {}{^5_5}H^2(1, 2, 3, 4, 5)  (1, 2, 3, 4, 5)
   \end{aligned}
\end{equation}

\begin{equation}
   \raisebox{-0.4\height}{\includegraphics[height=6\baselineskip]{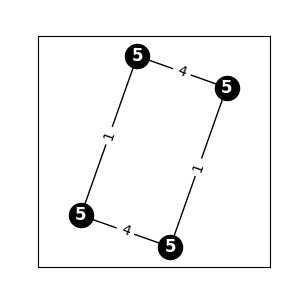}}
   \begin{aligned}
      \quad
      {}{^5_5}H^4(1, 2, 3, 4, 5)  (2, 3, 4, 5, 6)  (1, 7, 8, 9, 10)  (6, 7, 8, 9, 10)
   \end{aligned}
\end{equation}

\begin{equation}
   \raisebox{-0.4\height}{\includegraphics[height=6\baselineskip]{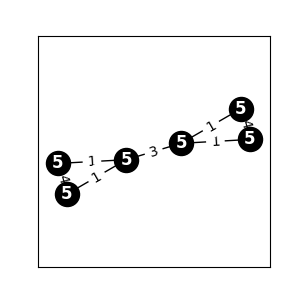}}
   \begin{aligned}
      \quad
      {}{^5_5}H^6(1, 2, 3, 4, 5)
      (2, 3, 4, 5, 6)
      (1, 6, 7, 8, 9)
      (7, 8, 9, 10, 11) \\
      (11, 12, 13, 14, 15)
      (10, 12, 13, 14, 15)
   \end{aligned}
\end{equation}

\begin{equation}
   \raisebox{-0.4\height}{\includegraphics[height=6\baselineskip]{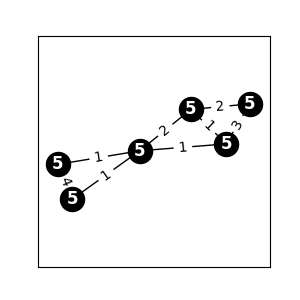}}
   \begin{aligned}
      \quad
      {}{^5_5}H^6(1, 2, 3, 4, 5)
      (3, 4, 5, 6, 7)
      (1, 2, 7, 8, 9)
      (6, 8, 9, 10, 11) \\
      (11, 12, 13, 14, 15)
      (10, 12, 13, 14, 15)
   \end{aligned}
\end{equation}

\begin{equation}
   \raisebox{-0.4\height}{\includegraphics[height=6\baselineskip]{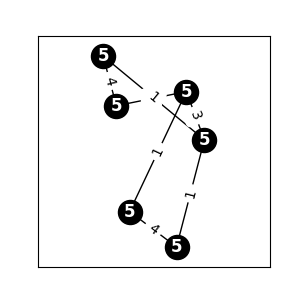}}
   \begin{aligned}
      \quad
      {}{^5_5}H^6(1, 2, 3, 4, 5)
      (2, 3, 4, 5, 6)
      (1, 7, 8, 9, 10)
      (6, 8, 9, 10, 11) \\
      (7, 12, 13, 14, 15)
      (11, 12, 13, 14, 15)
   \end{aligned}
\end{equation}

\begin{equation}
   \raisebox{-0.4\height}{\includegraphics[height=6\baselineskip]{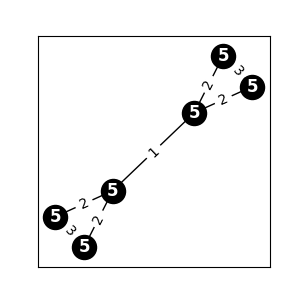}}
   \begin{aligned}
      \quad
      {}{^5_5}H^6(1, 2, 3, 4, 5)
      (3, 4, 5, 6, 7)
      (1, 2, 6, 7, 8)
      (8, 9, 10, 11, 12) \\
      (11, 12, 13, 14, 15)
      (9, 10, 13, 14, 15)
   \end{aligned}
\end{equation}

\begin{equation}
   \raisebox{-0.4\height}{\includegraphics[height=6\baselineskip]{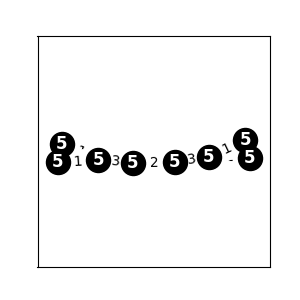}}
   \begin{aligned}
      \quad
      {}{^5_5}H^8(1, 2, 3, 4, 5)
      (2, 3, 4, 5, 6)
      (1, 6, 7, 8, 9)
      (7, 8, 9, 10, 11)    \\
      (10, 11, 12, 13, 14)
      (12, 13, 14, 15, 16) \\
      (16, 17, 18, 19, 20)
      (15, 17, 18, 19, 20)
   \end{aligned}
\end{equation}

\begin{equation}
   \raisebox{-0.4\height}{\includegraphics[height=6\baselineskip]{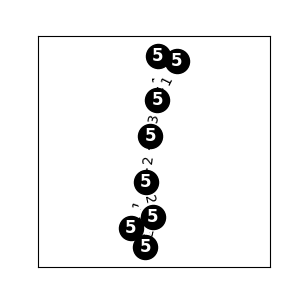}}
   \begin{aligned}
      \quad
      {}{^5_5}H^8(0, 1, 2, 3, 4)  (2, 3, 4, 5, 6)  (0, 1, 6, 7, 8)  (5, 7, 8, 9, 10) \\
      (9, 10, 11, 12, 13)  (11, 12, 13, 14, 15)                                            \\ (15, 16, 17, 18, 19)  (14, 16, 17, 18, 19)
   \end{aligned}
\end{equation}

\subsubsection{Mixed Homogeneous Invariants: 2+\textbf{2}+\textbf{3}}
\begin{equation}
   \raisebox{-0.4\height}{\includegraphics[height=6\baselineskip]{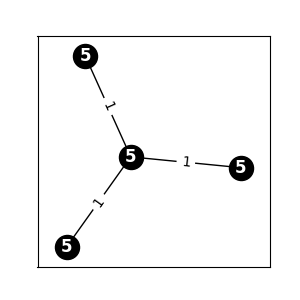}}
   \begin{aligned}
      \quad
      {}{^5_1}H^3{}{^5_3}H(1)(2)(3)(1, 2, 3)
   \end{aligned}
\end{equation}

\begin{equation}
   \raisebox{-0.4\height}{\includegraphics[height=6\baselineskip]{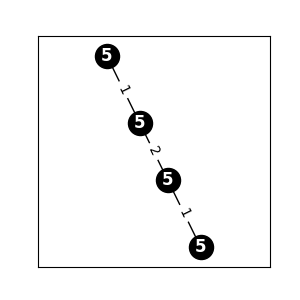}}
   \begin{aligned}
      \quad
      {}{^5_1}H^2{}{^5_3}H^2(1)  (2)  (1,3,4)  (2,3,4)
   \end{aligned}
\end{equation}

\begin{equation}
   \raisebox{-0.4\height}{\includegraphics[height=6\baselineskip]{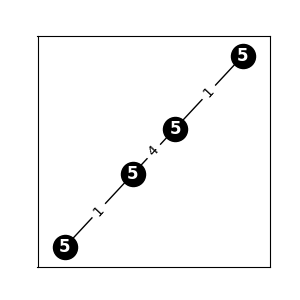}}
   \begin{aligned}
      \quad
      {}{^5_1}H^5{}{^5_5}H^2(1)(2)(3)(4)(5)(1, 2, 3,4,5)
   \end{aligned}
\end{equation}

\begin{equation}
   \raisebox{-0.4\height}{\includegraphics[height=6\baselineskip]{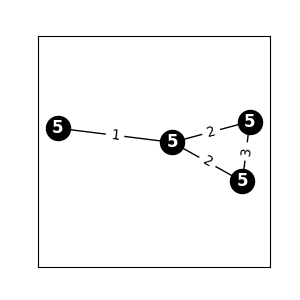}}
   \begin{aligned}
      \quad
      {}{^5_1}H^2{}{^5_5}H^2(1)(2)(1, 3, 4, 5, 6)(2, 3, 4, 5, 6)
   \end{aligned}
\end{equation}

\begin{equation}
   \raisebox{-0.4\height}{\includegraphics[height=6\baselineskip]{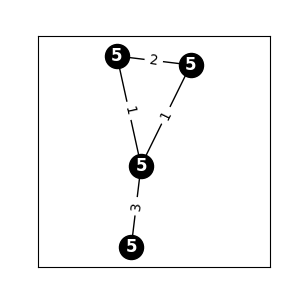}}
   \begin{aligned}
      \quad
      {}{^5_3}H^3{}{^5_5}H(1, 2, 3)  (2, 3, 4)  (5, 6, 7)  (1, 4, 5, 6, 7)
   \end{aligned}
\end{equation}

\begin{equation}
   \raisebox{-0.4\height}{\includegraphics[height=6\baselineskip]{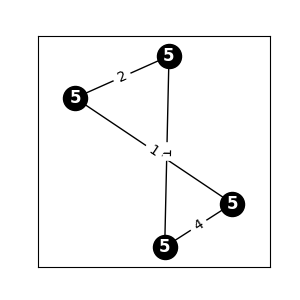}}
   \begin{aligned}
      \quad
      {}{^5_3}H^2{}{^5_5}H^2(1, 2, 3)  (2, 3, 4)  (1, 5, 6, 7, 8)  (4, 5, 6, 7, 8)
   \end{aligned}
\end{equation}

\begin{equation}
   \raisebox{-0.4\height}{\includegraphics[height=6\baselineskip]{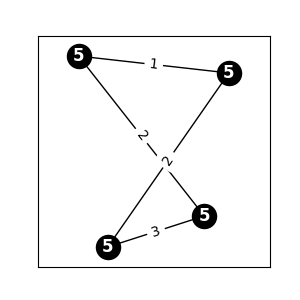}}
   \begin{aligned}
      \quad
      {}{^5_3}H^2{}{^5_5}H^2(1, 2, 3)  (3, 4, 5)  (1, 2, 6, 7, 8)  (4, 5, 6, 7, 8)
   \end{aligned}
\end{equation}

\subsubsection{Mixed Simultaneous 1 5 Invariants: \textbf{1}+2+2}
\begin{equation}
   \raisebox{-0.4\height}{\includegraphics[height=6\baselineskip]{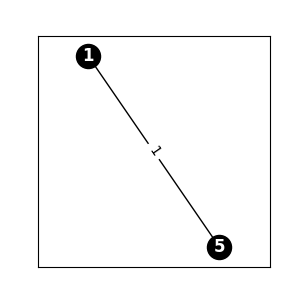}}
   \begin{aligned}
      \quad
      {}{^1_1}H{}{^5_1}H(1)(1)
   \end{aligned}
\end{equation}

\begin{equation}
   \raisebox{-0.4\height}{\includegraphics[height=6\baselineskip]{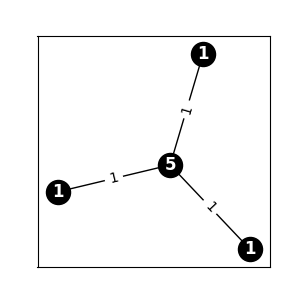}}
   \begin{aligned}
      \quad
      {}{^1_1}H^3{}{^5_3}H(1)(2)(3)(1, 2, 3)
   \end{aligned}
\end{equation}

\begin{equation}
   \raisebox{-0.4\height}{\includegraphics[height=6\baselineskip]{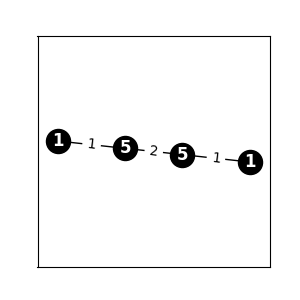}}
   \begin{aligned}
      \quad
      {}{^1_1}H^2{}{^5_3}H^2(1)  (2)  (1,3,4)  (2,3,4)
   \end{aligned}
\end{equation}

\begin{equation}
   \raisebox{-0.4\height}{\includegraphics[height=6\baselineskip]{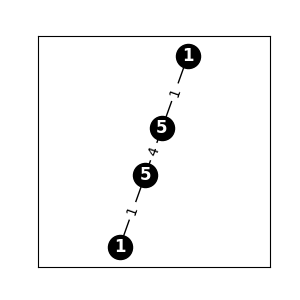}}
   \begin{aligned}
      \quad
      {}{^1_1}H^2{}{^5_5}H^2(1)(2)(1, 3, 4, 5, 6)(2, 3, 4, 5, 6)
   \end{aligned}
\end{equation}

\begin{equation}
   \raisebox{-0.4\height}{\includegraphics[height=6\baselineskip]{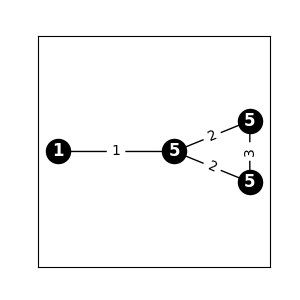}}
   \begin{aligned}
      \quad
      {}{^1_1}H{}{^5_5}H^3(1)(1, 2, 3, 4, 5)(4, 5, 6, 7, 8)(2, 3, 6, 7, 8)
   \end{aligned}
\end{equation}

\subsubsection{Mixed Simultaneous 2 5 Invariants: 2+3+\textbf{3}}
\begin{equation}
   \raisebox{-0.4\height}{\includegraphics[height=6\baselineskip]{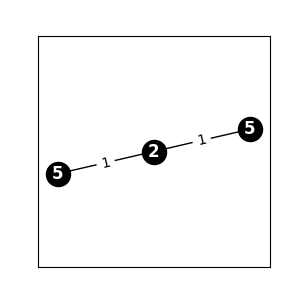}}
   \begin{aligned}
      \quad
      {}{^5_1}H^2{}{^2_2}H(1)(2)(1,2)
   \end{aligned}
\end{equation}

\begin{equation}
   \raisebox{-0.4\height}{\includegraphics[height=6\baselineskip]{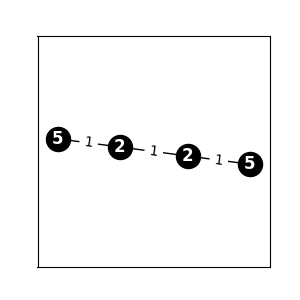}}
   \begin{aligned}
      \quad
      {}{^5_1}H^2{}{^2_2}H^2(1)(2)(1,3)(2,3)
   \end{aligned}
\end{equation}

\begin{equation}
   \raisebox{-0.4\height}{\includegraphics[height=6\baselineskip]{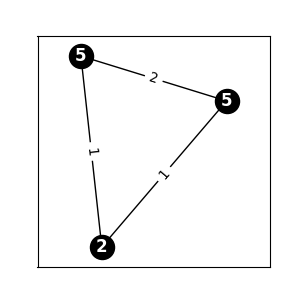}}
   \begin{aligned}
      \quad
      {}{^2_2}H{}{^5_3}H^2(1, 2)  (2, 3, 4)  (1, 3, 4)
   \end{aligned}
\end{equation}

\begin{equation}
   \raisebox{-0.4\height}{\includegraphics[height=6\baselineskip]{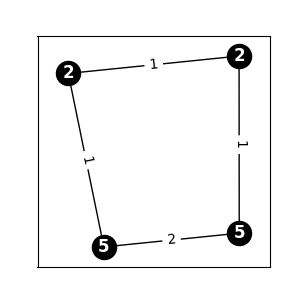}}
   \begin{aligned}
      \quad
      {}{^2_2}H^2{}{^5_3}H^2(1, 2)  (2, 3)  (1, 4, 5)  (3, 4, 5)
   \end{aligned}
\end{equation}

\begin{equation}
   \raisebox{-0.4\height}{\includegraphics[height=6\baselineskip]{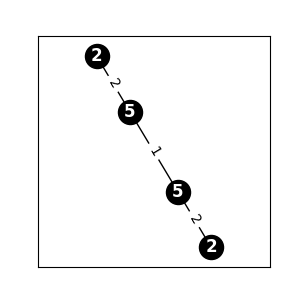}}
   \begin{aligned}
      \quad
      {}{^2_2}H^2{}{^5_3}H^2(1, 2)  (3, 4)  (1, 2, 5)  (3, 4, 5)
   \end{aligned}
\end{equation}

\begin{equation}
   \raisebox{-0.4\height}{\includegraphics[height=6\baselineskip]{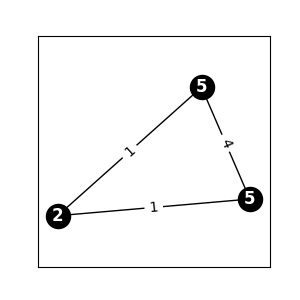}}
   \begin{aligned}
      \quad
      {}{^2_2}H{}{^5_5}H^2(1, 2)  (2, 3, 4, 5, 6)  (1, 3, 4, 5, 6)
   \end{aligned}
\end{equation}

\begin{equation}
   \raisebox{-0.4\height}{\includegraphics[height=6\baselineskip]{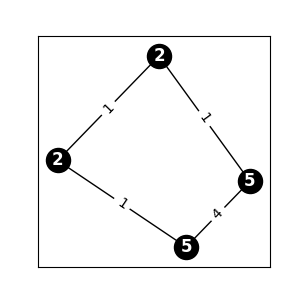}}
   \begin{aligned}
      \quad
      {}{^2_2}H^2{}{^5_5}H^2(1, 2)  (2, 3)  (1, 4, 5, 6, 7)  (3, 4, 5, 6, 7)
   \end{aligned}
\end{equation}

\begin{equation}
   \raisebox{-0.4\height}{\includegraphics[height=6\baselineskip]{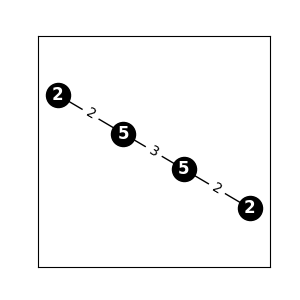}}
   \begin{aligned}
      \quad
      {}{^2_2}H^2{}{^5_5}H^2(1, 2)  (3, 4)  (1, 2, 5, 6, 7)  (3, 4, 5, 6, 7)
   \end{aligned}
\end{equation}

\subsubsection{Mixed Simultaneous 3 5 Invariants: 1+2+2+2+\textbf{3}+3}
\begin{equation}
   \raisebox{-0.4\height}{\includegraphics[height=6\baselineskip]{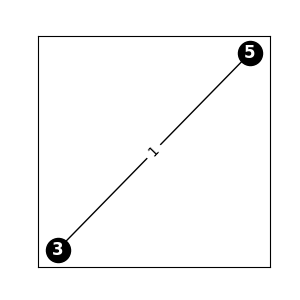}}
   \begin{aligned}
      \quad
      {}{^3_1}H{}{^5_1}H(1)(1)
   \end{aligned}
\end{equation}

\begin{equation}
   \raisebox{-0.4\height}{\includegraphics[height=6\baselineskip]{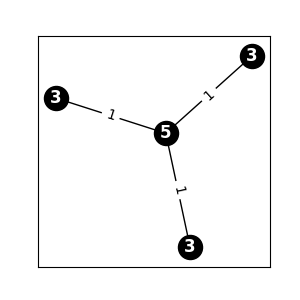}}
   \begin{aligned}
      \quad
      {}{^3_1}H^3{}{^5_3}H(1)(2)(3)(1, 2, 3)
   \end{aligned}
\end{equation}

\begin{equation}
   \raisebox{-0.4\height}{\includegraphics[height=6\baselineskip]{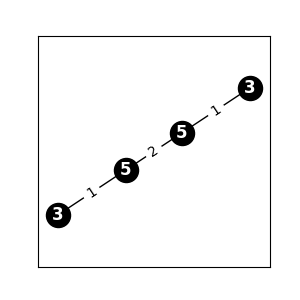}}
   \begin{aligned}
      \quad
      {}{^3_1}H^2{}{^5_3}H^2(1)  (2)  (1,3,4)  (2,3,4)
   \end{aligned}
\end{equation}

\begin{equation}
   \raisebox{-0.4\height}{\includegraphics[height=6\baselineskip]{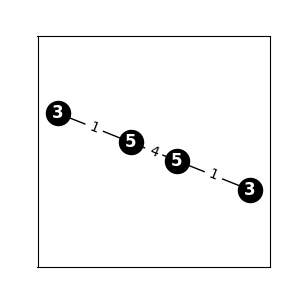}}
   \begin{aligned}
      \quad
      {}{^3_1}H^2{}{^5_5}H^2(1)(2)(1, 3, 4, 5, 6)(2, 3, 4, 5, 6)
   \end{aligned}
\end{equation}

\begin{equation}
   \raisebox{-0.4\height}{\includegraphics[height=6\baselineskip]{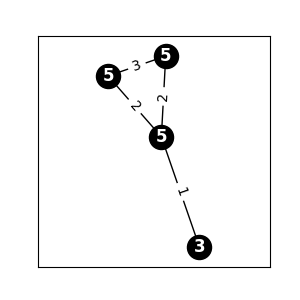}}
   \begin{aligned}
      \quad
      {}{^3_1}H{}{^5_5}H^3(1)(1, 2, 3, 4, 5)(4, 5, 6, 7, 8)(2, 3, 6, 7, 8)
   \end{aligned}
\end{equation}

\begin{equation}
   \raisebox{-0.4\height}{\includegraphics[height=6\baselineskip]{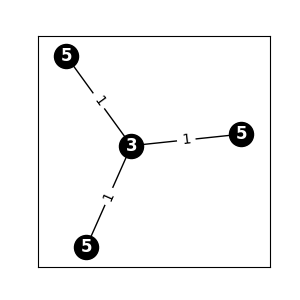}}
   \begin{aligned}
      \quad
      {}{^5_1}H^3{}{^3_3}H(1)(2)(3)(1, 2, 3)
   \end{aligned}
\end{equation}

\begin{equation}
   \raisebox{-0.4\height}{\includegraphics[height=6\baselineskip]{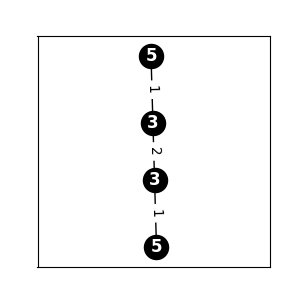}}
   \begin{aligned}
      \quad
      {}{^5_1}H^2{}{^3_3}H^2(1)  (2)  (1,3,4)  (2,3,4)
   \end{aligned}
\end{equation}

\begin{equation}
   \raisebox{-0.4\height}{\includegraphics[height=6\baselineskip]{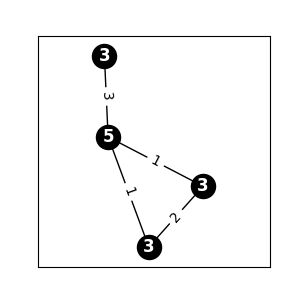}}
   \begin{aligned}
      \quad
      {}{^3_3}H^3{}{^5_3}H
   \end{aligned}
\end{equation}

\begin{equation}
   \raisebox{-0.4\height}{\includegraphics[height=6\baselineskip]{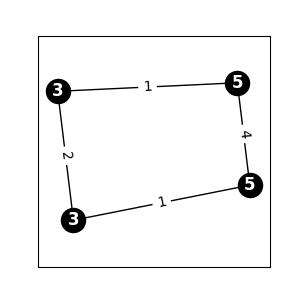}}
   \begin{aligned}
      \quad
      {}{^3_3}H^2{}{^5_3}H^2
   \end{aligned}
\end{equation}

\begin{equation}
   \raisebox{-0.4\height}{\includegraphics[height=6\baselineskip]{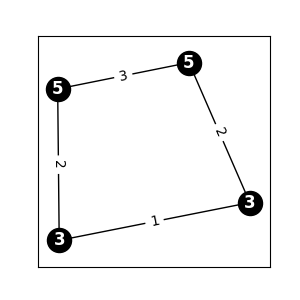}}
   \begin{aligned}
      \quad
      {}{^3_3}H{}{^5_3}H^3
   \end{aligned}
\end{equation}

\begin{equation}
   \raisebox{-0.4\height}{\includegraphics[height=6\baselineskip]{pics/mixed_3_3_5_5_0.png}}
   \begin{aligned}
      \quad
      {}{^3_3}H^3{}{^5_5}H(1, 2, 3)  (2, 3, 4)  (5, 6, 7)  (1, 4, 5, 6, 7)
   \end{aligned}
\end{equation}

\begin{equation}
   \raisebox{-0.4\height}{\includegraphics[height=6\baselineskip]{pics/mixed_3_3_5_5_1.png}}
   \begin{aligned}
      \quad
      {}{^3_3}H^2{}{^5_5}H^2(1, 2, 3)  (2, 3, 4)  (1, 5, 6, 7, 8)  (4, 5, 6, 7, 8)
   \end{aligned}
\end{equation}

\begin{equation}
   \raisebox{-0.4\height}{\includegraphics[height=6\baselineskip]{pics/mixed_3_3_5_5_2.png}}
   \begin{aligned}
      \quad
      {}{^3_3}H^2{}{^5_5}H^2(1, 2, 3)  (3, 4, 5)  (1, 2, 6, 7, 8)  (4, 5, 6, 7, 8)
   \end{aligned}
\end{equation}

\subsubsection{Mixed Simultaneous 4 5 Invariants: 2+3+3+2+3+\textbf{3}}
\begin{equation}
   \raisebox{-0.4\height}{\includegraphics[height=6\baselineskip]{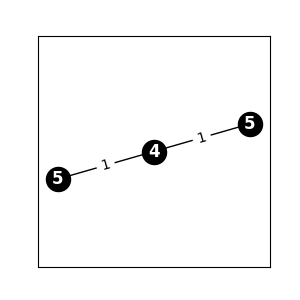}}
   \begin{aligned}
      \quad
      {}{^5_1}H^2{}{^4_2}H(1)(2)(1,2)
   \end{aligned}
\end{equation}

\begin{equation}
   \raisebox{-0.4\height}{\includegraphics[height=6\baselineskip]{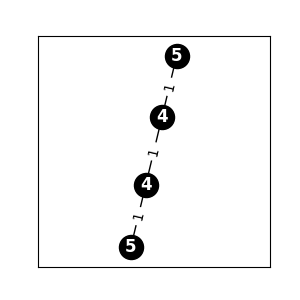}}
   \begin{aligned}
      \quad
      {}{^5_1}H^2{}{^4_2}H^2(1)(2)(1,3)(2,3)
   \end{aligned}
\end{equation}

\begin{equation}
   \raisebox{-0.4\height}{\includegraphics[height=6\baselineskip]{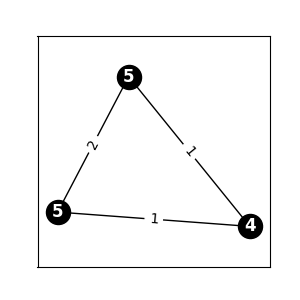}}
   \begin{aligned}
      \quad
      {}{^4_2}H{}{^5_3}H^2(1, 2)  (2, 3, 4)  (1, 3, 4)
   \end{aligned}
\end{equation}

\begin{equation}
   \raisebox{-0.4\height}{\includegraphics[height=6\baselineskip]{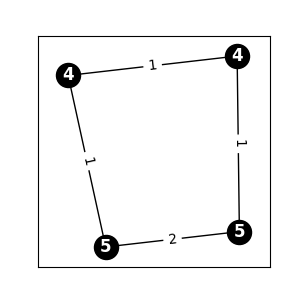}}
   \begin{aligned}
      \quad
      {}{^4_2}H^2{}{^5_3}H^2(1, 2)  (2, 3)  (1, 4, 5)  (3, 4, 5)
   \end{aligned}
\end{equation}

\begin{equation}
   \raisebox{-0.4\height}{\includegraphics[height=6\baselineskip]{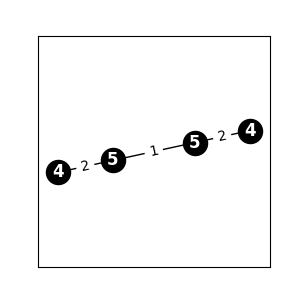}}
   \begin{aligned}
      \quad
      {}{^4_2}H^2{}{^5_3}H^2(1, 2)  (3, 4)  (1, 2, 5)  (3, 4, 5)
   \end{aligned}
\end{equation}

\begin{equation}
   \raisebox{-0.4\height}{\includegraphics[height=6\baselineskip]{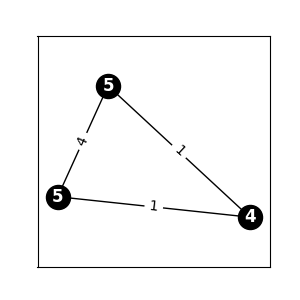}}
   \begin{aligned}
      \quad
      {}{^4_2}H{}{^5_5}H^2(1, 2)  (2, 3, 4, 5, 6)  (1, 3, 4, 5, 6)
   \end{aligned}
\end{equation}

\begin{equation}
   \raisebox{-0.4\height}{\includegraphics[height=6\baselineskip]{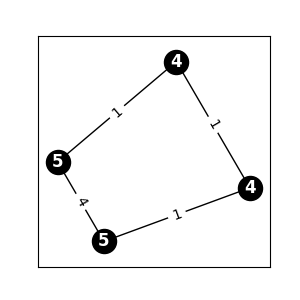}}
   \begin{aligned}
      \quad
      {}{^4_2}H^2{}{^5_5}H^2(1, 2)  (2, 3)  (1, 4, 5, 6, 7)  (3, 4, 5, 6, 7)
   \end{aligned}
\end{equation}

\begin{equation}
   \raisebox{-0.4\height}{\includegraphics[height=6\baselineskip]{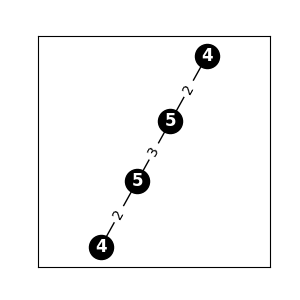}}
   \begin{aligned}
      \quad
      {}{^4_2}H^2{}{^5_5}H^2(1, 2)  (3, 4)  (1, 2, 5, 6, 7)  (3, 4, 5, 6, 7)
   \end{aligned}
\end{equation}

\begin{equation}
   \raisebox{-0.4\height}{\includegraphics[height=6\baselineskip]{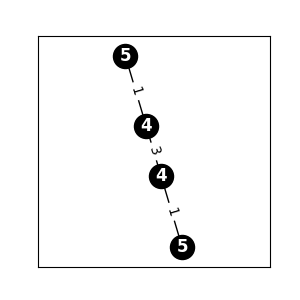}}
   \begin{aligned}
      \quad
      {}{^5_1}H^2{}{^4_4}H^2(1)  (2)  (1, 3, 4, 5)  (2, 3, 4, 5)
   \end{aligned}
\end{equation}

\begin{equation}
   \raisebox{-0.4\height}{\includegraphics[height=6\baselineskip]{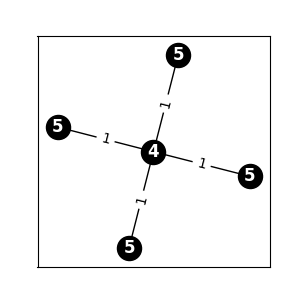}}
   \begin{aligned}
      \quad
      {}{^5_1}H^4{}{^4_4}H(1)  (2)  (3)  (4)  (1, 2, 3, 4)
   \end{aligned}
\end{equation}

\begin{equation}
   \raisebox{-0.4\height}{\includegraphics[height=6\baselineskip]{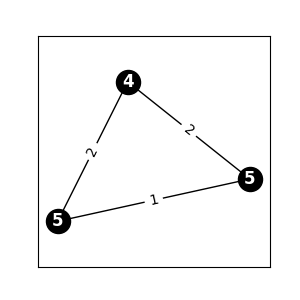}}
   \begin{aligned}
      \quad
      {}{^5_3}H^2{}{^4_4}H(1, 2, 3)  (3, 4, 5)  (1, 2, 4, 5)
   \end{aligned}
\end{equation}

\begin{equation}
   \raisebox{-0.4\height}{\includegraphics[height=6\baselineskip]{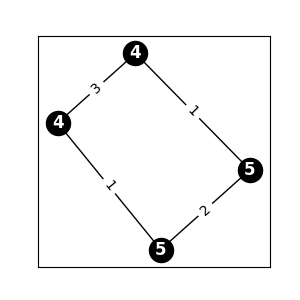}}
   \begin{aligned}
      \quad
      {}{^5_3}H^2{}{^4_4}H^2(1, 2, 3)  (2, 3, 4)  (1, 5, 6, 7)  (4, 5, 6, 7)
   \end{aligned}
\end{equation}

\begin{equation}
   \raisebox{-0.4\height}{\includegraphics[height=6\baselineskip]{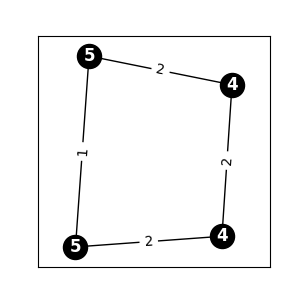}}
   \begin{aligned}
      \quad
      {}{^5_3}H^2{}{^4_4}H^2(1, 2, 3)  (3, 4, 5)  (1, 2, 6, 7)  (4, 5, 6, 7)
   \end{aligned}
\end{equation}

\begin{equation}
   \raisebox{-0.4\height}{\includegraphics[height=6\baselineskip]{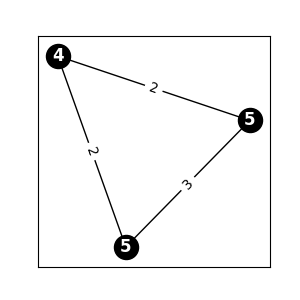}}
   \begin{aligned}
      \quad
      {}{^4_4}H{}{^5_5}H^2(1, 2, 3, 4)  (3, 4, 5, 6, 7)  (1, 2, 5, 6, 7)
   \end{aligned}
\end{equation}

\begin{equation}
   \raisebox{-0.4\height}{\includegraphics[height=6\baselineskip]{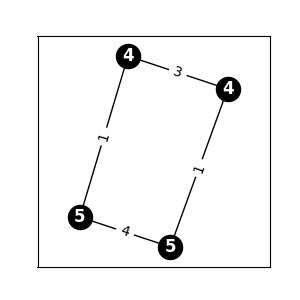}}
   \begin{aligned}
      \quad
      {}{^4_4}H^2{}{^5_5}H^2(1, 2, 3, 4)  (2, 3, 4, 5)  (1, 6, 7, 8, 9)  (5, 6, 7, 8, 9)
   \end{aligned}
\end{equation}

\begin{equation}
   \raisebox{-0.4\height}{\includegraphics[height=6\baselineskip]{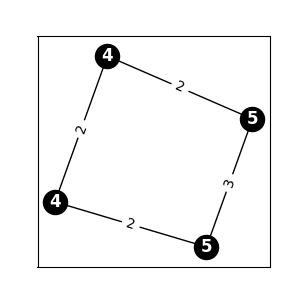}}
   \begin{aligned}
      \quad
      {}{^4_4}H^2{}{^5_5}H^2(1, 2, 3, 4)  (3, 4, 5, 6)  (1, 2, 7, 8, 9)  (5, 6, 7, 8, 9)
   \end{aligned}
\end{equation}

\subsection{Order 6}

\subsubsection{Irreducible Decomposition: 3}
\begin{equation}
      {}{^6} M_{ijklmn}
      =
      {}{^6_6} H_{ijklmn}
      +
      {}{^6_4} H_{(ijkl}\delta_{mn)}
      +
      {}{^6_2} H_{(ij}\delta_{kl}\delta_{mn)}
      +
      {}{^6_0} H\delta_{(ij}\delta_{kl}\delta_{mn)}
\end{equation}

\subsubsection{Pure Invariants: 1+2+6+\textbf{10}}
\begin{equation}
   \raisebox{-0.4\height}{\includegraphics[height=6\baselineskip]{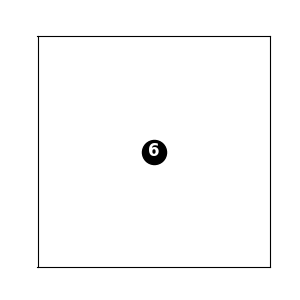}}
   \begin{aligned}
      \quad
      {}{^6_0}H
   \end{aligned}
\end{equation}

\begin{equation}
   \raisebox{-0.4\height}{\includegraphics[height=6\baselineskip]{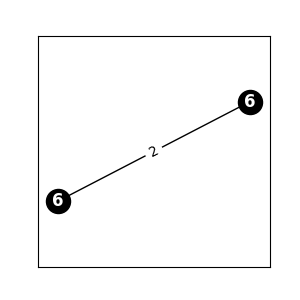}}
   \begin{aligned}
      \quad
      {}{^6_2}H^2(1,2)(1,2)
   \end{aligned}
\end{equation}

\begin{equation}
   \raisebox{-0.4\height}{\includegraphics[height=6\baselineskip]{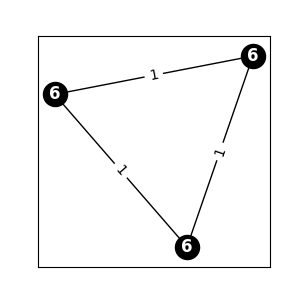}}
   \begin{aligned}
      \quad
      {}{^6_2}H^3(1,2)(2,3)(1,3)
   \end{aligned}
\end{equation}

\begin{equation}
   \raisebox{-0.4\height}{\includegraphics[height=6\baselineskip]{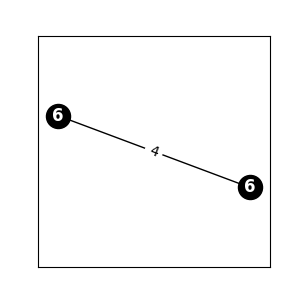}}
   \begin{aligned}
      \quad
      {}{^6_4}H^2(1, 2, 3, 4)  (1, 2, 3, 4)
   \end{aligned}
\end{equation}

\begin{equation}
   \raisebox{-0.4\height}{\includegraphics[height=6\baselineskip]{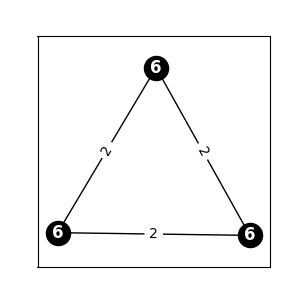}}
   \begin{aligned}
      \quad
      {}{^6_4}H^3(1, 2, 3, 4)  (3, 4, 5, 6)  (1, 2, 5, 6)
   \end{aligned}
\end{equation}

\begin{equation}
   \raisebox{-0.4\height}{\includegraphics[height=6\baselineskip]{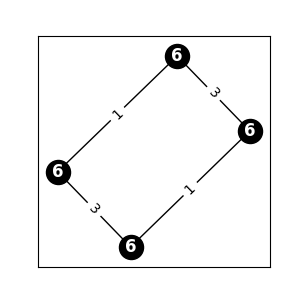}}
   \begin{aligned}
      \quad
      {}{^6_4}H^4(1, 2, 3, 4)  (2, 3, 4, 5)  (1, 6, 7, 8)  (5, 6, 7, 8)
   \end{aligned}
\end{equation}

\begin{equation}
   \raisebox{-0.4\height}{\includegraphics[height=6\baselineskip]{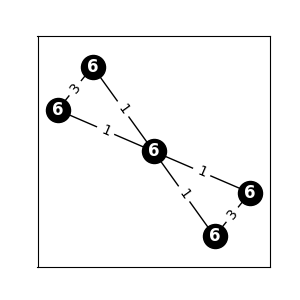}}
   \begin{aligned}
      \quad
      {}{^6_4}H^5(1, 2, 3, 4)  (2, 3, 4, 5)  (1, 5, 6, 7) \\
      (7, 8, 9, 10)  (6, 8, 9, 10)
   \end{aligned}
\end{equation}

\begin{equation}
   \raisebox{-0.4\height}{\includegraphics[height=6\baselineskip]{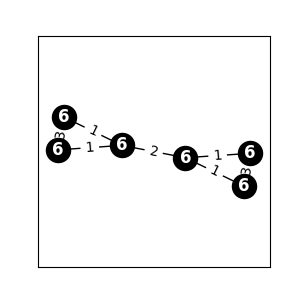}}
   \begin{aligned}
      \quad
      {}{^6_4}H^6(1, 2, 3, 4)
      (2, 3, 4, 5)
      (1, 5, 6, 7)
      (6, 7, 8, 9) \\
      (9, 10, 11, 12)
      (8, 10, 11, 12)
   \end{aligned}
\end{equation}

\begin{equation}
   \raisebox{-0.4\height}{\includegraphics[height=6\baselineskip]{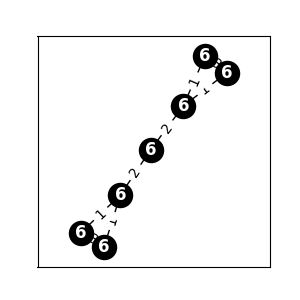}}
   \begin{aligned}
      \quad
      {}{^6_4}H^7(1, 2, 3, 4)
      (2, 3, 4, 5)
      (1, 5, 6, 7)
      (6, 7, 8, 9) \\
      (8, 9, 10, 11)
      (11, 12, 13, 14)
      (10, 12, 13, 14)
   \end{aligned}
\end{equation}

\begin{equation}
   \raisebox{-0.4\height}{\includegraphics[height=6\baselineskip]{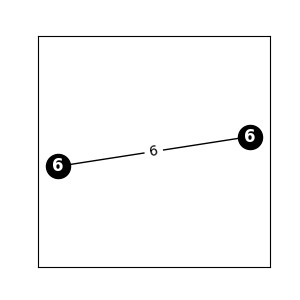}}
   \begin{aligned}
      \quad
      {}{^6_6}H^2(1, 2, 3, 4, 5, 6)  (1, 2, 3, 4, 5, 6)
   \end{aligned}
\end{equation}

\begin{equation}
   \raisebox{-0.4\height}{\includegraphics[height=6\baselineskip]{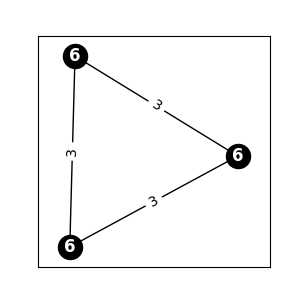}}
   \begin{aligned}
      \quad
      {}{^6_6}H^3(1, 2, 3, 4, 5, 6), (4, 5, 6, 7, 8, 9), (1, 2, 3, 7, 8, 9)
   \end{aligned}
\end{equation}

\begin{equation}
   \raisebox{-0.4\height}{\includegraphics[height=6\baselineskip]{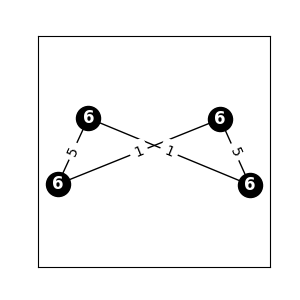}}
   \begin{aligned}
      \quad
      {}{^6_6}H^4(1, 2, 3, 4, 5, 6)
      (2, 3, 4, 5, 6, 7)
      (1, 8, 9, 10, 11, 12) \\
      (7, 8, 9, 10, 11, 12)
   \end{aligned}
\end{equation}

\begin{equation}
   \raisebox{-0.4\height}{\includegraphics[height=6\baselineskip]{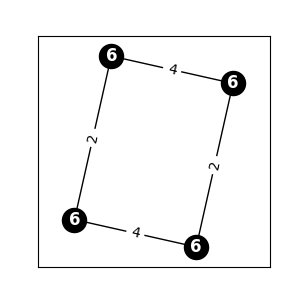}}
   \begin{aligned}
      \quad
      {}{^6_6}H^4(1, 2, 3, 4, 5, 6)
      (3, 4, 5, 6, 7, 8)
      (1, 2, 9, 10, 11, 12) \\
      (7, 8, 9, 10, 11, 12)
   \end{aligned}
\end{equation}

\begin{equation}
   \raisebox{-0.4\height}{\includegraphics[height=6\baselineskip]{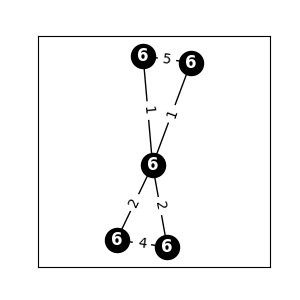}}
   \begin{aligned}
      \quad
      {}{^6_6}H^5(1, 2, 3, 4, 5, 6)
      (3, 4, 5, 6, 7, 8)
      (1, 2, 7, 8, 9, 10) \\
      (10, 11, 12, 13, 14, 15)
      (9, 11, 12, 13, 14, 15)
   \end{aligned}
\end{equation}

\begin{equation}
   \raisebox{-0.4\height}{\includegraphics[height=6\baselineskip]{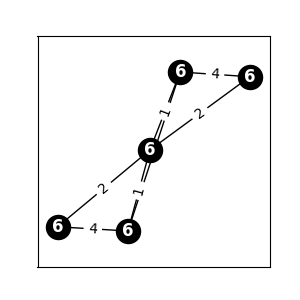}}
   \begin{aligned}
      \quad
      {}{^6_6}H^5(1, 2, 3, 4, 5, 6)
      (3, 4, 5, 6, 7, 8)
      (1, 2, 8, 9, 10, 11) \\
      (7, 11, 12, 13, 14, 15)
      (9, 10, 12, 13, 14, 15)
   \end{aligned}
\end{equation}

\begin{equation}
   \raisebox{-0.4\height}{\includegraphics[height=6\baselineskip]{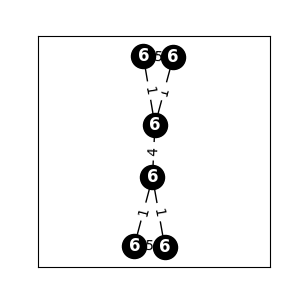}}
   \begin{aligned}
      \quad
      {}{^6_6}H^6(1, 2, 3, 4, 5, 6)
      (2, 3, 4, 5, 6, 7)
      (1, 7, 8, 9, 10, 11) \\
      (8, 9, 10, 11, 12, 13)
      (13, 14, 15, 16, 17, 18) \\
      (12, 14, 15, 16, 17, 18)
   \end{aligned}
\end{equation}

\begin{equation}
   \raisebox{-0.4\height}{\includegraphics[height=6\baselineskip]{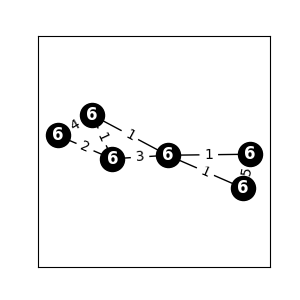}}
   \begin{aligned}
      \quad
      {}{^6_6}H^6(1, 2, 3, 4, 5, 6)
      (3, 4, 5, 6, 7, 8)
      (1, 2, 8, 9, 10, 11)     \\
      (7, 9, 10, 11, 12, 13)
      (13, 14, 15, 16, 17, 18) \\
      (12, 14, 15, 16, 17, 18)
   \end{aligned}
\end{equation}

\begin{equation}
   \raisebox{-0.4\height}{\includegraphics[height=6\baselineskip]{pics/pure_6_6_8}}
   \begin{aligned}
      \quad
      {}{^6_6}H^6(1, 2, 3, 4, 5, 6)
      (2, 3, 4, 5, 6, 7)
      (1, 8, 9, 10, 11, 12)   \\
      (7, 9, 10, 11, 12, 13)
      (8, 14, 15, 16, 17, 18) \\
      (13, 14, 15, 16, 17, 18)
   \end{aligned}
\end{equation}

\begin{equation}
   \raisebox{-0.4\height}{\includegraphics[height=6\baselineskip]{pics/pure_6_6_9}}
   \begin{aligned}
      \quad
      {}{^6_6}H^6(1, 2, 3, 4, 5, 6)
      (3, 4, 5, 6, 7, 8)
      (1, 2, 7, 8, 9, 10)      \\
      (9, 10, 11, 12, 13, 14)
      (13, 14, 15, 16, 17, 18) \\
      (11, 12, 15, 16, 17, 18)
   \end{aligned}
\end{equation}

\subsubsection{Mixed Homogeneous Invariants: 3+\textbf{3}+\textbf{3}}
\begin{equation}
   \raisebox{-0.4\height}{\includegraphics[height=6\baselineskip]{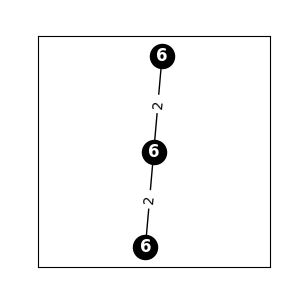}}
   \begin{aligned}
      \quad
      {}{^6_2}H^2{}{^6_4}H(1, 2)  (3, 4)  (1, 2, 3, 4)
   \end{aligned}
\end{equation}

\begin{equation}
   \raisebox{-0.4\height}{\includegraphics[height=6\baselineskip]{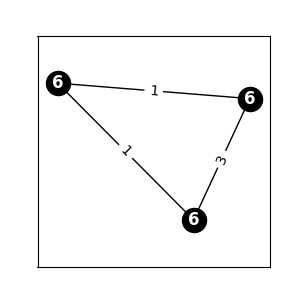}}
   \begin{aligned}
      \quad
      {}{^6_2}H{}{^6_4}H^2(1, 2)  (2, 3, 4, 5)  (1, 3, 4, 5)
   \end{aligned}
\end{equation}

\begin{equation}
   \raisebox{-0.4\height}{\includegraphics[height=6\baselineskip]{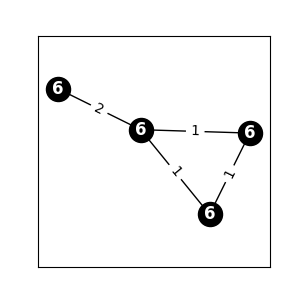}}
   \begin{aligned}
      \quad
      {}{^6_2}H^3{}{^6_4}H(1, 2)  (2, 3)  (4, 5)  (1, 3, 4, 5)
   \end{aligned}
\end{equation}

\begin{equation}
   \raisebox{-0.4\height}{\includegraphics[height=6\baselineskip]{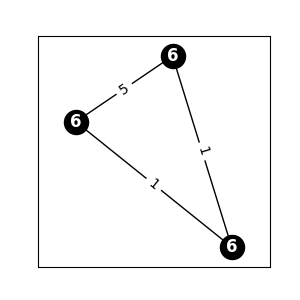}}
   \begin{aligned}
      \quad
      {}{^6_2}H{}{^6_6}H^2(1, 2)  (2, 3, 4, 5, 6, 7)  (1, 3, 4, 5, 6, 7)
   \end{aligned}
\end{equation}

\begin{equation}
   \raisebox{-0.4\height}{\includegraphics[height=6\baselineskip]{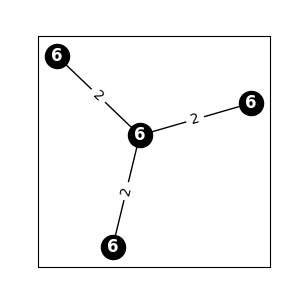}}
   \begin{aligned}
      \quad
      {}{^6_2}H^3{}{^6_6}H(1, 2)  (3, 4)  (5, 6)  (1, 2, 3, 4, 5, 6)
   \end{aligned}
\end{equation}

\begin{equation}
   \raisebox{-0.4\height}{\includegraphics[height=6\baselineskip]{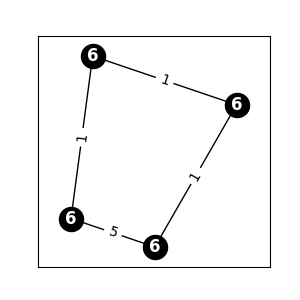}}
   \begin{aligned}
      \quad
      {}{^6_2}H^2{}{^6_6}H^2(1, 2)  (2, 3)  (1, 4, 5, 6, 7, 8)  (3, 4, 5, 6, 7, 8)
   \end{aligned}
\end{equation}

\begin{equation}
   \raisebox{-0.4\height}{\includegraphics[height=6\baselineskip]{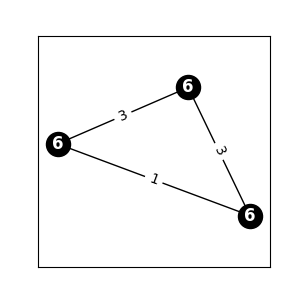}}
   \begin{aligned}
      \quad
      {}{^6_4}H^2{}{^6_6}H (1, 2, 3, 4)  (4, 5, 6, 7)  (1, 2, 3, 5, 6, 7)
   \end{aligned}
\end{equation}

\begin{equation}
   \raisebox{-0.4\height}{\includegraphics[height=6\baselineskip]{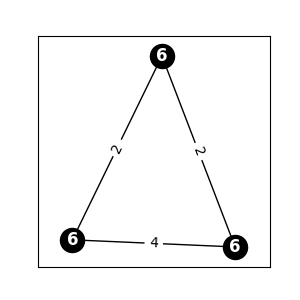}}
   \begin{aligned}
      \quad
      {}{^6_4}H{}{^6_6}H^2 (1, 2, 3, 4)  (3, 4, 5, 6, 7, 8)  (1, 2, 5, 6, 7, 8)
   \end{aligned}
\end{equation}

\begin{equation}
   \raisebox{-0.4\height}{\includegraphics[height=6\baselineskip]{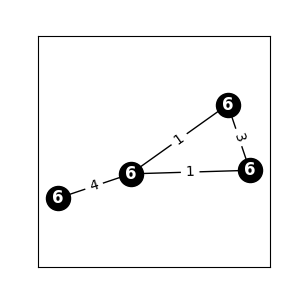}}
   \begin{aligned}
      \quad
      {}{^6_4}H^3{}{^6_6}H^2(1, 2, 3, 4)  (2, 3, 4, 5)  (6, 7, 8, 9)  (1, 5, 6, 7, 8, 9)
   \end{aligned}
\end{equation}

\subsubsection{Mixed Simultaneous 1 6 Invariants: 2+2+\textbf{2}}
\begin{equation}
   \raisebox{-0.4\height}{\includegraphics[height=6\baselineskip]{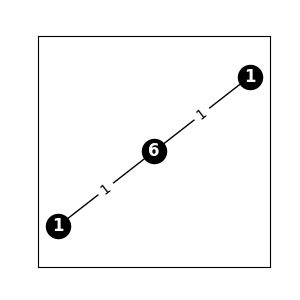}}
   \begin{aligned}
      \quad
      {}{^1_1}H^2{}{^6_2}H(1)  (2)  (1, 2)
   \end{aligned}
\end{equation}

\begin{equation}
   \raisebox{-0.4\height}{\includegraphics[height=6\baselineskip]{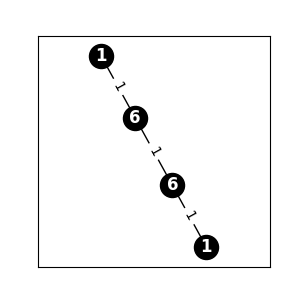}}
   \begin{aligned}
      \quad
      {}{^1_1}H^2{}{^6_2}H^2(1)  (2)  (1,3)  (2, 3)
   \end{aligned}
\end{equation}

\begin{equation}
   \raisebox{-0.4\height}{\includegraphics[height=6\baselineskip]{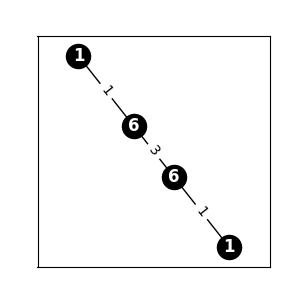}}
   \begin{aligned}
      \quad
      {}{^1_1}H^2{}{^6_4}H^2(1)  (2)  (1, 3, 4, 5)  (2, 3, 4, 5)
   \end{aligned}
\end{equation}

\begin{equation}
   \raisebox{-0.4\height}{\includegraphics[height=6\baselineskip]{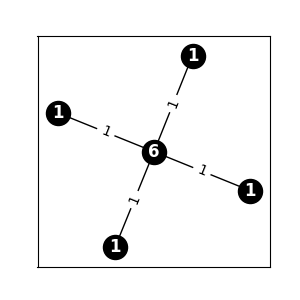}}
   \begin{aligned}
      \quad
      {}{^1_1}H^4{}{^6_4}H(1)  (2)  (3)  (4)  (1, 2, 3, 4)
   \end{aligned}
\end{equation}

\begin{equation}
   \raisebox{-0.4\height}{\includegraphics[height=6\baselineskip]{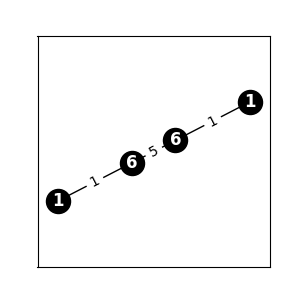}}
   \begin{aligned}
      \quad
      {}{^1_1}H^2{}{^6_6}H^2(1)  (2)  (1, 3, 4, 5, 6, 7)  (2, 3, 4, 5, 6, 7)
   \end{aligned}
\end{equation}

\begin{equation}
   \raisebox{-0.4\height}{\includegraphics[height=6\baselineskip]{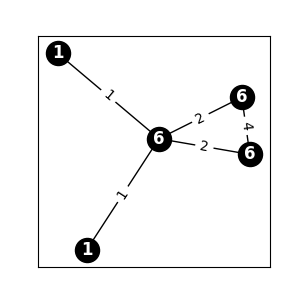}}
   \begin{aligned}
      \quad
      {}{^1_1}H^2{}{^6_6}H^3(1)  (2)  (1, 2, 3, 4, 5, 6)  (5, 6, 7, 8, 9, 10) \\
      (3, 4, 7, 8, 9, 10)
   \end{aligned}
\end{equation}

\subsubsection{Mixed Simultaneous 2 6 Invariants: 3+3+3}
\begin{equation}
   \raisebox{-0.4\height}{\includegraphics[height=6\baselineskip]{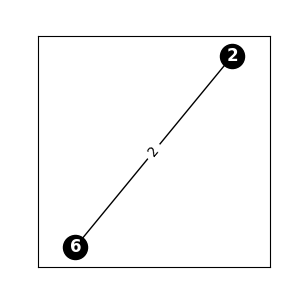}}
   \begin{aligned}
      \quad
      {}{^2_2}H{}{^6_2}H(1,2)(1,2)
   \end{aligned}
\end{equation}

\begin{equation}
   \raisebox{-0.4\height}{\includegraphics[height=6\baselineskip]{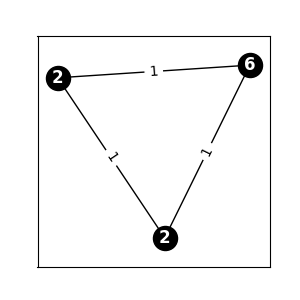}}
   \begin{aligned}
      \quad
      {}{^2_2}H^2{}{^6_2}H(1,2)(2,3)(1,3)
   \end{aligned}
\end{equation}

\begin{equation}
   \raisebox{-0.4\height}{\includegraphics[height=6\baselineskip]{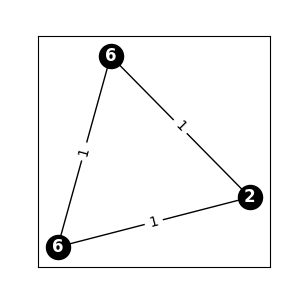}}
   \begin{aligned}
      \quad
      {}{^2_2}H{}{^6_2}H^2(1,2)(2,3)(1,3)
   \end{aligned}
\end{equation}

\begin{equation}
   \raisebox{-0.4\height}{\includegraphics[height=6\baselineskip]{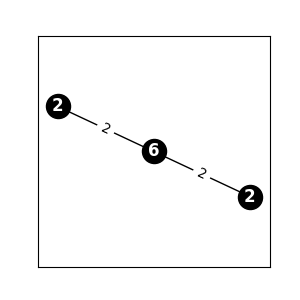}}
   \begin{aligned}
      \quad
      {}{^2_2}H^2{}{^6_4}H(1, 2)  (3, 4)  (1, 2, 3, 4)
   \end{aligned}
\end{equation}

\begin{equation}
   \raisebox{-0.4\height}{\includegraphics[height=6\baselineskip]{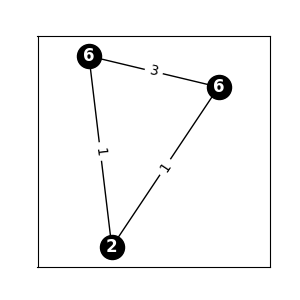}}
   \begin{aligned}
      \quad
      {}{^2_2}H{}{^6_4}H^2(1, 2)  (2, 3, 4, 5)  (1, 3, 4, 5)
   \end{aligned}
\end{equation}

\begin{equation}
   \raisebox{-0.4\height}{\includegraphics[height=6\baselineskip]{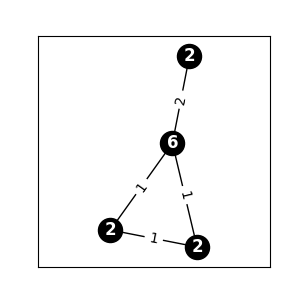}}
   \begin{aligned}
      \quad
      {}{^2_2}H^3{}{^6_4}H(1, 2)  (2, 3)  (4, 5)  (1, 3, 4, 5)
   \end{aligned}
\end{equation}

\begin{equation}
   \raisebox{-0.4\height}{\includegraphics[height=6\baselineskip]{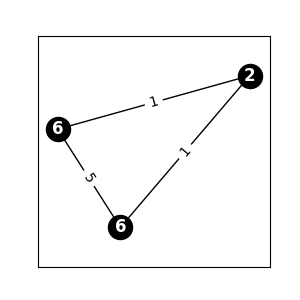}}
   \begin{aligned}
      \quad
      {}{^2_2}H{}{^6_6}H^2(1, 2)  (2, 3, 4, 5, 6, 7)  (1, 3, 4, 5, 6, 7)
   \end{aligned}
\end{equation}

\begin{equation}
   \raisebox{-0.4\height}{\includegraphics[height=6\baselineskip]{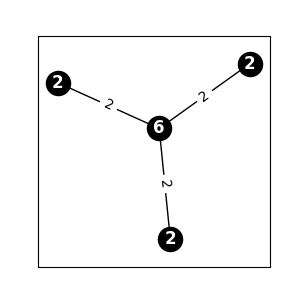}}
   \begin{aligned}
      \quad
      {}{^2_2}H^3{}{^6_6}H(1, 2)  (3, 4)  (5, 6)  (1, 2, 3, 4, 5, 6)
   \end{aligned}
\end{equation}

\begin{equation}
   \raisebox{-0.4\height}{\includegraphics[height=6\baselineskip]{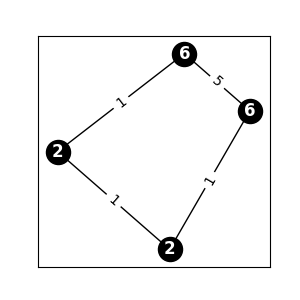}}
   \begin{aligned}
      \quad
      {}{^2_2}H^2{}{^6_6}H^2(1, 2)  (2, 3)  (1, 4, 5, 6, 7, 8)  (3, 4, 5, 6, 7, 8)
   \end{aligned}
\end{equation}

\subsubsection{Mixed Simultaneous 3 6 Invariants: 2+2+2+3+3+\textbf{3}}
\begin{equation}
   \raisebox{-0.4\height}{\includegraphics[height=6\baselineskip]{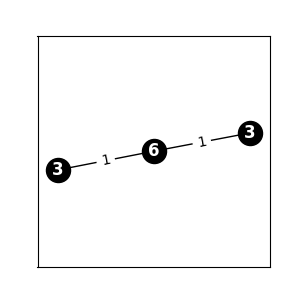}}
   \begin{aligned}
      \quad
      {}{^3_1}H^2{}{^6_2}H(1)(2)(1,2)
   \end{aligned}
\end{equation}

\begin{equation}
   \raisebox{-0.4\height}{\includegraphics[height=6\baselineskip]{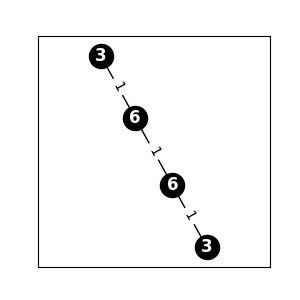}}
   \begin{aligned}
      \quad
      {}{^3_1}H^2{}{^6_2}H^2(1)(2)(1,3)(2,3)
   \end{aligned}
\end{equation}

\begin{equation}
   \raisebox{-0.4\height}{\includegraphics[height=6\baselineskip]{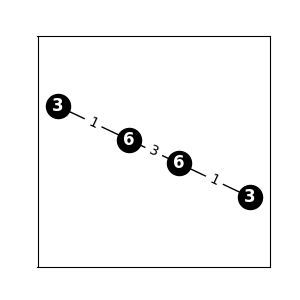}}
   \begin{aligned}
      \quad
      {}{^3_1}H^2{}{^6_4}H^2(1)  (2)  (1, 3, 4, 5)  (2, 3, 4, 5)
   \end{aligned}
\end{equation}

\begin{equation}
   \raisebox{-0.4\height}{\includegraphics[height=6\baselineskip]{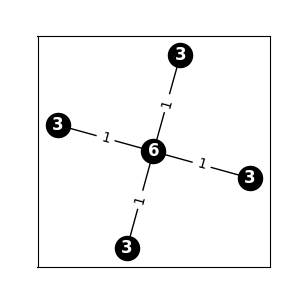}}
   \begin{aligned}
      \quad
      {}{^3_1}H^4{}{^6_4}H(1)  (2)  (3)  (4)  (1, 2, 3, 4)
   \end{aligned}
\end{equation}

\begin{equation}
   \raisebox{-0.4\height}{\includegraphics[height=6\baselineskip]{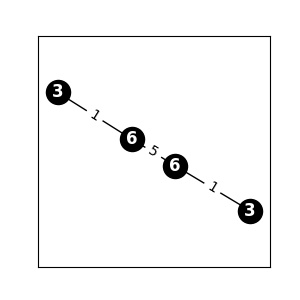}}
   \begin{aligned}
      \quad
      {}{^3_1}H^2{}{^6_6}H^2(1)  (2)  (1, 3, 4, 5, 6, 7)  (2, 3, 4, 5, 6, 7)
   \end{aligned}
\end{equation}

\begin{equation}
   \raisebox{-0.4\height}{\includegraphics[height=6\baselineskip]{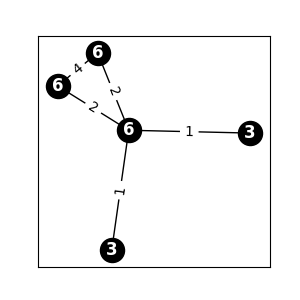}}
   \begin{aligned}
      \quad
      {}{^3_1}H^2{}{^6_6}H^3(1)  (2)  (1, 2, 3, 4, 5, 6)  (5, 6, 7, 8, 9, 10) \\
      (3, 4, 7, 8, 9, 10)
   \end{aligned}
\end{equation}

\begin{equation}
   \raisebox{-0.4\height}{\includegraphics[height=6\baselineskip]{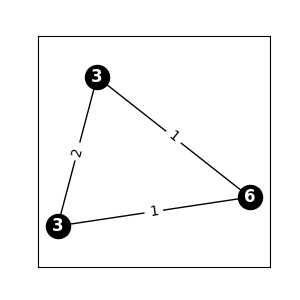}}
   \begin{aligned}
      \quad
      {}{^6_2}H{}{^3_3}H^2(1, 2)  (2, 3, 4)  (1, 3, 4)
   \end{aligned}
\end{equation}

\begin{equation}
   \raisebox{-0.4\height}{\includegraphics[height=6\baselineskip]{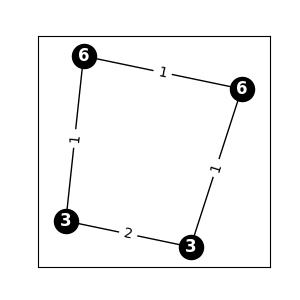}}
   \begin{aligned}
      \quad
      {}{^6_2}H^2{}{^3_3}H^2(1, 2)  (2, 3)  (1, 4, 5)  (3, 4, 5)
   \end{aligned}
\end{equation}

\begin{equation}
   \raisebox{-0.4\height}{\includegraphics[height=6\baselineskip]{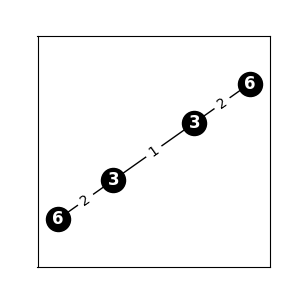}}
   \begin{aligned}
      \quad
      {}{^6_2}H^2{}{^3_3}H^2(1, 2)  (3, 4)  (1, 2, 5)  (3, 4, 5)
   \end{aligned}
\end{equation}

\begin{equation}
   \raisebox{-0.4\height}{\includegraphics[height=6\baselineskip]{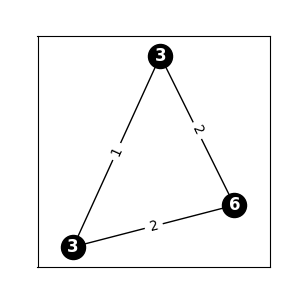}}
   \begin{aligned}
      \quad
      {}{^3_3}H^2{}{^6_4}H(1, 2, 3)  (3, 4, 5)  (1, 2, 4, 5)
   \end{aligned}
\end{equation}

\begin{equation}
   \raisebox{-0.4\height}{\includegraphics[height=6\baselineskip]{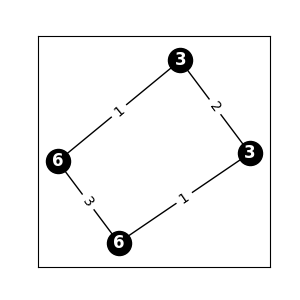}}
   \begin{aligned}
      \quad
      {}{^3_3}H^2{}{^6_4}H^2(1, 2, 3)  (2, 3, 4)  (1, 5, 6, 7)  (4, 5, 6, 7)
   \end{aligned}
\end{equation}

\begin{equation}
   \raisebox{-0.4\height}{\includegraphics[height=6\baselineskip]{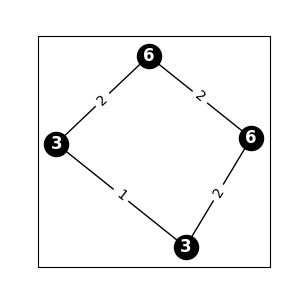}}
   \begin{aligned}
      \quad
      {}{^3_3}H^2{}{^6_4}H^2(1, 2, 3)  (3, 4, 5)  (1, 2, 6, 7)  (4, 5, 6, 7)
   \end{aligned}
\end{equation}

\begin{equation}
   \raisebox{-0.4\height}{\includegraphics[height=6\baselineskip]{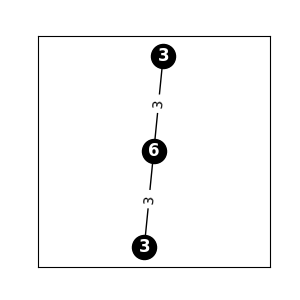}}
   \begin{aligned}
      \quad
      {}{^3_3}H^2{}{^6_6}H(1, 2, 3)  (4, 5, 6)  (1, 2, 3, 4, 5, 6)
   \end{aligned}
\end{equation}

\begin{equation}
   \raisebox{-0.4\height}{\includegraphics[height=6\baselineskip]{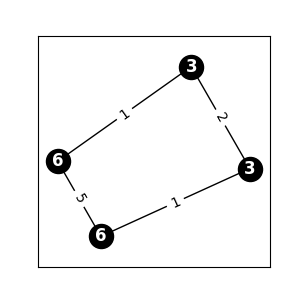}}
   \begin{aligned}
      \quad
      {}{^3_3}H^2{}{^6_6}H^2(1, 2, 3)  (2, 3, 4)  (1, 5, 6, 7, 8, 9)  (4, 5, 6, 7, 8, 9)
   \end{aligned}
\end{equation}

\begin{equation}
   \raisebox{-0.4\height}{\includegraphics[height=6\baselineskip]{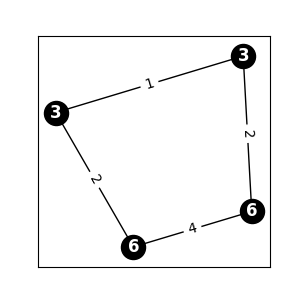}}
   \begin{aligned}
      \quad
      {}{^3_3}H^2{}{^6_6}H^2(1, 2, 3)  (3, 4, 5)  (1, 2, 6, 7, 8, 9)  (4, 5, 6, 7, 8, 9)
   \end{aligned}
\end{equation}

\subsubsection{Mixed Simultaneous 4 6 Invariants: 3+3+3+3+\textbf{3}+3}
\begin{equation}
   \raisebox{-0.4\height}{\includegraphics[height=6\baselineskip]{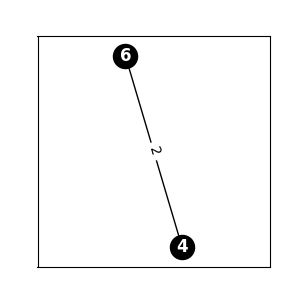}}
   \begin{aligned}
      \quad
      {}{^4_2}H{}{^6_2}H(1,2)(1,2)
   \end{aligned}
\end{equation}

\begin{equation}
   \raisebox{-0.4\height}{\includegraphics[height=6\baselineskip]{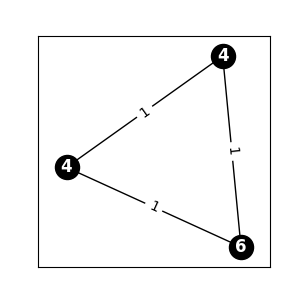}}
   \begin{aligned}
      \quad
      {}{^4_2}H^2{}{^6_2}H(1,2)(2,3)(1,3)
   \end{aligned}
\end{equation}

\begin{equation}
   \raisebox{-0.4\height}{\includegraphics[height=6\baselineskip]{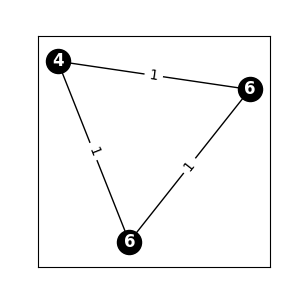}}
   \begin{aligned}
      \quad
      {}{^4_2}H{}{^6_2}H^2(1,2)(2,3)(1,3)
   \end{aligned}
\end{equation}

\begin{equation}
   \raisebox{-0.4\height}{\includegraphics[height=6\baselineskip]{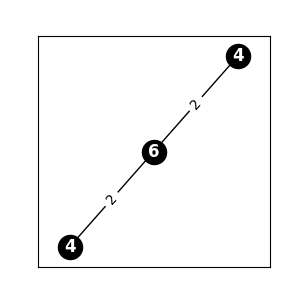}}
   \begin{aligned}
      \quad
      {}{^4_2}H^2{}{^6_4}H(1, 2)  (3, 4)  (1, 2, 3, 4)
   \end{aligned}
\end{equation}

\begin{equation}
   \raisebox{-0.4\height}{\includegraphics[height=6\baselineskip]{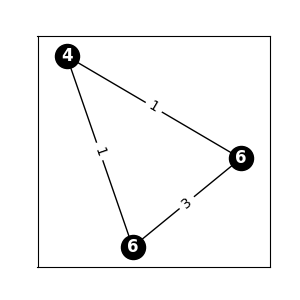}}
   \begin{aligned}
      \quad
      {}{^4_2}H{}{^6_4}H^2(1, 2)  (2, 3, 4, 5)  (1, 3, 4, 5)
   \end{aligned}
\end{equation}

\begin{equation}
   \raisebox{-0.4\height}{\includegraphics[height=6\baselineskip]{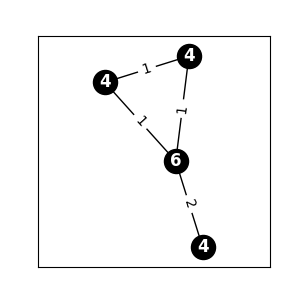}}
   \begin{aligned}
      \quad
      {}{^4_2}H^3{}{^6_4}H(1, 2)  (2, 3)  (4, 5)  (1, 3, 4, 5)
   \end{aligned}
\end{equation}

\begin{equation}
   \raisebox{-0.4\height}{\includegraphics[height=6\baselineskip]{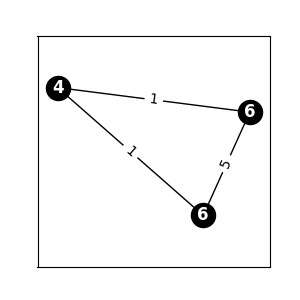}}
   \begin{aligned}
      \quad
      {}{^4_2}H{}{^6_6}H^2(1, 2)  (2, 3, 4, 5, 6, 7)  (1, 3, 4, 5, 6, 7)
   \end{aligned}
\end{equation}

\begin{equation}
   \raisebox{-0.4\height}{\includegraphics[height=6\baselineskip]{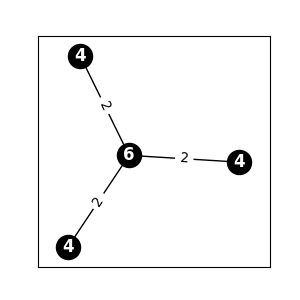}}
   \begin{aligned}
      \quad
      {}{^4_2}H^3{}{^6_6}H(1, 2)  (3, 4)  (5, 6)  (1, 2, 3, 4, 5, 6)
   \end{aligned}
\end{equation}

\begin{equation}
   \raisebox{-0.4\height}{\includegraphics[height=6\baselineskip]{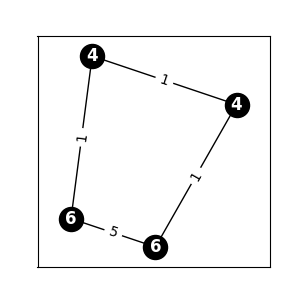}}
   \begin{aligned}
      \quad
      {}{^4_2}H^2{}{^6_6}H^2(1, 2)  (2, 3)  (1, 4, 5, 6, 7, 8)  (3, 4, 5, 6, 7, 8)
   \end{aligned}
\end{equation}

\begin{equation}
   \raisebox{-0.4\height}{\includegraphics[height=6\baselineskip]{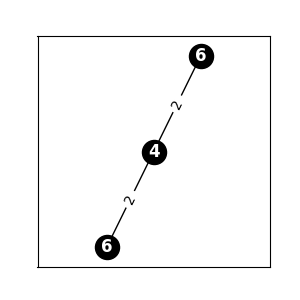}}
   \begin{aligned}
      \quad
      {}{^6_2}H^2{}{^4_4}H(1, 2)  (3, 4)  (1, 2, 3, 4)
   \end{aligned}
\end{equation}

\begin{equation}
   \raisebox{-0.4\height}{\includegraphics[height=6\baselineskip]{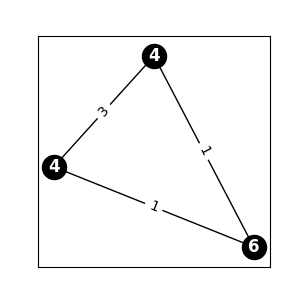}}
   \begin{aligned}
      \quad
      {}{^6_2}H{}{^4_4}H^2(1, 2)  (2, 3, 4, 5)  (1, 3, 4, 5)
   \end{aligned}
\end{equation}

\begin{equation}
   \raisebox{-0.4\height}{\includegraphics[height=6\baselineskip]{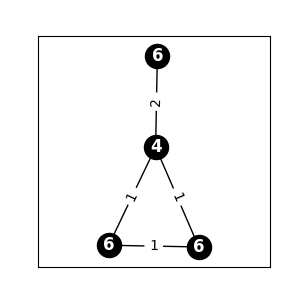}}
   \begin{aligned}
      \quad
      {}{^6_2}H^3{}{^4_4}H(1, 2)  (2, 3)  (4, 5)  (1, 3, 4, 5)
   \end{aligned}
\end{equation}

\begin{equation}
   \raisebox{-0.4\height}{\includegraphics[height=6\baselineskip]{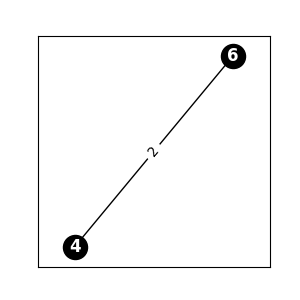}}
   \begin{aligned}
      \quad
      {}{^4_4}H{}{^6_4}H (1,2,3,4) (1,2,3,4)
   \end{aligned}
\end{equation}

\begin{equation}
   \raisebox{-0.4\height}{\includegraphics[height=6\baselineskip]{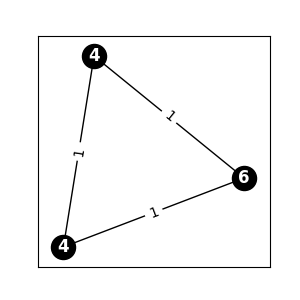}}
   \begin{aligned}
      \quad
      {}{^4_4}H^2{}{^6_4}H (1,2,3,4) (3,4,5,6) (1,2,5,6)
   \end{aligned}
\end{equation}

\begin{equation}
   \raisebox{-0.4\height}{\includegraphics[height=6\baselineskip]{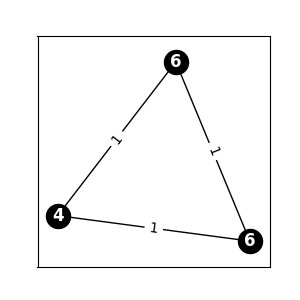}}
   \begin{aligned}
      \quad
      {}{^4_4}H{}{^6_4} H^2 (1,2,3,4) (3,4,5,6) (1,2,5,6)
   \end{aligned}
\end{equation}

\begin{equation}
   \raisebox{-0.4\height}{\includegraphics[height=6\baselineskip]{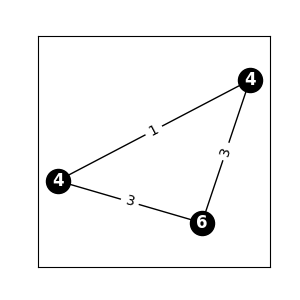}}
   \begin{aligned}
      \quad
      {}{^4_4}H^2{}{^6_6}H (1, 2, 3, 4)  (4, 5, 6, 7)  (1, 2, 3, 5, 6, 7)
   \end{aligned}
\end{equation}

\begin{equation}
   \raisebox{-0.4\height}{\includegraphics[height=6\baselineskip]{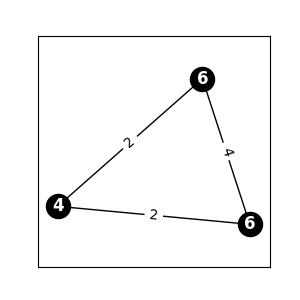}}
   \begin{aligned}
      \quad
      {}{^4_4}H{}{^6_6}H^2 (1, 2, 3, 4)  (3, 4, 5, 6, 7, 8)  (1, 2, 5, 6, 7, 8)
   \end{aligned}
\end{equation}

\begin{equation}
   \raisebox{-0.4\height}{\includegraphics[height=6\baselineskip]{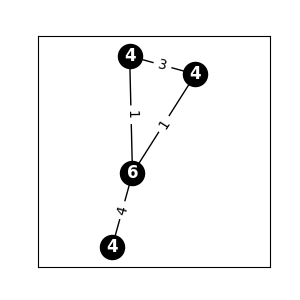}}
   \begin{aligned}
      \quad
      {}{^4_4}H^3{}{^6_6}H^2(1, 2, 3, 4)  (2, 3, 4, 5)  (6, 7, 8, 9)  (1, 5, 6, 7, 8, 9)
   \end{aligned}
\end{equation}

\subsubsection{Mixed Simultaneous 5 6 Invariants: 2+3+3+2+3+3+2+3+\textbf{3}}
\begin{equation}
   \raisebox{-0.4\height}{\includegraphics[height=6\baselineskip]{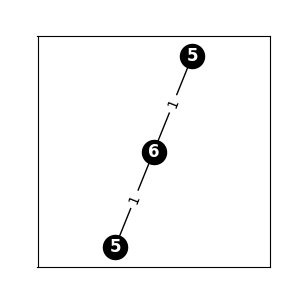}}
   \begin{aligned}
      \quad
      {}{^5_1}H^2{}{^6_2}H(1)(2)(1,2)
   \end{aligned}
\end{equation}

\begin{equation}
   \raisebox{-0.4\height}{\includegraphics[height=6\baselineskip]{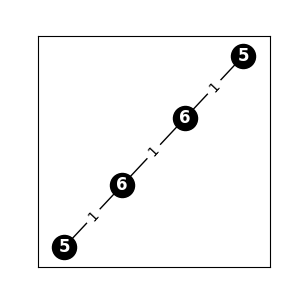}}
   \begin{aligned}
      \quad
      {}{^5_1}H^2{}{^6_2}H^2(1)(2)(1,3)(2,3)
   \end{aligned}
\end{equation}

\begin{equation}
   \raisebox{-0.4\height}{\includegraphics[height=6\baselineskip]{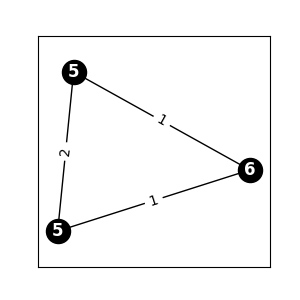}}
   \begin{aligned}
      \quad
      {}{^6_2}H{}{^5_3}H^2(1, 2)  (2, 3, 4)  (1, 3, 4)
   \end{aligned}
\end{equation}

\begin{equation}
   \raisebox{-0.4\height}{\includegraphics[height=6\baselineskip]{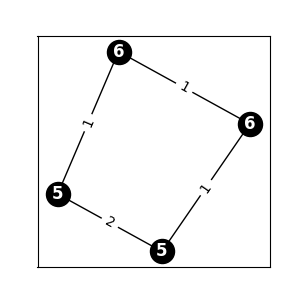}}
   \begin{aligned}
      \quad
      {}{^6_2}H^2{}{^5_3}H^2(1, 2)  (2, 3)  (1, 4, 5)  (3, 4, 5)
   \end{aligned}
\end{equation}

\begin{equation}
   \raisebox{-0.4\height}{\includegraphics[height=6\baselineskip]{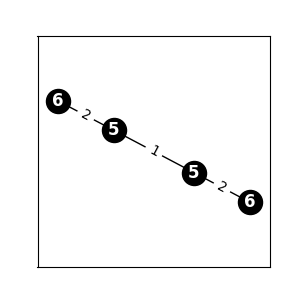}}
   \begin{aligned}
      \quad
      {}{^6_2}H^2{}{^5_3}H^2(1, 2)  (3, 4)  (1, 2, 5)  (3, 4, 5)
   \end{aligned}
\end{equation}

\begin{equation}
   \raisebox{-0.4\height}{\includegraphics[height=6\baselineskip]{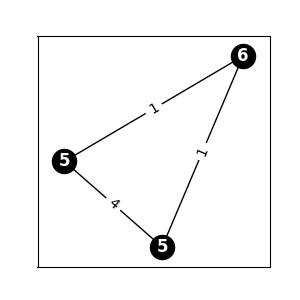}}
   \begin{aligned}
      \quad
      {}{^6_2}H{}{^5_5}H^2(1, 2)  (2, 3, 4, 5, 6)  (1, 3, 4, 5, 6)
   \end{aligned}
\end{equation}

\begin{equation}
   \raisebox{-0.4\height}{\includegraphics[height=6\baselineskip]{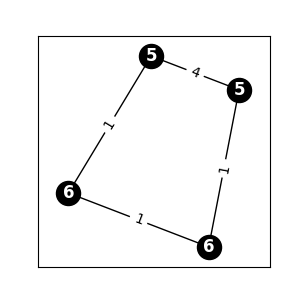}}
   \begin{aligned}
      \quad
      {}{^6_2}H^2{}{^5_5}H^2(1, 2)  (2, 3)  (1, 4, 5, 6, 7)  (3, 4, 5, 6, 7)
   \end{aligned}
\end{equation}

\begin{equation}
   \raisebox{-0.4\height}{\includegraphics[height=6\baselineskip]{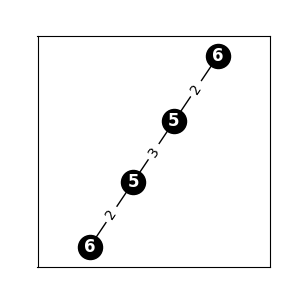}}
   \begin{aligned}
      \quad
      {}{^6_2}H^2{}{^5_5}H^2(1, 2)  (3, 4)  (1, 2, 5, 6, 7)  (3, 4, 5, 6, 7)
   \end{aligned}
\end{equation}

\begin{equation}
   \raisebox{-0.4\height}{\includegraphics[height=6\baselineskip]{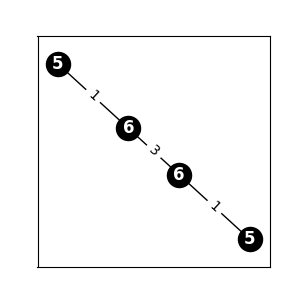}}
   \begin{aligned}
      \quad
      {}{^5_1}H^2{}{^6_4}H^2(1)  (2)  (1, 3, 4, 5)  (2, 3, 4, 5)
   \end{aligned}
\end{equation}

\begin{equation}
   \raisebox{-0.4\height}{\includegraphics[height=6\baselineskip]{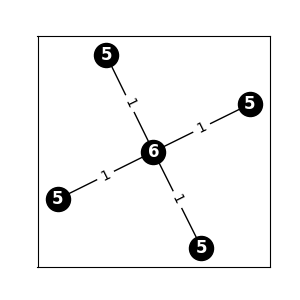}}
   \begin{aligned}
      \quad
      {}{^5_1}H^4{}{^6_4}H(1)  (2)  (3)  (4)  (1, 2, 3, 4)
   \end{aligned}
\end{equation}

\begin{equation}
   \raisebox{-0.4\height}{\includegraphics[height=6\baselineskip]{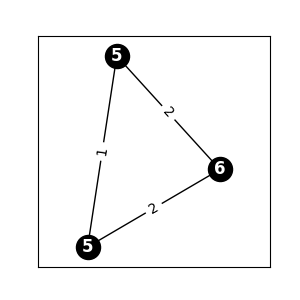}}
   \begin{aligned}
      \quad
      {}{^5_3}H^2{}{^6_4}H(1, 2, 3)  (3, 4, 5)  (1, 2, 4, 5)
   \end{aligned}
\end{equation}

\begin{equation}
   \raisebox{-0.4\height}{\includegraphics[height=6\baselineskip]{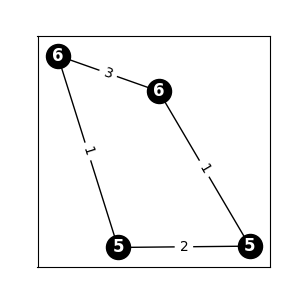}}
   \begin{aligned}
      \quad
      {}{^5_3}H^2{}{^6_4}H^2(1, 2, 3)  (2, 3, 4)  (1, 5, 6, 7)  (4, 5, 6, 7)
   \end{aligned}
\end{equation}

\begin{equation}
   \raisebox{-0.4\height}{\includegraphics[height=6\baselineskip]{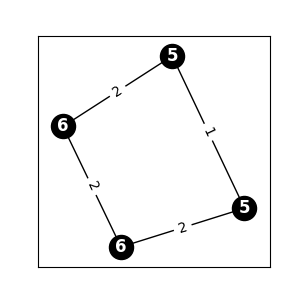}}
   \begin{aligned}
      \quad
      {}{^5_3}H^2{}{^6_4}H^2(1, 2, 3)  (3, 4, 5)  (1, 2, 6, 7)  (4, 5, 6, 7)
   \end{aligned}
\end{equation}

\begin{equation}
   \raisebox{-0.4\height}{\includegraphics[height=6\baselineskip]{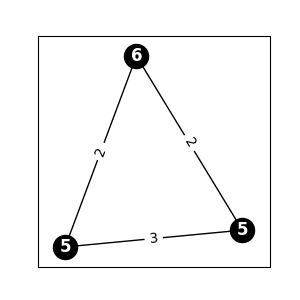}}
   \begin{aligned}
      \quad
      {}{^6_4}H{}{^5_5}H^2(1, 2, 3, 4)  (3, 4, 5, 6, 7)  (1, 2, 5, 6, 7)
   \end{aligned}
\end{equation}

\begin{equation}
   \raisebox{-0.4\height}{\includegraphics[height=6\baselineskip]{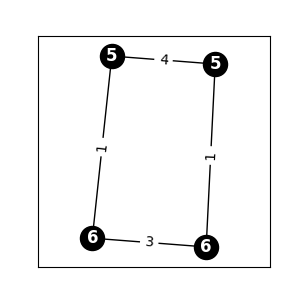}}
   \begin{aligned}
      \quad
      {}{^6_4}H^2{}{^5_5}H^2(1, 2, 3, 4)  (2, 3, 4, 5)  (1, 6, 7, 8, 9)  (5, 6, 7, 8, 9)
   \end{aligned}
\end{equation}

\begin{equation}
   \raisebox{-0.4\height}{\includegraphics[height=6\baselineskip]{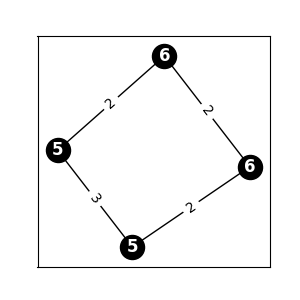}}
   \begin{aligned}
      \quad
      {}{^6_4}H^2{}{^5_5}H^2(1, 2, 3, 4)  (3, 4, 5, 6)  (1, 2, 7, 8, 9)  (5, 6, 7, 8, 9)
   \end{aligned}
\end{equation}

\begin{equation}
   \raisebox{-0.4\height}{\includegraphics[height=6\baselineskip]{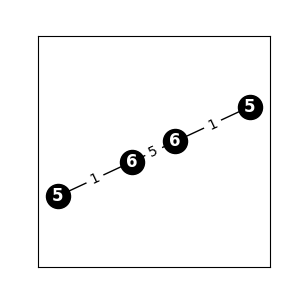}}
   \begin{aligned}
      \quad
      {}{^5_1}H^2{}{^6_6}H^2(1)  (2)  (1, 3, 4, 5, 6, 7)  (2, 3, 4, 5, 6, 7)
   \end{aligned}
\end{equation}

\begin{equation}
   \raisebox{-0.4\height}{\includegraphics[height=6\baselineskip]{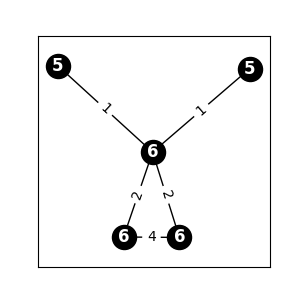}}
   \begin{aligned}
      \quad
      {}{^5_1}H^2{}{^6_6}H^3(1)  (2)  (1, 2, 3, 4, 5, 6)  (5, 6, 7, 8, 9, 10) \\
      (3, 4, 7, 8, 9, 10)
   \end{aligned}
\end{equation}

\begin{equation}
   \raisebox{-0.4\height}{\includegraphics[height=6\baselineskip]{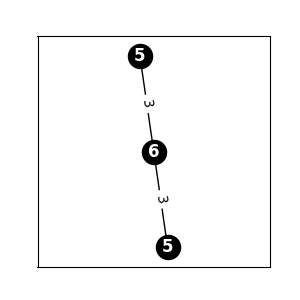}}
   \begin{aligned}
      \quad
      {}{^5_3}H^2{}{^6_6}H(1, 2, 3)  (4, 5, 6)  (1, 2, 3, 4, 5, 6)
   \end{aligned}
\end{equation}

\begin{equation}
   \raisebox{-0.4\height}{\includegraphics[height=6\baselineskip]{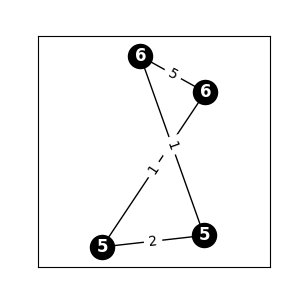}}
   \begin{aligned}
      \quad
      {}{^5_3}H^2{}{^6_6}H^2(1, 2, 3)  (2, 3, 4)  (1, 5, 6, 7, 8, 9)  (4, 5, 6, 7, 8, 9)
   \end{aligned}
\end{equation}

\begin{equation}
   \raisebox{-0.4\height}{\includegraphics[height=6\baselineskip]{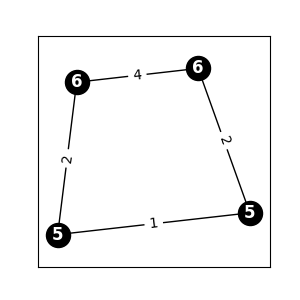}}
   \begin{aligned}
      \quad
      {}{^5_3}H^2{}{^6_6}H^2(1, 2, 3)  (3, 4, 5)  (1, 2, 6, 7, 8, 9)  (4, 5, 6, 7, 8, 9)
   \end{aligned}
\end{equation}

\begin{equation}
   \raisebox{-0.4\height}{\includegraphics[height=6\baselineskip]{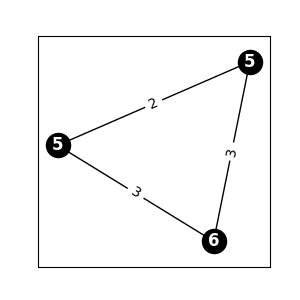}}
   \begin{aligned}
      \quad
      {}{^5_5}H^2{}{^6_6}H(1, 2, 3, 4, 5)  (4, 5, 6, 7, 8)  (1, 2, 3, 6, 7, 8)
   \end{aligned}
\end{equation}

\begin{equation}
   \raisebox{-0.4\height}{\includegraphics[height=6\baselineskip]{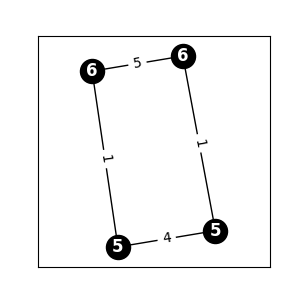}}
   \begin{aligned}
      \quad
      {}{^5_5}H^2{}{^6_6}H^2(1, 2, 3, 4, 5) (2, 3, 4, 5, 6)  (1, 7, 8, 9, 10, 11) \\
      (6, 7, 8, 9, 10, 11)
   \end{aligned}
\end{equation}

\begin{equation}
   \raisebox{-0.4\height}{\includegraphics[height=6\baselineskip]{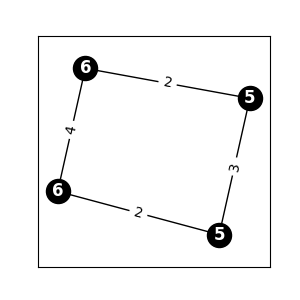}}
   \begin{aligned}
      \quad
      {}{^5_5}H^2{}{^6_6}H^2(1, 2, 3, 4, 5)  (3, 4, 5, 6, 7)  (1, 2, 8, 9, 10, 11) \\
      (6, 7, 8, 9, 10, 11)
   \end{aligned}
\end{equation}

% \begin{equation}
% \raisebox{-0.4\height}{\includegraphics[height=6\baselineskip]{mixed_13_1}}
% \begin{aligned} 
% \quad
% {}{^1_1}H^2{}{^3_3}H^2(1)  (2)  (1,3,4)  (2,3,4)
% \end{aligned} 
% \end{equation}

% \begin{equation}
% \raisebox{-0.4\height}{\includegraphics[height=6\baselineskip]{mixed_12_1}}
% \begin{aligned} 
% \quad
% {}{^3_1}H^2{}{^2_2}H^2(1)(2)(1,3)(2,3)
% \end{aligned} 
% \end{equation}

% \begin{equation}
% \raisebox{-0.4\height}{\includegraphics[height=6\baselineskip]{mixed_23_1}}
% \begin{aligned} 
% \quad
% {}{^2_2}H^2{}{^3_3}H^2(1, 2)  (2, 3)  (1, 4, 5)  (3, 4, 5)
% \end{aligned} 
% \end{equation}

% \begin{equation}
% \raisebox{-0.4\height}{\includegraphics[height=6\baselineskip]{mixed_23_2}}
% \begin{aligned} 
% \quad
% {}{^2_2}H^2{}{^3_3}H^2(1, 2)  (3, 4)  (1, 2, 5)  (3, 4, 5)
% \end{aligned} 
% \end{equation}

\section{Minimal Flexible Sets for Spherical Functions}\label{a:minimalSpherical}
Assume that the functions in question are all spherical, i.e., that they do not depend on $r$ in Eq.~\eqref{spherical}. Then we know from Sec~\ref{s:analysis} that all irreducible decompositions are trivial, i.e., that only the first irreducible tensor, whose low-left index equals its upper-left index, is non-zero and is identical to the moment tensor. In this setting, the minimal flexible set reduces significantly, such that there are no longer two invariants with the same contraction structure.

\subsection{Order 0}

\subsubsection{Irreducible Decomposition: 1}
\begin{equation}
   \begin{aligned}
      {}{^0}{\hat M}={}{^0_0}H
   \end{aligned}
\end{equation}

\subsubsection{Independent Elements of the Irreducible Decomposition: 1}
\begin{equation}
   \begin{aligned}
      {}{^0_0}H_i
   \end{aligned}
\end{equation}

\subsubsection{Pure Invariants: \textbf{1}}
\begin{equation}
   \raisebox{-0.4\height}{\includegraphics[height=6\baselineskip]{pics/pure_0_0_0}}
   \begin{aligned}
      \quad
      {}{^0_0}H
   \end{aligned}
\end{equation}

% \subsubsection{mixed^ Invariants: 0}
% \subsubsection{Mixed Homogeneous Invariants: 0}
% \subsubsection{Mixed Simultaneous Invariants: 0 }

\subsection{Order 1}

\subsubsection{Irreducible Decomposition: 1}
\begin{equation}
   \begin{aligned}
      {}{^1}{\hat M}_i & ={}{^1_1}H_i
   \end{aligned}
\end{equation}

\subsubsection{Independent Elements of the Irreducible Decomposition: 1}
\begin{equation}
   \begin{aligned}
      {}{^1_1}H_i
   \end{aligned}
\end{equation}

\subsubsection{Pure Invariants: \textbf{1}}
\begin{equation}
   \raisebox{-0.4\height}{\includegraphics[height=6\baselineskip]{pics/pure_1_1_0.png}}
   \begin{aligned}
      % {}{^1_1}H^{(1,2)}
      \quad
      {}{^1_1}H^2(1)  (1)
   \end{aligned}
\end{equation}

% \subsubsection{mixed^ Invariants: 0}
% \subsubsection{Mixed Homogeneous Invariants: 0}
% \subsubsection{Mixed Simultaneous Invariants: 0 }

\subsection{Order 2}

\subsubsection{Irreducible Decomposition: 2}
\begin{equation}
   \begin{aligned}
      {}{^2}{\hat M}_{ij} & ={}{^2_2}H_{ij}+{}{_{0}^{2}}H\delta_{ij}
   \end{aligned}
\end{equation}

\subsubsection{Independent Elements of the Irreducible Decomposition: 1}
\begin{equation}
   \begin{aligned}
      {}{^2_2}H_{ij}
   \end{aligned}
\end{equation}

\subsubsection{Pure Invariants: \textbf{2}}
\begin{equation}
   \raisebox{-0.4\height}{\includegraphics[height=6\baselineskip]{pics/pure_2_2_0.png}}
   \begin{aligned}
      \quad
      {}{^2_2}H^2(1,2)(1,2)
   \end{aligned}
\end{equation}

\begin{equation}
   \raisebox{-0.4\height}{\includegraphics[height=6\baselineskip]{pics/pure_2_2_1.png}}
   \begin{aligned}
      \quad
      {}{^2_2}H^3(1,2)(2,3)(1,3)
   \end{aligned}
\end{equation}

\subsubsection{Mixed Simultaneous 1 2 Invariants: \textbf{2}}
\begin{equation}
   \raisebox{-0.4\height}{\includegraphics[height=6\baselineskip]{pics/mixed_1_1_2_2_0.png}}
   \begin{aligned}
      \quad
      {}{^1_1}H^2{}{^2_2}H(1)(2)(1,2)
   \end{aligned}
\end{equation}

\begin{equation}
   \raisebox{-0.4\height}{\includegraphics[height=6\baselineskip]{pics/mixed_1_1_2_2_1.png}}
   \begin{aligned}
      \quad
      {}{^1_1}H^2{}{^2_2}H^2(1)(2)(1,3)(2,3)
   \end{aligned}
\end{equation}

\subsection{Order 3}

\subsubsection{Irreducible Decomposition: 2}
\begin{equation}
   \begin{aligned}
      {}{^3}{\hat M}_{ijk} & ={}{^3_3}H_{ijk}+ {}{^3_{1}}{H}_{(i }   \delta_{jk)}
   \end{aligned}
\end{equation}

\subsubsection{Independent Elements of the Irreducible Decomposition: 1}
\begin{equation}
   \begin{aligned}
      {}{^3_3}H_i
   \end{aligned}
\end{equation}

\subsubsection{Pure Invariants: \textbf{4}}
\begin{equation}
   \raisebox{-0.4\height}{\includegraphics[height=6\baselineskip]{pics/pure_3_3_0.png}}
   \begin{aligned}
      \quad
      {}{^3_3}H^2(1,2,3)(1,2,3)
   \end{aligned}
\end{equation}

\begin{equation}
   \raisebox{-0.4\height}{\includegraphics[height=6\baselineskip]{pics/pure_3_3_1.png}}
   \begin{aligned}
      \quad
      {}{^3_3}H^4(1,2,3)  (1,2,4)  (3,5,6)  (4,5,6)
   \end{aligned}
\end{equation}

\begin{equation}
   \raisebox{-0.4\height}{\includegraphics[height=6\baselineskip]{pics/pure_3_3_2.png}}
   \begin{aligned}
      \quad
      {}{^3_3}H^6(1, 2, 3)  (2, 3, 4)  (1, 4, 5)  (5, 6, 7)  (7, 8, 9)  (6, 8, 9)
   \end{aligned}
\end{equation}

\begin{equation}
   \raisebox{-0.4\height}{\includegraphics[height=6\baselineskip]{pics/pure_3_3_3.png}}
   \begin{aligned}
      \quad
      {}{^3_3}H^{10}(1, 2, 3)  (2, 3, 4)  (1, 4, 5)  (5, 6, 7)  (6, 7, 8)  (8, 9, 10) \\
      (9, 10, 11)  (11, 12, 13)  (13, 14, 15)  (12, 14, 15)
   \end{aligned}
\end{equation}

\subsubsection{Mixed Simultaneous 1 3 Invariants: \textbf{2}}
\begin{equation}
   \raisebox{-0.4\height}{\includegraphics[height=6\baselineskip]{pics/mixed_1_1_3_3_0.png}}
   \begin{aligned}
      \quad
      {}{^1_1}H^3{}{^3_3}H(1)(2)(3)(1, 2, 3)
   \end{aligned}
\end{equation}

\begin{equation}
   \raisebox{-0.4\height}{\includegraphics[height=6\baselineskip]{pics/mixed_1_1_3_3_1.png}}
   \begin{aligned}
      \quad
      {}{^1_1}H^2{}{^3_3}H^2(1)  (2)  (1,3,4)  (2,3,4)
   \end{aligned}
\end{equation}

\subsubsection{Mixed Simultaneous 2 3 Invariants: \textbf{3}}
\begin{equation}
   \raisebox{-0.4\height}{\includegraphics[height=6\baselineskip]{pics/mixed_2_2_3_3_0.png}}
   \begin{aligned}
      \quad
      {}{^2_2}H{}{^3_3}H^2(1, 2)  (2, 3, 4)  (1, 3, 4)
   \end{aligned}
\end{equation}

\begin{equation}
   \raisebox{-0.4\height}{\includegraphics[height=6\baselineskip]{pics/mixed_2_2_3_3_1.png}}
   \begin{aligned}
      \quad
      {}{^2_2}H^2{}{^3_3}H^2(1, 2)  (2, 3)  (1, 4, 5)  (3, 4, 5)
   \end{aligned}
\end{equation}

\begin{equation}
   \raisebox{-0.4\height}{\includegraphics[height=6\baselineskip]{pics/mixed_2_2_3_3_2.png}}
   \begin{aligned}
      \quad
      {}{^2_2}H^2{}{^3_3}H^2(1, 2)  (3, 4)  (1, 2, 5)  (3, 4, 5)
   \end{aligned}
\end{equation}

\subsection{Order 4}

\subsubsection{Irreducible Decomposition: 3}
\begin{equation}
   \begin{aligned}
      {}{^4}{\hat M}_{ijkl} & =
      {}{^4_4}H_{ijkl}
      +{}{^4_{2}}{H}_{(ij }  \delta_{kl)}
      +{}{^4_{0}}{H}   \delta_{(ij}\delta_{kl)}
   \end{aligned}
\end{equation}

\subsubsection{Independent Elements of the Irreducible Decomposition: 1}
\begin{equation}
   \begin{aligned}
      {}{^4_4}H_i
   \end{aligned}
\end{equation}

\subsubsection{Pure Invariants:\textbf{6}}
\begin{equation}
   \raisebox{-0.4\height}{\includegraphics[height=6\baselineskip]{pics/pure_4_4_0}}
   \begin{aligned}
      \quad
      {}{^4_4}H^2(1, 2, 3, 4)  (1, 2, 3, 4)
   \end{aligned}
\end{equation}

\begin{equation}
   \raisebox{-0.4\height}{\includegraphics[height=6\baselineskip]{pics/pure_4_4_1.png}}
   \begin{aligned}
      \quad
      {}{^4_4}H^3(1, 2, 3, 4)  (3, 4, 5, 6)  (1, 2, 5, 6)
   \end{aligned}
\end{equation}

\begin{equation}
   \raisebox{-0.4\height}{\includegraphics[height=6\baselineskip]{pics/pure_4_4_2.png}}
   \begin{aligned}
      \quad
      {}{^4_4}H^4(1, 2, 3, 4)  (2, 3, 4, 5)  (1, 6, 7, 8)  (5, 6, 7, 8)
   \end{aligned}
\end{equation}

\begin{equation}
   \raisebox{-0.4\height}{\includegraphics[height=6\baselineskip]{pics/pure_4_4_3.png}}
   \begin{aligned}
      \quad
      {}{^4_4}H^5(1, 2, 3, 4)  (2, 3, 4, 5)  (1, 5, 6, 7) \\
      (7, 8, 9, 10)  (6, 8, 9, 10)
   \end{aligned}
\end{equation}

\begin{equation}
   \raisebox{-0.4\height}{\includegraphics[height=6\baselineskip]{pics/pure_4_4_4.png}}
   \begin{aligned}
      \quad
      {}{^4_4}H^6(1, 2, 3, 4)
      (2, 3, 4, 5)
      (1, 5, 6, 7)
      (6, 7, 8, 9) \\
      (9, 10, 11, 12)
      (8, 10, 11, 12)
   \end{aligned}
\end{equation}

\begin{equation}
   \raisebox{-0.4\height}{\includegraphics[height=6\baselineskip]{pics/pure_4_4_5.png}}
   \begin{aligned}
      \quad
      {}{^4_4}H^7(1, 2, 3, 4)
      (2, 3, 4, 5)
      (1, 5, 6, 7)
      (6, 7, 8, 9) \\
      (8, 9, 10, 11) 
      (11, 12, 13, 14)
      (10, 12, 13, 14)
   \end{aligned}
\end{equation}

\subsubsection{Mixed Simultaneous 1 4 Invariants: \textbf{2}}
\begin{equation}
   \raisebox{-0.4\height}{\includegraphics[height=6\baselineskip]{pics/mixed_1_1_4_4_0.png}}
   \begin{aligned}
      \quad
      {}{^1_1}H^2{}{^4_4}H^2(1)  (2)  (1, 3, 4, 5)  (2, 3, 4, 5)
   \end{aligned}
\end{equation}

\begin{equation}
   \raisebox{-0.4\height}{\includegraphics[height=6\baselineskip]{pics/mixed_1_1_4_4_1.png}}
   \begin{aligned}
      \quad
      {}{^1_1}H^4{}{^4_4}H(1)  (2)  (3)  (4)  (1, 2, 3, 4)
   \end{aligned}
\end{equation}

\subsubsection{Mixed Simultaneous 2 4 Invariants: \textbf{3}}
\begin{equation}
   \raisebox{-0.4\height}{\includegraphics[height=6\baselineskip]{pics/mixed_2_2_4_4_0.png}}
   \begin{aligned}
      \quad
      {}{^2_2}H^2{}{^4_4}H(1, 2)  (3, 4)  (1, 2, 3, 4)
   \end{aligned}
\end{equation}

\begin{equation}
   \raisebox{-0.4\height}{\includegraphics[height=6\baselineskip]{pics/mixed_2_2_4_4_1.png}}
   \begin{aligned}
      \quad
      {}{^2_2}H{}{^4_4}H^2(1, 2)  (2, 3, 4, 5)  (1, 3, 4, 5)
   \end{aligned}
\end{equation}

\begin{equation}
   \raisebox{-0.4\height}{\includegraphics[height=6\baselineskip]{pics/mixed_2_2_4_4_2.png}}
   \begin{aligned}
      \quad
      {}{^2_2}H^3{}{^4_4}H(1, 2)  (2, 3)  (4, 5)  (1, 3, 4, 5)
   \end{aligned}
\end{equation}

\subsubsection{Mixed Simultaneous 3 4 Invariants: \textbf{3}}
\begin{equation}
   \raisebox{-0.4\height}{\includegraphics[height=6\baselineskip]{pics/mixed_3_3_4_4_0.png}}
   \begin{aligned}
      \quad
      {}{^3_3}H^2{}{^4_4}H(1, 2, 3)  (3, 4, 5)  (1, 2, 4, 5)
   \end{aligned}
\end{equation}

\begin{equation}
   \raisebox{-0.4\height}{\includegraphics[height=6\baselineskip]{pics/mixed_3_3_4_4_1.png}}
   \begin{aligned}
      \quad
      {}{^3_3}H^2{}{^4_4}H^2(1, 2, 3)  (2, 3, 4)  (1, 5, 6, 7)  (4, 5, 6, 7)
   \end{aligned}
\end{equation}

\begin{equation}
   \raisebox{-0.4\height}{\includegraphics[height=6\baselineskip]{pics/mixed_3_3_4_4_2.png}}
   \begin{aligned}
      \quad
      {}{^3_3}H^2{}{^4_4}H^2(1, 2, 3)  (3, 4, 5)  (1, 2, 6, 7)  (4, 5, 6, 7)
   \end{aligned}
\end{equation}

\subsection{Order 5}

\subsubsection{Irreducible Decomposition: 3}
\begin{equation}
   \begin{aligned}
      {}{^5}{\hat M}_{ijklm}
       & =
      {}{^5_5} H_{ijklm}
      +
      {}{5_3} H_{(ijk}\delta_{lm)}
      +
      {}{^5_1} H_{(i}\delta_{jk}\delta_{lm)}
   \end{aligned}
\end{equation}

\subsubsection{Independent Elements of the Irreducible Decomposition: 1}
\begin{equation}
   \begin{aligned}
      {}{^5_5}H_i
   \end{aligned}
\end{equation}

\subsubsection{Pure Invariants: \textbf{8}}
\begin{equation}
   \raisebox{-0.4\height}{\includegraphics[height=6\baselineskip]{pics/pure_5_5_0.png}}
   \begin{aligned}
      \quad
      {}{^5_5}H^2(1, 2, 3, 4, 5)  (1, 2, 3, 4, 5)
   \end{aligned}
\end{equation}

\begin{equation}
   \raisebox{-0.4\height}{\includegraphics[height=6\baselineskip]{pics/pure_5_5_1.png}}
   \begin{aligned}
      \quad
      {}{^5_5}H^4(1, 2, 3, 4, 5)  (2, 3, 4, 5, 6) \\ (1, 7, 8, 9, 10)  (6, 7, 8, 9, 10)
   \end{aligned}
\end{equation}

\begin{equation}
   \raisebox{-0.4\height}{\includegraphics[height=6\baselineskip]{pics/pure_5_5_2.png}}
   \begin{aligned}
      \quad
      {}{^5_5}H^6(1, 2, 3, 4, 5)
      (2, 3, 4, 5, 6)
      (1, 6, 7, 8, 9) \\
      (7, 8, 9, 10, 11) 
      (11, 12, 13, 14, 15) \\
      (10, 12, 13, 14, 15)
   \end{aligned}
\end{equation}

\begin{equation}
   \raisebox{-0.4\height}{\includegraphics[height=6\baselineskip]{pics/pure_5_5_3.png}}
   \begin{aligned}
      \quad
      {}{^5_5}H^6(1, 2, 3, 4, 5)
      (3, 4, 5, 6, 7)
      (1, 2, 7, 8, 9) \\
      (6, 8, 9, 10, 11) 
      (11, 12, 13, 14, 15) \\
      (10, 12, 13, 14, 15)
   \end{aligned}
\end{equation}

\begin{equation}
   \raisebox{-0.4\height}{\includegraphics[height=6\baselineskip]{pics/pure_5_5_4.png}}
   \begin{aligned}
      \quad
      {}{^5_5}H^6(1, 2, 3, 4, 5)
      (2, 3, 4, 5, 6)
      (1, 7, 8, 9, 10) \\
      (6, 8, 9, 10, 11)
      (7, 12, 13, 14, 15) \\
      (11, 12, 13, 14, 15)
   \end{aligned}
\end{equation}

\begin{equation}
   \raisebox{-0.4\height}{\includegraphics[height=6\baselineskip]{pics/pure_5_5_5.png}}
   \begin{aligned}
      \quad
      {}{^5_5}H^6(1, 2, 3, 4, 5)
      (3, 4, 5, 6, 7)
      (1, 2, 6, 7, 8) \\
      (8, 9, 10, 11, 12) 
      (11, 12, 13, 14, 15) \\
      (9, 10, 13, 14, 15)
   \end{aligned}
\end{equation}

\begin{equation}
   \raisebox{-0.4\height}{\includegraphics[height=6\baselineskip]{pics/pure_5_5_6.png}}
   \begin{aligned}
      \quad
      {}{^5_5}H^8(1, 2, 3, 4, 5)
      (2, 3, 4, 5, 6)
      (1, 6, 7, 8, 9) \\
      (7, 8, 9, 10, 11)
      (10, 11, 12, 13, 14) \\
      (12, 13, 14, 15, 16) 
      (16, 17, 18, 19, 20) \\
      (15, 17, 18, 19, 20)
   \end{aligned}
\end{equation}

\begin{equation}
   \raisebox{-0.4\height}{\includegraphics[height=6\baselineskip]{pics/pure_5_5_7.png}}
   \begin{aligned}
      \quad
      {}{^5_5}H^8(5, 5, 5, 5, 5, 5, 5, 5) (0, 1, 2, 3, 4)  (2, 3, 4, 5, 6) \\
      (0, 1, 6, 7, 8)  (5, 7, 8, 9, 10)  (9, 10, 11, 12, 13)                     \\
      (11, 12, 13, 14, 15)  (15, 16, 17, 18, 19)                                 \\
      (14, 16, 17, 18, 19)
   \end{aligned}
\end{equation}

\subsubsection{Mixed Simultaneous 1 5 Invariants: \textbf{2}}
\begin{equation}
   \raisebox{-0.4\height}{\includegraphics[height=6\baselineskip]{pics/mixed_1_1_5_5_0.png}}
   \begin{aligned}
      \quad
      {}{^1_1}H^2{}{^5_5}H^2((1)(2)(1, 3, 4, 5, 6)(2, 3, 4, 5, 6))
   \end{aligned}
\end{equation}

\begin{equation}
   \raisebox{-0.4\height}{\includegraphics[height=6\baselineskip]{pics/mixed_1_1_5_5_1.png}}
   \begin{aligned}
      \quad
      {}{^1_1}H{}{^5_5}H^3(1)(1, 2, 3, 4, 5)(4, 5, 6, 7, 8)(2, 3, 6, 7, 8)
   \end{aligned}
\end{equation}

\subsubsection{Mixed Simultaneous 2 5 Invariants: \textbf{3}}
\begin{equation}
   \raisebox{-0.4\height}{\includegraphics[height=6\baselineskip]{pics/mixed_2_2_5_5_0.png}}
   \begin{aligned}
      \quad
      {}{^2_2}H{}{^5_5}H^2(1, 2)  (2, 3, 4, 5, 6)  (1, 3, 4, 5, 6)
   \end{aligned}
\end{equation}

\begin{equation}
   \raisebox{-0.4\height}{\includegraphics[height=6\baselineskip]{pics/mixed_2_2_5_5_1.png}}
   \begin{aligned}
      \quad
      {}{^2_2}H^2{}{^5_5}H^2(1, 2)  (2, 3)  (1, 4, 5, 6, 7)  (3, 4, 5, 6, 7)
   \end{aligned}
\end{equation}

\begin{equation}
   \raisebox{-0.4\height}{\includegraphics[height=6\baselineskip]{pics/mixed_2_2_5_5_2.png}}
   \begin{aligned}
      \quad
      {}{^2_2}H^2{}{^5_5}H^2(1, 2)  (3, 4)  (1, 2, 5, 6, 7)  (3, 4, 5, 6, 7)
   \end{aligned}
\end{equation}

\subsubsection{Mixed Simultaneous 3 5 Invariants: \textbf{3}}
\begin{equation}
   \raisebox{-0.4\height}{\includegraphics[height=6\baselineskip]{pics/mixed_3_3_5_5_0.png}}
   \begin{aligned}
      \quad
      {}{^3_3}H^3{}{^5_5}H(1, 2, 3)  (2, 3, 4)  (5, 6, 7)  (1, 4, 5, 6, 7)
   \end{aligned}
\end{equation}

\begin{equation}
   \raisebox{-0.4\height}{\includegraphics[height=6\baselineskip]{pics/mixed_3_3_5_5_1.png}}
   \begin{aligned}
      \quad
      {}{^3_3}H^2{}{^5_5}H^2(1, 2, 3)  (2, 3, 4)  (1, 5, 6, 7, 8)  (4, 5, 6, 7, 8)
   \end{aligned}
\end{equation}

\begin{equation}
   \raisebox{-0.4\height}{\includegraphics[height=6\baselineskip]{pics/mixed_3_3_5_5_2.png}}
   \begin{aligned}
      \quad
      {}{^3_3}H^2{}{^5_5}H^2(1, 2, 3)  (3, 4, 5)  (1, 2, 6, 7, 8)  (4, 5, 6, 7, 8)
   \end{aligned}
\end{equation}

\subsubsection{Mixed Simultaneous 4 5 Invariants: \textbf{3}}
\begin{equation}
   \raisebox{-0.4\height}{\includegraphics[height=6\baselineskip]{pics/mixed_4_4_5_5_0.png}}
   \begin{aligned}
      \quad
      {}{^4_4}H{}{^5_5}H^2(1, 2, 3, 4)  (3, 4, 5, 6, 7)  (1, 2, 5, 6, 7)
   \end{aligned}
\end{equation}

\begin{equation}
   \raisebox{-0.4\height}{\includegraphics[height=6\baselineskip]{pics/mixed_4_4_5_5_1.png}}
   \begin{aligned}
      \quad
      {}{^4_4}H^2{}{^5_5}H^2(1, 2, 3, 4)  (2, 3, 4, 5)  (1, 6, 7, 8, 9)  (5, 6, 7, 8, 9)
   \end{aligned}
\end{equation}

\begin{equation}
   \raisebox{-0.4\height}{\includegraphics[height=6\baselineskip]{pics/mixed_4_4_5_5_2.png}}
   \begin{aligned}
      \quad
      {}{^4_4}H^2{}{^5_5}H^2(1, 2, 3, 4)  (3, 4, 5, 6)  (1, 2, 7, 8, 9)  (5, 6, 7, 8, 9)
   \end{aligned}
\end{equation}

\subsection{Order 6}

\subsubsection{Irreducible Decomposition: 1}
\begin{equation}
   \begin{aligned}
      {}{^6}{\hat M}_{ijklmn}
       & =
      {}{^6_6} H_{ijklmn}
      +
      {}{^6_4} H_{(ijkl}\delta_{mn)}
      +
      {}{^6_2} H_{(ij}\delta_{kl}\delta_{mn)}
      +
      {}{^6_0} H\delta_{(ij}\delta_{kl}\delta_{mn)}
   \end{aligned}
\end{equation}

\subsubsection{Independent Elements of the Irreducible Decomposition: 1}
\begin{equation}
   \begin{aligned}
      {}{^6_6}H_i
   \end{aligned}
\end{equation}

\subsubsection{Pure Invariants: \textbf{10}}
\begin{equation}
   \raisebox{-0.4\height}{\includegraphics[height=6\baselineskip]{pics/pure_6_6_0.png}}
   \begin{aligned}
      \quad
      {}{^6_6}H^2(1, 2, 3, 4, 5, 6)  (1, 2, 3, 4, 5, 6)
   \end{aligned}
\end{equation}

\begin{equation}
   \raisebox{-0.4\height}{\includegraphics[height=6\baselineskip]{pics/pure_6_6_1.png}}
   \begin{aligned}
      \quad
      {}{^6_6}H^3(1, 2, 3, 4, 5, 6), (4, 5, 6, 7, 8, 9), (1, 2, 3, 7, 8, 9)
   \end{aligned}
\end{equation}

\begin{equation}
   \raisebox{-0.4\height}{\includegraphics[height=6\baselineskip]{pics/pure_6_6_2.png}}
   \begin{aligned}
      \quad
      {}{^6_6}H^4(1, 2, 3, 4, 5, 6)
      (2, 3, 4, 5, 6, 7)
      (1, 8, 9, 10, 11, 12) \\
      (7, 8, 9, 10, 11, 12)
   \end{aligned}
\end{equation}

\begin{equation}
   \raisebox{-0.4\height}{\includegraphics[height=6\baselineskip]{pics/pure_6_6_3.png}}
   \begin{aligned}
      \quad
      {}{^6_6}H^4(1, 2, 3, 4, 5, 6)
      (3, 4, 5, 6, 7, 8)
      (1, 2, 9, 10, 11, 12) \\
      (7, 8, 9, 10, 11, 12)
   \end{aligned}
\end{equation}

\begin{equation}
   \raisebox{-0.4\height}{\includegraphics[height=6\baselineskip]{pics/pure_6_6_4.png}}
   \begin{aligned}
      \quad
      {}{^6_6}H^5(1, 2, 3, 4, 5, 6)
      (3, 4, 5, 6, 7, 8)
      (1, 2, 7, 8, 9, 10) \\
      (10, 11, 12, 13, 14, 15)
      (9, 11, 12, 13, 14, 15)
   \end{aligned}
\end{equation}

\begin{equation}
   \raisebox{-0.4\height}{\includegraphics[height=6\baselineskip]{pics/pure_6_6_5.png}}
   \begin{aligned}
      \quad
      {}{^6_6}H^5(1, 2, 3, 4, 5, 6)
      (3, 4, 5, 6, 7, 8)
      (1, 2, 8, 9, 10, 11) \\
      (7, 11, 12, 13, 14, 15)
      (9, 10, 12, 13, 14, 15)
   \end{aligned}
\end{equation}

\begin{equation}
   \raisebox{-0.4\height}{\includegraphics[height=6\baselineskip]{pics/pure_6_6_6.png}}
   \begin{aligned}
      \quad
      {}{^6_6}H^6(1, 2, 3, 4, 5, 6)
      (2, 3, 4, 5, 6, 7)
      (1, 7, 8, 9, 10, 11)     \\
      (8, 9, 10, 11, 12, 13)
      (13, 14, 15, 16, 17, 18) \\
      (12, 14, 15, 16, 17, 18)
   \end{aligned}
\end{equation}

\begin{equation}
   \raisebox{-0.4\height}{\includegraphics[height=6\baselineskip]{pics/pure_6_6_7.png}}
   \begin{aligned}
      \quad
      {}{^6_6}H^6(1, 2, 3, 4, 5, 6)
      (3, 4, 5, 6, 7, 8)
      (1, 2, 8, 9, 10, 11)     \\
      (7, 9, 10, 11, 12, 13)
      (13, 14, 15, 16, 17, 18) \\
      (12, 14, 15, 16, 17, 18)
   \end{aligned}
\end{equation}

\begin{equation}
   \raisebox{-0.4\height}{\includegraphics[height=6\baselineskip]{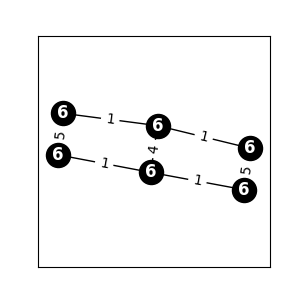}}
   \begin{aligned}
      \quad
      {}{^6_6}H^6(1, 2, 3, 4, 5, 6)
      (2, 3, 4, 5, 6, 7)
      (1, 8, 9, 10, 11, 12)   \\
      (7, 9, 10, 11, 12, 13)
      (8, 14, 15, 16, 17, 18) \\
      (13, 14, 15, 16, 17, 18)
   \end{aligned}
\end{equation}

\begin{equation}
   \raisebox{-0.4\height}{\includegraphics[height=6\baselineskip]{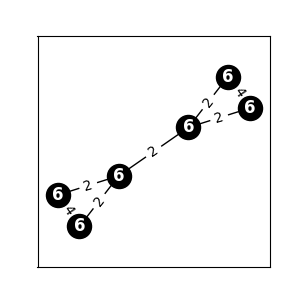}}
   \begin{aligned}
      \quad
      {}{^6_6}H^6(1, 2, 3, 4, 5, 6)
      (3, 4, 5, 6, 7, 8)
      (1, 2, 7, 8, 9, 10)      \\
      (9, 10, 11, 12, 13, 14)
      (13, 14, 15, 16, 17, 18) \\
      (11, 12, 15, 16, 17, 18)
   \end{aligned}
\end{equation}

\subsubsection{Mixed Simultaneous 1 6 Invariants: \textbf{2}}
\begin{equation}
   \raisebox{-0.4\height}{\includegraphics[height=6\baselineskip]{pics/mixed_1_1_6_6_0.png}}
   \begin{aligned}
      \quad
      {}{^1_1}H^2{}{^6_6}H^2(1)  (2)  (1, 3, 4, 5, 6, 7)  (2, 3, 4, 5, 6, 7)
   \end{aligned}
\end{equation}

\begin{equation}
   \raisebox{-0.4\height}{\includegraphics[height=6\baselineskip]{pics/mixed_1_1_6_6_1.png}}
   \begin{aligned}
      \quad
      {}{^1_1}H^2{}{^6_6}H^3(1)  (2)  (1, 2, 3, 4, 5, 6)  (5, 6, 7, 8, 9, 10) \\
      (3, 4, 7, 8, 9, 10)
   \end{aligned}
\end{equation}

\subsubsection{Mixed Simultaneous 2 6 Invariants: \textbf{3}}
\begin{equation}
   \raisebox{-0.4\height}{\includegraphics[height=6\baselineskip]{pics/mixed_2_2_6_6_0.png}}
   \begin{aligned}
      \quad
      {}{^2_2}H{}{^6_6}H^2(1, 2)  (2, 3, 4, 5, 6, 7)  (1, 3, 4, 5, 6, 7)
   \end{aligned}
\end{equation}

\begin{equation}
   \raisebox{-0.4\height}{\includegraphics[height=6\baselineskip]{pics/mixed_2_2_6_6_1.png}}
   \begin{aligned}
      \quad
      {}{^2_2}H^3{}{^6_6}H(1, 2)  (3, 4)  (5, 6)  (1, 2, 3, 4, 5, 6)
   \end{aligned}
\end{equation}

\begin{equation}
   \raisebox{-0.4\height}{\includegraphics[height=6\baselineskip]{pics/mixed_2_2_6_6_2.png}}
   \begin{aligned}
      \quad
      {}{^2_2}H^2{}{^6_6}H^2(1, 2)  (2, 3)  (1, 4, 5, 6, 7, 8)  (3, 4, 5, 6, 7, 8)
   \end{aligned}
\end{equation}

\subsubsection{Mixed Simultaneous 3 6 Invariants: \textbf{3}}
\begin{equation}
   \raisebox{-0.4\height}{\includegraphics[height=6\baselineskip]{pics/mixed_3_3_6_6_0.png}}
   \begin{aligned}
      \quad
      {}{^3_3}H^2{}{^6_6}H(1, 2, 3)  (4, 5, 6)  (1, 2, 3, 4, 5, 6)
   \end{aligned}
\end{equation}

\begin{equation}
   \raisebox{-0.4\height}{\includegraphics[height=6\baselineskip]{pics/mixed_3_3_6_6_1.png}}
   \begin{aligned}
      \quad
      {}{^3_3}H^2{}{^6_6}H^2(1, 2, 3)  (2, 3, 4)  (1, 5, 6, 7, 8, 9)  (4, 5, 6, 7, 8, 9)
   \end{aligned}
\end{equation}

\begin{equation}
   \raisebox{-0.4\height}{\includegraphics[height=6\baselineskip]{pics/mixed_3_3_6_6_2.png}}
   \begin{aligned}
      \quad
      {}{^3_3}H^2{}{^6_6}H^2(1, 2, 3)  (3, 4, 5)  (1, 2, 6, 7, 8, 9)  (4, 5, 6, 7, 8, 9)
   \end{aligned}
\end{equation}

\subsubsection{Mixed Simultaneous 4 6 Invariants: \textbf{3}}
\begin{equation}
   \raisebox{-0.4\height}{\includegraphics[height=6\baselineskip]{pics/mixed_4_4_6_6_0.png}}
   \begin{aligned}
      \quad
      {}{^4_4}H^2{}{^6_6}H (1, 2, 3, 4)  (4, 5, 6, 7)  (1, 2, 3, 5, 6, 7)
   \end{aligned}
\end{equation}

\begin{equation}
   \raisebox{-0.4\height}{\includegraphics[height=6\baselineskip]{pics/mixed_4_4_6_6_1.png}}
   \begin{aligned}
      \quad
      {}{^4_4}H{}{^6_6}H^2 (1, 2, 3, 4)  (3, 4, 5, 6, 7, 8)  (1, 2, 5, 6, 7, 8)
   \end{aligned}
\end{equation}

\begin{equation}
   \raisebox{-0.4\height}{\includegraphics[height=6\baselineskip]{pics/mixed_4_4_6_6_2.png}}
   \begin{aligned}
      \quad
      {}{^4_4}H^3{}{^6_6}H^2(1, 2, 3, 4)  (2, 3, 4, 5)  (6, 7, 8, 9)  (1, 5, 6, 7, 8, 9)
   \end{aligned}
\end{equation}

\subsubsection{Mixed Simultaneous 5 6 Invariants: \textbf{3}}
\begin{equation}
   \raisebox{-0.4\height}{\includegraphics[height=6\baselineskip]{pics/mixed_5_5_6_6_0.png}}
   \begin{aligned}
      \quad
      {}{^5_5}H^2{}{^6_6}H(1, 2, 3, 4, 5)  (4, 5, 6, 7, 8)  (1, 2, 3, 6, 7, 8)
   \end{aligned}
\end{equation}

\begin{equation}
   \raisebox{-0.4\height}{\includegraphics[height=6\baselineskip]{pics/mixed_5_5_6_6_1.png}}
   \begin{aligned}
      \quad
      {}{^5_5}H^2{}{^6_6}H^2(1, 2, 3, 4, 5)  (2, 3, 4, 5, 6)  (1, 7, 8, 9, 10, 11) \\
      (6, 7, 8, 9, 10, 11)
   \end{aligned}
\end{equation}

\begin{equation}
   \raisebox{-0.4\height}{\includegraphics[height=6\baselineskip]{pics/mixed_5_5_6_6_2.png}}
   \begin{aligned}
      \quad
      {}{^5_5}H^2{}{^6_6}H^2(1, 2, 3, 4, 5)  (3, 4, 5, 6, 7)  (1, 2, 8, 9, 10, 11) \\
      (6, 7, 8, 9, 10, 11)
   \end{aligned}
\end{equation}

\printcredits

%\vskip3pt

%\bio{}
%Author biography without author photo.
%Author biography. Author biography. Author biography.
%Author biography. Author biography. Author biography.
%Author biography. Author biography. Author biography.
%Author biography. Author biography. Author biography.
%Author biography. Author biography. Author biography.
%Author biography. Author biography. Author biography.
%Author biography. Author biography. Author biography.
%Author biography. Author biography. Author biography.
%Author biography. Author biography. Author biography.
%\endbio
%
%%\bio{figs/cas-pic1}
%Author biography with author photo.
%Author biography. Author biography. Author biography.
%Author biography. Author biography. Author biography.
%Author biography. Author biography. Author biography.
%Author biography. Author biography. Author biography.
%Author biography. Author biography. Author biography.
%Author biography. Author biography. Author biography.
%Author biography. Author biography. Author biography.
%Author biography. Author biography. Author biography.
%Author biography. Author biography. Author biography.
%\endbio
%
%\vskip3pc
%
%%\bio{figs/cas-pic1}
%Author biography with author photo.
%Author biography. Author biography. Author biography.
%Author biography. Author biography. Author biography.
%Author biography. Author biography. Author biography.
%Author biography. Author biography. Author biography.
%\endbio

\end{document}